\RequirePackage{fix-cm} %
\documentclass{applemlr}

\usepackage{amsmath}
\usepackage{enumerate}
\usepackage{algorithm}
\usepackage{algpseudocode}
\usepackage{amsfonts}
\usepackage{amsthm}
\usepackage{cleveref}
\usepackage{diagbox}
\usepackage{colortbl}
\usepackage{amssymb}
\usepackage{xspace}
\usepackage{wrapfig}
\usepackage{adjustbox}
\usepackage{tabularx}
\usepackage{booktabs}
\usepackage{mathtools}
\usepackage{tikz}
\usepackage{enumitem}
\usepackage{silence}
\usepackage{dsfont}
\usepackage[table]{xcolor}
\usepackage[dvipsnames]{xcolor}
\usepackage{multirow}
\usepackage{makecell}
\usepackage{xfakebold}

\usepackage{amsmath,amsfonts,bm}

\def\eqref#1{equation~\ref{#1}}

\def\1{\bm{1}}

\DeclareMathAlphabet{\mathsfit}{\encodingdefault}{\sfdefault}{m}{sl}
\SetMathAlphabet{\mathsfit}{bold}{\encodingdefault}{\sfdefault}{bx}{n}

\definecolor{textgray}{HTML}{6E6E73}
\usetikzlibrary{positioning, calc}
\usetikzlibrary{decorations.pathmorphing}

\makeatletter
\patchcmd{\wrong@fontshape}{\@gobbletwo}{}{}{}
\makeatother
\numberwithin{equation}{section}
\makeatletter
\AtBeginDocument{
  \urlstyle{sf}
  
}
\makeatother

\definecolor{light}{RGB}{125, 125, 125}
\crefname{tcb@cnt@pbox}{code}{code}
\Crefname{tcb@cnt@pbox}{Code}{Code}
\crefname{assumption}{assumption}{assumption}
\Crefname{assumption}{Assumption}{Assumptions}

\newtcolorbox[auto counter]{pbox}[2][]{
  colback=white,
  title=Code~\thetcbcounter: #2,
  #1,fonttitle=\sffamily,
  fontupper=\sffamily,
  arc=2pt,
  colframe=bgcolor,
  coltitle=fgcolor,
  colbacktitle=bgcolor,
  toptitle=0.25cm,
  bottomtitle=0.125cm
}

\makeatletter
\newcommand\applefootnote[1]{%
  \begingroup
  \renewcommand\thefootnote{}%
  \renewcommand\@makefntext[1]{\noindent##1}%
  \footnote{#1}%
  \addtocounter{footnote}{-1}%
  \endgroup
}
\makeatother

\definecolor{cverbbg}{gray}{0.90}

\newif\ifapplebuild
\newcommand{\appscale}{\ifapplebuild0.8\else1\fi}

\usepackage{hyperref}
\definecolor{darkorange}{HTML}{BF572E}
\hypersetup{
  colorlinks=true,
  linkcolor=darkorange,
  citecolor=darkorange,
  urlcolor=darkorange
}
\usepackage{titletoc}
\usepackage{url}
\usepackage{graphicx}
\usepackage{booktabs}
\usepackage{array}
\newcolumntype{P}[1]{>{\raggedright\arraybackslash}p{#1}}
\usepackage{lipsum}
\usepackage{amsthm}
\usepackage{algorithm}
\usepackage{algpseudocode}
\usepackage{subcaption}
\usepackage{wrapfig}
\usepackage{placeins}
\usepackage{cleveref}
\usepackage{enumitem}
\usepackage{xspace}
\usepackage{twemojis}

\DeclareMathSizes{6.8}{6.8}{5}{5}

\newcommand{\method}{PUMBA\xspace}
\crefname{equation}{Eq.}{Eqs.}
\Crefname{equation}{Eq.}{Eqs.}
\crefname{appsec}{Appendix}{Appendices}
\Crefname{appsec}{Appendix}{Appendices}
\newtheorem{proposition}{Proposition}
\newtheorem{definition}{Definition}
\newtheorem{theorem}{Theorem}

\newtheorem{corollary}{Corollary}
\theoremstyle{definition}
\newtheorem{assumption}{Assumption}

\crefname{assumption}{Assumption}{Assumptions}
\crefname{theorem}{Theorem}{Theorems}
\crefname{proposition}{Proposition}{Propositions}
\crefname{lemma}{Lemma}{Lemmas}
\crefname{corollary}{Corollary}{Corollaries}
\crefname{definition}{Definition}{Definitions}
\Crefname{assumption}{Assumption}{Assumptions}

\usepackage[colorinlistoftodos,textsize=tiny]{todonotes}        %

\title{On Trajectory-Aware Training for Masked\\ Diffusion Language Models} %

\applebuildtrue
\usepackage{needspace}

\author[\dagger]{Manuel Madeira}
\author{Amitis Shidani}
\author{Alice Bizeul}
\author{Victor Turrisi}
\author{Louis B\'ethune}
\author[\ddag]{Bhavika Devnani}
\author{Dan Busbridge}
\author{Pierre Ablin}
\author{Jo\~ao Monteiro}

\affiliation{Apple}
\affiliation[\dagger]{EPFL}
\affiliation[\ddag]{Georgia Institute of Technology}

\abstract{Masked diffusion models (MDMs) generate text by unmasking several tokens per step, but they are trained and sampled under different conditions.
The model is trained on randomly masked sequences, whereas inference follows a trajectory shaped by the model's own predictions. Additionally, each step has no access to what the previous one computed.
Recent methods narrow these limitations from separate angles, leaving open how these choices interact.
We introduce \method, a unified framework for trajectory-aware training that trains the denoiser on consecutive steps of policy-induced trajectories, passes information between steps, and optimizes them jointly by backpropagation through time.
A controlled study of this design space shows that \emph{i)} exact train--inference alignment fails due to local overfitting, whereas a looser alignment still brings training masks closer to those seen at inference;
\emph{ii)} passing continuous information outperforms discrete gradient estimators through the commitment at each step; and \emph{iii)} performance improves as backpropagation through time spans more steps, which we support theoretically.
Combined, these components match the best checkpoint of a same-size autoregressive model.
Building on these findings, we scale \method to supervised fine-tuning of LLaDA-8B, where it improves the trade-off between performance and number of function evaluations (NFEs) in both full-canvas and block diffusion generation.
At matched performance, it needs up to 22\% fewer NFEs than standard fine-tuning with twice the budget in full-canvas generation, and up to 26\% fewer than standard fine-tuning for the same number of steps in block diffusion. \looseness=-1

}

\metadata[Correspondence]{\sffamily
\url{manuel.madeira@epfl.ch}; \url{bdevnani3@gatech.edu};\newline
\mbox{\href{mailto:amitis_shidani@apple.com,abizeul@apple.com,v_turrisi@apple.com,l_bethune@apple.com,dbusbridge@apple.com,p_ablin@apple.com,jmonteiro2@apple.com}{\{amitis\_shidani, abizeul, v\_turrisi, l\_bethune, dbusbridge, p\_ablin, jmonteiro2\}@apple.com}}}
\metadata[Contributions]{\sffamily Work done while MM and BD were interns at Apple.}
\date{\sffamily\today}

\begin{document}

\maketitle

\vspace{-5pt}
\section{Introduction}
\label{sec:intro}
\vspace{-5pt}

Masked diffusion models (MDMs) \citep{shi2024simplified,sahoo2024mdlm} generate text by iteratively unmasking several tokens at a time, and have recently been scaled up \citep{nie2025llada,ye2025dream,labs2025mercury,wu2025fastdllm,nie2026illada,bie2025llada2}.
These models are trained to unmask partially masked sequences with randomly drawn masks.
At inference, however, an unmasking policy uses the model's predictions to decide which positions to reveal iteratively.
This disagreement between training and inference opens two gaps.
First, the model and policy jointly shape the inference trajectory, whose states differ from the training distribution (\Cref{fig:overview}).
Second, since training sequences are drawn independently, the model never learns to use information from previous steps.
Recent work addresses these gaps in different ways: by passing continuous information between steps \citep{jo2026loopholing,hu2026rcd}; by training along trajectories induced by the inference policy, known as Progressive Unmasking (PU)\footnote{
PU denotes the progressive unmasking framework; PUMA denotes the instantiation of \citet{kim2026puma}, which randomizes and anneals the number of tokens revealed per step.
} \citep{kim2026puma}; and by training on several consecutive steps jointly (\citealp{xia2026metastate,rozonoyer2026relay}).

\newif\ifnewfigone
\newfigonefalse
\newfigonetrue
\begin{figure}[t]
    \centering
\ifnewfigone
    \begingroup\fontfamily{ptm}\selectfont %
\usetikzlibrary{arrows.meta, calc}
\definecolor{ink}{HTML}{1d2227}      \definecolor{ink2}{HTML}{56616b}
\definecolor{hair}{HTML}{c9d0d6}     \definecolor{maskf}{HTML}{bfc7d0}
\definecolor{gt}{HTML}{16877b}       \definecolor{smp}{HTML}{7450c0}     \definecolor{err}{HTML}{d62839}
\definecolor{live}{HTML}{bf572e}
\definecolor{ok}{HTML}{1e9650}       \definecolor{no}{HTML}{d62839}
\definecolor{mMDM}{HTML}{4e5d64} \definecolor{mPU}{HTML}{4782a6}
\definecolor{mC1}{HTML}{fee0ae}  \definecolor{mCW}{HTML}{bf572e}

\def\fa{\fontsize{8}{9.6}\selectfont}     %
\def\fb{\fontsize{6.8}{8.2}\selectfont}   %

\def\cw{0.19}\def\ch{0.28}\def\cg{0.03}\def\km{m}\def\kg{g}\def\kp{p}
\def\sw{1.29}   %
\def\mw{0.50}   %
\def\gp{0.14}   %
\def\strip#1#2#3{%
  \foreach \c [count=\i] in {#3} {
    \pgfmathsetmacro{\xx}{#1+(\i-1)*(\cw+\cg)}
    \ifx\c\km\def\f{maskf}\else\ifx\c\kg\def\f{gt}\else\ifx\c\kp\def\f{smp}\else\def\f{err}\fi\fi\fi
    \fill[\f, rounded corners=0.4pt] (\xx,{#2-\ch/2}) rectangle ++(\cw,\ch);
  }}
\tikzset{
  arr/.style={-{Latex[length=2.3pt, width=2.1pt]}, line width=0.65pt},
  fth/.style={draw=ink2!85, fill=white, line width=0.5pt, rounded corners=1.5pt,
              minimum width=\mw cm, minimum height=0.34cm, inner sep=0pt, font=\fb, text=ink},
  lbl/.style={anchor=west, font=\fa\bfseries, text=ink, inner sep=0pt},
  sub/.style={anchor=west, font=\fb, text=ink2, inner sep=0pt},
  tiny/.style={font=\fb, inner sep=1pt},
}
\def\Y{\tikz[baseline=-0.5ex]\draw[ok, line width=1.1pt, line cap=round, line join=round]
                  (0,0.02) -- (0.065,-0.05) -- (0.19,0.09);}
\def\N{\tikz[baseline=-0.5ex]{\draw[no, line width=1.0pt, line cap=round]
                  (0,-0.06) -- (0.13,0.07) (0,0.07) -- (0.13,-0.06);}}

\def\chip#1#2{%
  \fill[\chipc] (#1,#2) circle (0.125);
  \draw[white, line width=0.85pt, line cap=round, line join=round]
    ({#1-0.060},{#2+0.002}) -- ({#1-0.018},{#2-0.042}) -- ({#1+0.064},{#2+0.050});}
\def\nochip#1#2{\draw[ink2!45, line width=0.6pt] (#1,#2) circle (0.110);}
\tikzset{iarr/.style={-{Latex[length=1.9pt, width=1.8pt]}, line width=0.55pt}}
\def\icw{0.075}\def\ich{0.15}\def\icg{0.016}
\def\icstrip#1#2#3#4#5{%
  \foreach \c [count=\i] in {#3} {
    \ifx\c\km\def\f{#5}\else\def\f{#4}\fi
    \fill[\f, rounded corners=0.2pt] ({#1+(\i-1)*(\icw+\icg)},{#2-\ich/2}) rectangle ++(\icw,\ich);}}
\def\icbox#1#2#3{%
  \filldraw[draw=#3, fill=white, line width=0.4pt, rounded corners=1pt] ({#1-0.20},{#2-0.125}) rectangle ({#1+0.20},{#2+0.125});
  \node[inner sep=0pt, font=\fontsize{5.6}{6}\selectfont, text=#3] at (#1,#2) {$f_\theta$};}
\def\ictraj#1#2#3#4{%
  \icstrip{#1-0.515}{#2}{m,m,m}{#3}{#4} \icstrip{#1-0.129}{#2}{m,g,m}{#3}{#4} \icstrip{#1+0.258}{#2}{g,g,m}{#3}{#4}
  \foreach \x in {-0.228,0.158} \draw[iarr, #3] ({#1+\x},#2) -- ++(0.07,0);}
\def\icself#1#2#3#4{%
  \icbox{#1-0.30}{#2}{#3} \icbox{#1+0.30}{#2}{#3}
  \draw[iarr, #4, rounded corners=1.2pt] ({#1-0.19},{#2-0.125}) -- ++(0,-0.12) -| ({#1+0.19},{#2-0.125});}
\def\icbptt#1#2#3#4{%
  \icself{#1}{#2}{#3}{#4}
  \draw[iarr, #4, dash pattern=on 1.2pt off 0.8pt, rounded corners=1.2pt]
    ({#1+0.30},{#2+0.125}) -- ++(0,0.12) -| ({#1-0.30},{#2+0.125});}

\ifdefined\figonelegend\else\newsavebox\figonelegend\fi
\sbox\figonelegend{\begin{tikzpicture}[every node/.style={anchor=west, font=\fb, text=ink2, inner sep=0pt}]
  \def\lg{0.55}   %
  \fill[maskf] (0,-0.09) rectangle ++(0.17,0.18);  \node (l1) at (0.24,0) {Masked};
  \fill[gt]  ($(l1.east)+(\lg,-0.09)$) rectangle ++(0.17,0.18);  \node (l2) at ($(l1.east)+(\lg+0.24,0)$) {Ground truth};
  \fill[smp] ($(l2.east)+(\lg,-0.09)$) rectangle ++(0.17,0.18);  \node (l3) at ($(l2.east)+(\lg+0.24,0)$) {Sampled};
  \fill[err] ($(l3.east)+(\lg,-0.09)$) rectangle ++(0.17,0.18);  \node (l4) at ($(l3.east)+(\lg+0.24,0)$) {Wrong, kept};
  \draw[arr, live] ($(l4.east)+(\lg,0)$) -- ++(0.42,0);             \node (l5) at ($(l4.east)+(\lg+0.50,0)$) {Carry $h_j$};
  \draw[arr, live, dash pattern=on 1.4pt off 1pt] ($(l5.east)+(\lg+0.42,0)$) -- ++(-0.42,0); \node (l6) at ($(l5.east)+(\lg+0.50,0)$) {Gradient $\nabla\ell$};
  \node (l7) at ($(l6.east)+(\lg,0)$) {$\mathrm{sg}[\cdot]$: stop-gradient};
\end{tikzpicture}}

\resizebox{\linewidth}{!}{%
\begin{tikzpicture}
\def\sa{3.30}\def\sb{5.75}\def\sc{8.20}                    %
\def\yH{5.20}\def\yM{4.40}\def\yP{3.45}\def\yC{2.65}\def\yB{1.55}\def\yI{-0.42}
\def\ka{11.06}\def\kb{12.40}\def\kc{13.74}                 %
\def\xl{10.21}\def\xr{10.43}\def\xe{14.40}   %

\def\step#1#2#3#4{%
  \strip{#1}{#2}{#3}
  \node[fth] (#4) at ({#1+\sw+\gp+\mw/2},#2) {$f_\theta$};
  \draw[ink2!80, line width=0.45pt] ({#1+\sw+0.02},#2) -- (#4.west);}
\def\lane#1#2#3#4{%
  \filldraw[fill=#2, draw=#2!70!black, line width=0.25pt] (0,{#1-0.26}) rectangle (0.055,{#1+0.20});
  \node[lbl] at (0.14,{#1+0.10}) {#3};
  \node[sub] at (0.14,{#1-0.17}) {#4};}
\def\carry#1#2#3#4{%
  \draw[arr, #3, rounded corners=1.8pt] ($(#1.south)+(0.16,0)$) -- ++(0,{-#4}) -| ($(#2.south)+(-0.16,0)$);}

\foreach \x/\j in {\sa/0,\sb/1,\sc/2} \node[tiny, text=ink2, anchor=south] at ({\x+\sw/2},{\yI+0.17}) {$\mathbf{x}_{t_\j}$};
\def\yIc{5.90}\def\yHl{5.50}   %
\foreach \x/\t in {\ka/{Policy\\trajectory}, \kb/{Self-\\conditioning}, \kc/{BPTT}}
  \node[tiny, text=ink, align=center, anchor=north] at (\x,{\yI+0.40}) {\t};

\node[lbl] at (0.14,{\yI+0.10}) {Inference};
\node[sub] at (0.14,{\yI-0.17}) {Own tokens, masks from $g$};
\step{\sa}{\yI}{m,m,m,m,m,m}{I0} \step{\sb}{\yI}{m,p,m,m,p,m}{I1} \step{\sc}{\yI}{p,p,m,m,p,e}{I2}
\draw[arr, smp] (I0.east) -- ({\sb-0.05},\yI);  \draw[arr, smp] (I1.east) -- ({\sc-0.05},\yI);
\node[tiny, text=smp!80!black, anchor=south] at ($(I0.east)!0.5!(\sb,\yI)+(0,0.19)$) {$g$};

\foreach \u/\v in {-0.08/\xl, \xr/\xe} \draw[hair, line width=0.5pt] (\u,{(\yM+\yP)/2}) -- (\v,{(\yM+\yP)/2});
\foreach \u/\v in {-0.08/\xl, \xr/\xe} \draw[hair, line width=0.5pt] (\u,{\yI+0.53}) -- (\v,{\yI+0.53});

\lane{\yM}{mMDM}{MDM}{Fresh $\mathbf{x}_t \sim q(\cdot\mid\mathbf{x}_0)$}
\step{4.40}{\yM}{g,m,m,g,m,g}{M0} \step{7.07}{\yM}{m,g,g,m,m,g}{M1}
\foreach \mx/\n/\t in {4.40/M0/{t}, 7.07/M1/{t'}} {
  \draw[ink2!45, line width=0.45pt, dash pattern=on 1.6pt off 1.2pt, rounded corners=2pt]
    ({\mx-0.18},{\yM-0.30}) rectangle ($(\n.east)+(0.16,0.30)$);
  \node[tiny, text=ink2, anchor=south] at ({\mx+\sw/2},{\yM+0.31}) {$\mathbf{x}_{\t}$};}

\def\pulane#1#2{%
  \step{\sa}{#1}{m,m,m,m,m,m}{#20} \step{\sb}{#1}{m,g,m,m,g,m}{#21} \step{\sc}{#1}{g,g,m,m,g,g}{#22}
  \draw[arr, gt] (#20.east) -- ({\sb-0.05},#1);  \draw[arr, gt] (#21.east) -- ({\sc-0.05},#1);}

\lane{\yP}{mPU}{+ PU}{One $\mathbf{x}_0$, masks from $g$}
\pulane{\yP}{P}
\node[tiny, text=gt!80!black, anchor=south] at ($(P0.east)!0.5!(\sb,\yP)+(0,0.19)$) {$g,\,\mathbf{x}_0$};

\lane{\yC}{mC1}{+ Carry}{Passes $h_j$ on, $W{=}1$}
\pulane{\yC}{C}
\carry{C0}{C1}{live}{0.20} \carry{C1}{C2}{live}{0.20}
\foreach \a/\b/\j in {C0/C1/1,C1/C2/2}
  \node[tiny, text=live, anchor=north] at ($(\a.south)!0.5!(\b.south)+(0,-0.21)$) {$\mathrm{sg}[h_\j]$};

\foreach \u/\v in {-0.08/\xl, \xr/\xe} \fill[live!10] (\u,{\yB-1.33}) rectangle (\v,{\yB+0.36});
\lane{\yB}{mCW}{+ BPTT}{PUMBA (ours)}
\pulane{\yB}{B}
\draw[arr, live, rounded corners=1.8pt] ($(B0.south)+(0.16,0)$) -- ++(0,-0.20) -| ($(B1.south)+(-0.16,0)$);
\draw[arr, live, rounded corners=1.8pt] ($(B1.south)+(0.16,0)$) -- ++(0,-0.20) -| ($(B2.south)+(-0.16,0)$);
\draw[arr, live, dash pattern=on 1.4pt off 1pt, rounded corners=1.8pt]
  ($(B2.south)+(-0.06,0)$) -- ++(0,-0.66) -| ($(B1.south)+(0.06,0)$);
\node[tiny, text=live, anchor=north] at ($(B0.south)!0.5!(B1.south)+(0,-0.21)$) {$\mathrm{sg}[h_1]$};
\node[tiny, text=live, anchor=north] at ($(B1.south)!0.5!(B2.south)+(0,-0.21)$) {$h_2$};
\node[tiny, text=live, fill=live!10, anchor=center] at ($(B1.south)!0.5!(B2.south)+(0,-0.66)$) {$\nabla\ell$};
\draw[live, line width=0.45pt] (\sb,{\yB-0.98}) -- ++(0,-0.05) -- ($(B2.east|-{0,\yB-1.03})$) -- ++(0,0.05);
\node[tiny, text=live, anchor=north] at ($(\sb,{\yB-1.03})!0.5!(B2.east|-{0,\yB-1.03})$) {Window, $W{=}2$};

\def\gh{ink2!22}
\foreach \yy/\a/\b/\c in {\yM/0/0/0, \yP/1/0/0, \yC/1/1/0, \yB/1/1/1} {
  \ifnum\a=1 \ictraj{\ka}{\yy}{gt}{maskf}\else\ictraj{\ka}{\yy}{\gh}{ink2!10}\fi
  \ifnum\b=1 \icself{\kb}{\yy}{ink2!85}{live}\else\icself{\kb}{\yy}{\gh}{\gh}\fi
  \ifnum\c=1 \icbptt{\kc}{\yy}{ink2!85}{live}\else\icbptt{\kc}{\yy}{\gh}{\gh}\fi}

\node[inner sep=0pt] at ({(\xe-0.08)/2},{\yI-0.72}) {\usebox\figonelegend};
\end{tikzpicture}}
\endgroup
    \caption{\textbf{From MDM training to \method.}
    Each row shows the inputs to the denoiser $f_\theta$ over consecutive steps; the bottom row corresponds to inference, where the policy $g$ commits the model's own samples, errors included.
    MDM trains on one independently masked $\mathbf{x}_t$.
    PU follows the trajectory of $g$ with ground-truth tokens, the carry passes the hidden state $h_j$ to the next step, and BPTT backpropagates the loss through $h_j$ within a window of $W$ steps.
    Right: the axes each row uses.}
\else
    \includegraphics[width=\linewidth]{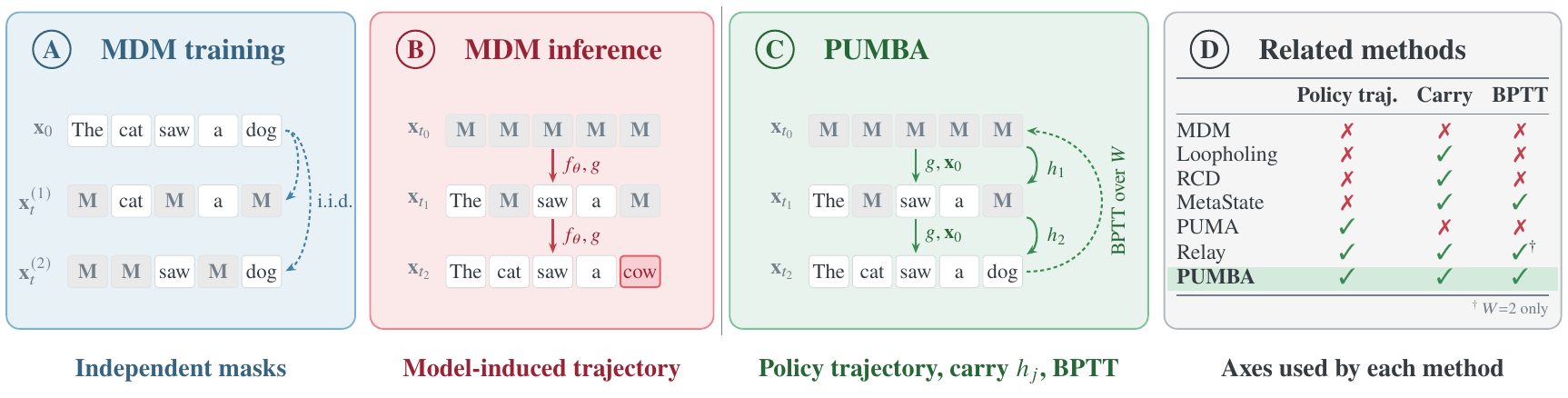}
    \caption{Train--inference mismatch and trajectory-aware training.
    \textbf{(A)} Standard MDM training masks $\mathbf{x}_0$ independently at each step.
    \textbf{(B)} At inference, the denoiser $f_\theta$ and policy $g$ induce a trajectory whose states differ from the training distribution.
    \textbf{(C)} \method trains along teacher-forced trajectories of $g$, passes a carry $h_j$ between steps, and backpropagates through $W$ steps.
    \textbf{(D)} Axes used by related methods (\Cref{sec:related,app:closest_work}).}
\fi
    \vspace{-15pt}
    \label{fig:overview}
\end{figure}

We unify these approaches in a single formulation of \emph{trajectory-aware} training, which we name \method (Progressive UnMasking with Backpropagation Across steps), where the denoiser is trained on consecutive steps of a trajectory rather than isolated sequences (\Cref{sec:exp_tinygsm}).
It spans three axes: the policy that builds each training trajectory; the \emph{carry}, \emph{i.e.} the information one step passes to the next; and the window $W$ of consecutive steps whose losses are optimized jointly by backpropagation through time (BPTT). Prior methods each address some of these axes.

To map this space, we first study a controlled setting where we train on TinyGSM \citep{liu2023tinygsm} and evaluate on GSM8K \citep{cobbe2021gsm8k}, analyzing how each axis contributes to reasoning performance (\Cref{sec:exp_tinygsm}).
First, we identify a \emph{local overfitting} phenomenon that prevents PU from revealing as few tokens per step in training as at inference (\Cref{sec:puma}). We show that even when training reveals several times more tokens per step, PU still reduces the train--inference mask discrepancy.
Second, we demonstrate that a continuous carry is the most effective way to cross the discrete commitment at each unmasking step (\Cref{sec:carry}).
Third, performance improves as $W$ grows, which we support theoretically.
Crucially, the three axes compound: each adds generative performance at every decoding budget. Together they match the performance of an autoregressive (AR) model of the same size trained on the same data, while decoding more than one token per step.
\looseness=-1

Building on these findings, we scale \method to the supervised fine-tuning (SFT) setting using LLaDA-8B-Base as the base model \citep{nie2025llada}, in both full-canvas and block diffusion formulations (\Cref{sec:exp_llada}).
To the best of our knowledge, this is the first trajectory-aware SFT at this scale (full tuning of all 8B parameters, 4096-token sequences, on a 1.8B-token general-purpose instruction dataset). 
\method suits SFT, which is data-constrained: further epochs soon yield little.
Guided by the controlled study and the failure modes of PU at this scale, we design a PUMBA instantiation that pushes the performance versus number of function evaluations (NFEs) Pareto frontier beyond standard SFT on reasoning and instruction following.
In full-canvas generation, 4000 steps of \method improve performance three times as much as doubling the SFT budget, in less training time.
At matched performance, \method needs up to $22\%$ fewer NFEs than an SFT run twice as long in full-canvas generation, and up to $26\%$ fewer in block diffusion than an MDM control trained as long.
\looseness=-1

\vspace{-5pt}
\section{Background}
\label{sec:background}
\vspace{-5pt}

\paragraph{Notation.}
Let $\mathcal{V}$ be a finite vocabulary containing a mask token $m \in \mathcal{V}$, and let $\mathbf{x} = (x^1,\ldots,x^L) \in \mathcal{V}^L$ be a sequence of length $L$ with masked positions $\mathcal{M}(\mathbf{x}) = \{i : x^i = m\}$.
We write $\Delta^{d}$ for the probability simplex in $\mathbb{R}^{d}$ and $\mathrm{sg}[\cdot]$ for stop-gradient.

\paragraph{Training of MDMs.}
The noising process of masked diffusion interpolates between a clean sequence $\mathbf{x}_0$ at level $t = 0$ and the fully masked sequence at $t = 1$.
Under the linear schedule, it replaces each token of $\mathbf{x}_0$ by $m$ independently with probability $t$ \citep{sahoo2024mdlm,shi2024simplified}.
An MDM learns a denoiser $f_\theta : \mathcal{V}^L \to (\Delta^{|\mathcal{V}|})^L$ that returns a distribution $f_\theta^i(\cdot \mid \mathbf{x}_t)$ over the clean token at each position $i$.
The denoiser infers $t$ from the masking ratio of $\mathbf{x}_t$, so it needs no time input \citep{gat2024discrete,sahoo2024mdlm,amin2026masking}.
Training maximizes an evidence lower bound on the log-likelihood that reduces to a weighted cross-entropy over the masked positions:
\begin{equation}
  \mathcal{L}_{\mathrm{MDM}}
  = \mathbb{E}_{\mathbf{x}_0,\, t,\, \mathbf{x}_t}
    \bigl[ \ell(\mathbf{x}_0, \mathbf{x}_t) \bigr],
  \qquad
  \ell(\mathbf{x}_0, \mathbf{x}_t)
  = \tfrac{1}{t} \textstyle\sum_{i \in \mathcal{M}(\mathbf{x}_t)}
      -\log f_\theta^i(x_0^i \mid \mathbf{x}_t).
  \label{eq:mdm}
\end{equation}
The factor $1/t$ follows from the bound and normalizes the sum over the $L \cdot t$ positions masked in expectation \citep{bethune2026design}.
Each training step draws $\mathbf{x}_0$ and $t \sim \mathcal{U}[0,1]$, masks $\mathbf{x}_0$ independently, and calls the denoiser once.

\vspace{-4pt}
\paragraph{Inference in MDMs.}
Generation starts from the fully masked sequence and reveals tokens over multiple steps.
We index steps by $j$ and write $\mathbf{x}_{t_j}$ for the sequence after step $j$, whose level is its masking ratio $t_j = |\mathcal{M}(\mathbf{x}_{t_j})| / L$.
Each step applies an \emph{unmasking policy} $g$ that selects $S = g(\mathbf{x}_{t_j}) \subseteq \mathcal{M}(\mathbf{x}_{t_j})$ and fills every $i \in S$ with a token drawn from $f_\theta^i(\cdot \mid \mathbf{x}_{t_j})$.
Revealing several positions at once trades exactness for speed by sampling from the product of their marginals rather than the joint distribution \citep{webb2026limitsconfidencediffusion}.
Choosing the most confident positions keeps this product closer to the joint distribution and improves samples at a fixed number of steps \citep{nie2025llada,wu2025fastdllm}.
The \emph{top-$u$} policy \citep{chang2022maskgit} reveals the $u$ masked positions with highest confidence $c_i = \max_{v \in \mathcal{V}} f_\theta^i(v \mid \mathbf{x}_{t_j})$.
We use this policy for the controlled experiments (\Cref{sec:exp_tinygsm}).
For larger MDMs, Fast-dLLM~\citep{wu2025fastdllm,wu2026fast} is an effective policy that reveals every position with $c_i \geq \tau$, or the single most confident one if none qualifies.
Combined with block diffusion decoding, which proceeds from left to right, block by block \citep{arriola2025block}, it preserves generation quality at far fewer denoiser calls, or \emph{number of function evaluations (NFEs)}, per sample.
We therefore adopt it as the inference policy for LLaDA in \Cref{sec:exp_llada}. See \Cref{app:background} for an extended version.

\section{Elucidating the Design Space of \method}
\label{sec:exp_tinygsm}

We map the design space of \method in a controlled setting following \citet{kim2026puma}: we train a 125M-parameter bidirectional transformer from scratch on TinyGSM~\citep{liu2023tinygsm} for 20 epochs and evaluate zero-shot on GSM8K~\citep{cobbe2021gsm8k} with the top-$u$ policy (\Cref{sec:background}).
All variants share the same architecture, data, batch size, and learning rate.

\subsection{Training--Inference Alignment via Progressive Unmasking}
\label{sec:puma}

Standard MDM training masks sequences independently of the model, whereas inference visits the sequences that the denoiser and the unmasking policy dictate.
Training and inference therefore encounter different distributions of masking patterns.
The patterns at inference time depend on the reveal count $u$.
In practice, the effective inference $u$ is small, since otherwise the product of marginals becomes a poor approximation of the joint distribution of the unmasked tokens, and generation quality degrades~\citep{webb2026limitsconfidencediffusion}.
For example, on TinyGSM, the accuracy of every variant we train peaks at $u \in \{2,3,4\}$ and roughly halves by $u{=}16$ (\Cref{fig:tinygsm_acc}, \textit{left}).
Similar trends hold for larger models~\citep{kang2025parallelbench}.
The regime of interest is therefore a rather small $u$.
Ideally, training would visit the same sequences as inference does at this $u$.

Progressive unmasking (PU) \citep{kim2026puma} aims at exactly this: instead of the independent corruptions of \Cref{eq:mdm}, it trains on the trajectory that the inference policy itself produces.
For each sample, training starts from the fully masked sequence.
At each optimizer step, we compute the loss on the current sequence and then use the inference-time unmasking policy to reveal more positions.
The resulting sequence is used in the next step, and this continues until the sample is fully revealed, after which a new sample is used.
Many consecutive training examples thus come from the same trajectory rather than from multiple data samples with independent mask patterns.
Each step still requires one denoiser call, and trajectories can be batched and progressively unmasked in parallel (\Cref{fig:local_overfitting}, \textit{left}).

\vspace{-0.25cm}
\paragraph{Formulation and training objective.}
We use the top-$u$ policy (\Cref{sec:background}), which matches the TinyGSM decoding policy: at each step, we reveal the $u$ most confident masked positions.
We use the MDM loss of \Cref{eq:mdm} unchanged, including the $1/t$ factor.
Here, $t$ is the realized masking ratio, so the loss is normalized by the number of masked tokens.
\citet{kim2026puma} show that, although PU sequences are not draws from the original noising process, this construction preserves the unique minimizer of the MDM objective.

\vspace{-0.25cm}
\paragraph{Revealed positions are teacher-forced.}
For stability, a trajectory advances by committing the ground-truth token (hence teacher-forced) at each selected position rather than using the predictions of the denoiser.
The mask pattern is therefore model-dependent, but the clean tokens are not.
Under the policies we consider, a committed position is never revisited, so a sampled token (student forcing; \citealp{bengio2015scheduled}) would introduce an error that could not be corrected later.
The loss would still ask for the ground-truth $\mathbf{x}_0$ at the remaining masked positions, so these targets could contradict the visible context.
We investigate student forcing and its failure cases in PU in \Cref{app:student_forcing}.

\begin{figure}[t]
\vspace{-0.25cm}
  \centering
  \includegraphics{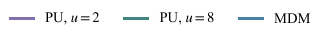}\\
  \includegraphics{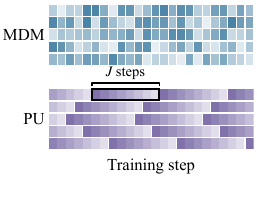}\hfill
  \includegraphics{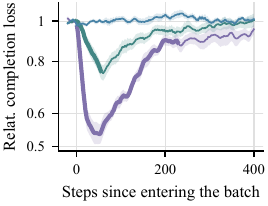}\hfill
  \includegraphics{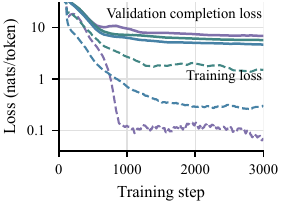}%
  \caption{Local overfitting of PU at small training $u$ on TinyGSM.
  \textit{Left}: batch rows over steps (shade: share of masked tokens); MDM draws a fresh sequence and mask per row at every step, PU keeps each sequence for $J \approx L/u$ steps.
  \textit{Middle}: completion loss of a sequence inserted into the batch, relative to the validation set (thick while in the batch, thin after).
  \textit{Right}: training and validation completion losses; PU and MDM training losses use different masks and are not comparable.}
  \vspace{-10pt}
  \label{fig:local_overfitting}
\end{figure}

\vspace{-0.15cm}
\paragraph{Local overfitting at small $\bm{u}$.}
At small $u$, the targeted inference regime, PU breaks two assumptions of stochastic optimization: each sequence is reused $J$ times instead of once, and consecutive batches share all but a fraction $1/J$ of their sequences instead of being independent.
To see how the model responds, we launch short TinyGSM training runs at fixed $u$, where every trajectory lasts exactly $J$ steps (\Cref{app:lo_setup}), and track the \emph{completion loss} of held-out sequences as they enter and leave the batch, \emph{i.e.}, the loss on the second half of the answer given the prompt and the first half.
We observe \emph{local overfitting} (\Cref{fig:local_overfitting}, \textit{middle}): at $u{=}2$, the completion loss of a sequence falls sharply within about $50$ steps of entering the batch, and even rises again before the sequence leaves it: the model soon fits the sequence on the masks of its trajectory, after which the sequence provides no gradient for its revealed tokens, which are forgotten (\Cref{app:lo_rebound}).
The drop shrinks as $u$ grows and, as expected, vanishes with MDM.
As in regular overfitting, the validation loss moves opposite to the training loss (\Cref{fig:local_overfitting}, \textit{right}); in these runs, PU trails MDM on validation completion loss for every $u \leq 8$ (\Cref{tab:lo_sweep}), so local overfitting prevents training PU at the small $u$ used at inference (see \Cref{app:local_overfitting} for an in-depth study).

\vspace{-0.15cm}
\ifapplebuild\def\dmasklines{10}\def\dmaskfile{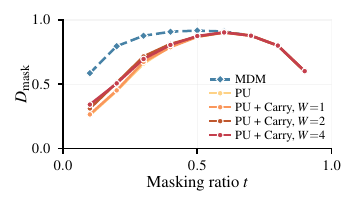}\else\def\dmasklines{13}\def\dmaskfile{figures/2026-09-25_tinygsm_dmask.pdf}\fi %
\newcommand{\dmaskfigure}{%
\begin{wrapfigure}[\dmasklines]{r}{0.35\linewidth}
\vspace{-0.7cm}
  \centering
  \includegraphics[width=\linewidth]{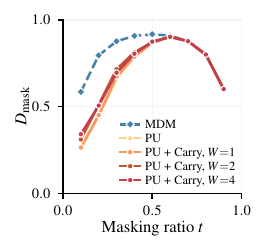}
  \vspace{-12pt}
  \captionsetup{justification=raggedright,singlelinecheck=false}
  \caption{Mask discrepancy at $u{=}2$ across $t$; lower is better.}
  \label{fig:tinygsm_dmask}
  \vspace{0.5cm}
\end{wrapfigure}
}

\paragraph{What about a larger $\bm{u}$?}
Since local overfitting fades as $u$ grows (\Cref{tab:lo_sweep}), training at a larger $u$ than we decode with may keep part of the alignment of PU without overfitting locally.
The original PUMA recipe does this by committing a random number of positions per step, which is annealed over training, and additionally committing every position whose confidence exceeds $0.9$ (\Cref{app:pu_details}).
Its effective training $u$, the number of generated positions divided by the number of forward passes per sequence, is about $133$ at the start of training and $67$ at the end, against $u{=}2$ at inference.
We use this recipe for the runs in \Cref{fig:tinygsm_acc,fig:tinygsm_dmask}.

\ifapplebuild\else\dmaskfigure\fi

\paragraph{Alignment under a mismatched $\bm{u}$.}
To test whether training at a larger $u$ than inference still reduces the train--inference mismatch, we compare the masking patterns seen in training under the PUMA recipe with those met during generation.

\begin{definition}[Mask discrepancy]
\label{def:dmask}
For a sample $\mathbf{x}_0$ and a masking ratio $t$, let $P_{\mathrm{train}}$ and $P_{\mathrm{inf}}$ be the distributions of masks $\mathbf{m} \in \{0,1\}^L$ at ratio $t$ produced by the training procedure and by the sampler, respectively.
The mask discrepancy is \looseness=-1
\[
  \begin{gathered}
  D_{\mathrm{mask}}(t) = \mathbb{E}_{\mathbf{x}_0}\bigl[\mathrm{MMD}_k^2(P_{\mathrm{train}}, P_{\mathrm{inf}})\bigr],
  \\
  k(\mathbf{m}, \mathbf{m}') = \exp\bigl(-d_H(\mathbf{m}, \mathbf{m}')/\sigma\bigr),
  \end{gathered}
\]
where $\mathrm{MMD}$ is maximum mean discrepancy~\citep{gretton2012kernel} and $d_H$ the Hamming distance.
\end{definition}
\ifapplebuild\else\newpage\fi %
\ifapplebuild\dmaskfigure\fi
The kernel $k$ is strictly positive definite on $\{0,1\}^L$, so $D_{\mathrm{mask}}(t) = 0$ exactly when training and inference produce the same masks at ratio $t$.
Larger values mean that inference visits mask patterns rarely seen in training.
$D_{\mathrm{mask}}$ depends only on which positions are masked, not on the value of the token or the carry, so it isolates the masking distribution shift.
\Cref{app:dmask} gives the estimator we use in practice.
Thus, although local overfitting forces a much larger training $u$, the PUMA recipe still reduces $D_{\mathrm{mask}}$ at $u{=}2$ well below MDM (\Cref{fig:tinygsm_dmask}).
The information propagation techniques of \Cref{sec:carry} do not undo the alignment of PU, inheriting this benefit.\looseness=-1

\subsection{Propagating Information Across Denoising Steps}
\label{sec:carry}

In PU, each step influences the input seen by the next step, but the objective trains each step only to make its own predictions, while the next step receives only the committed tokens.
We argue that a denoising step should not only predict its own tokens, but also produce an output that enables subsequent steps to operate effectively, especially since committed positions are never revisited.
Training for this requires a channel that passes information from one step to the next, and a gradient that flows back through it.
The channel can be a continuous \emph{carry} that bypasses the discrete commitment \citep{jo2026loopholing,hu2026rcd}, or a gradient estimator through the committed tokens themselves, which we call a \emph{discrete bridge}~\citep{jang2016categorical,maddison2016concrete,bengio2013estimating,williams1992simple}.
We first describe the carry and how we train it, then compare it with the discrete bridges.

\paragraph{Passing information forward.}
The carry can be the hidden state of the denoiser \citep{jo2026loopholing} or the probability-weighted input embedding of each prediction \citep{hu2026rcd}.
We adopt the carry design from loopholing \citep{jo2026loopholing} for its simplicity and compactness, and do not ablate other designs.
The denoiser thus also takes the previous hidden state as input and returns its own:
\begin{equation}
  \bigl(f_\theta(\cdot \mid \mathbf{x}_{t_j}, h_j),\ h_{j+1}\bigr)
  = f_{\mathrm{carry}}(\mathbf{x}_{t_j}, h_j).
  \label{eq:carry}
\end{equation}
The carry $h_j \in \mathbb{R}^{L \times d}$ is the last hidden state before the output projection, one vector of width $d \ll |\mathcal{V}|$ per position (\Cref{app:experimental_setup}).
Each trajectory starts from a zero carry, and each inference step passes its carry to the next.
The carry thus provides a weak form of student forcing: it is produced by the model at the previous step and reflects its imperfect predictions, while the visible tokens, and hence the targets, stay consistent with $\mathbf{x}_0$.

\newcommand{\bpttablationtable}{%
\begin{wraptable}[10]{r}{0.36\linewidth}
  \vspace{-1.2em}
  \centering
  \caption{Peak GSM8K accuracy (\%) for different BPTT techniques at $W{=}2$ on TinyGSM.}
  \label{tab:tinygsm_chain_ablation}
  \small
  \renewcommand{\arraystretch}{1.08}
  \begin{tabular}{@{}lc@{}}
    \toprule
    PU + REINFORCE           & $41.2$ \\
    PU + Gumbel-softmax      & $42.4$ \\
    PU + Straight-through    & $44.7$ \\
    PU + Carry        & $\mathbf{51.2}$ \\
    \bottomrule
  \end{tabular}
  \vspace{-1em}
\end{wraptable}
}
\ifapplebuild\bpttablationtable\fi

\paragraph{Passing gradient back.}
When each update trains on a single step ($W{=}1$), the model learns to use the carry it receives but not to produce a useful one, since no term of the objective depends on the carry's effect on the next step.
We therefore compute the loss on $W$ consecutive steps of the same trajectory in one update and backpropagate through the carries between them.
The loss at each step then reaches the preceding steps of the window, so the carry is also optimized for the predictions that follow.
As a trajectory usually exceeds $W$ steps, we split it into consecutive windows, each with its own update, and gradients do not cross between windows.
The resulting procedure (\Cref{alg:train}) is truncated backpropagation through time (BPTT).

\ifapplebuild\else\bpttablationtable\fi
\paragraph{Comparing the channels.}
To compare the carry with the discrete bridges, we fix the window at $W{=}2$ and replace the carry with a relaxed categorical (Gumbel-softmax) \citep{jang2016categorical,maddison2016concrete}, a straight-through estimator \citep{bengio2013estimating}, and a score-function (REINFORCE) estimator \citep{williams1992simple} with a leave-one-out baseline \citep{kool2019buy}.
To remove confounders and keep the training distribution stationary, all runs of this comparison keep the PUMA recipe at its initial setting instead of annealing it, under the same budget.

The carry outperforms all discrete bridges (\Cref{tab:tinygsm_chain_ablation}; training curves in \Cref{app:chain_ablation}), so we adopt it as the default strategy.

\paragraph{Implementation.}
The carry passes through a zero-initialized LayerNorm~\citep{ba2016layer} and is added to the token embeddings, so training starts from the unmodified denoiser.
MDM training has no predecessor to take a carry from, so loopholing obtains one by self-conditioning \citep{chen2022analog,jabri2022rin}, with an extra forward pass that adds about $30\%$ training time \citep{jo2026loopholing}.
Under PU, the predecessor exists, so we simply cache its carry, one $L \times d$ tensor per trajectory.
Each update averages the losses of its window, so compute and memory grow linearly in $W$ while the update scale does not.
\Cref{alg:train} in \Cref{app:pumba_alg} summarizes the resulting algorithm, \method.

\paragraph{Why a longer window helps.}
The comparison above fixes $W{=}2$ without exploring how large this window should be.
The following result bounds how close the learned sampler gets to the best achievable one as a function of $W$ (proofs in \Cref{app:theory}).

\begin{proposition}[Informal]
\label{prop:window}
Assume teacher-forced commits and a reveal rule that matches inference.
\begin{enumerate}[label=(\roman*), nosep, leftmargin=*]
  \item No gradient flows through committed tokens. Without a carry (or $W{=}1$), a step's loss cannot reward earlier steps; with BPTT, it reaches at most $W{-}1$ of the earlier steps of its window.
  \item Assume the carry is contractive with factor $\gamma<1$. Let $\theta_W$ be a point at which the truncated gradient vanishes and the loss satisfies a Polyak--{\L}ojasiewicz condition~\citep{karimi2016linear}. Then
  \[
    \mathrm{KL}(p^\star \Vert p_{\theta_W}) - \inf_{\theta} \mathrm{KL}(p^\star \Vert p_{\theta}) \le C \Bigl(\frac{J}{W} - 1\Bigr)^2,
  \]
  where $J$ is the number of decoding steps and $C$ does not depend on $W$.
\end{enumerate}
\end{proposition}

The first part of \Cref{prop:window} explains why discrete gradient estimators do not match the carry (\Cref{tab:tinygsm_chain_ablation}), while the second part shows that the gap to the best achievable sampler shrinks as the window grows, and vanishes when the window covers the whole trajectory.
If the carry-augmented sampler can approximate $p^\star$ arbitrarily well, a large enough window then yields a sampler closer to $p^\star$ than any AR model with nonzero error (\Cref{cor:ar}).
Part (i) holds along any teacher-forced trajectory; part (ii) holds when training and inference share the reveal rule, and \Cref{fig:tinygsm_dmask} and \Cref{tab:tinygsm_compose} test empirically whether its prediction carries over to the smaller inference $u$.
We discuss the scope of the bound in more detail in \Cref{app:theory_scope}.

\subsection{Composing the Design Choices}
\label{sec:design_compose}

\begin{wraptable}[13]{r}{0.36\linewidth}
  \vspace{-1.5em}
  \centering
  \caption{Best GSM8K accuracy (\%) at $u{=}2$; MDM and PU: mean $\pm$ s.d.\ over three seeds.}
  \label{tab:tinygsm_compose}
  \vspace{-1.em}
  \small
  \renewcommand{\arraystretch}{0.95}
  \begin{tabular}{@{}lc@{}}
    \toprule
    AR                         & $\underline{55.3}$ \\
    \midrule
    MDM                        & $34.8 \pm 2.5$ \\
    MDM $\times 2$ epochs      & $42.6$ \\
    MDM $\times 8$ epochs      & $43.6$ \\
    \midrule
    PU                         & $40.2 \pm 0.3$ \\
    PU + Carry, $W{=}1$        & $44.4$ \\
    \phantom{PU + Carry,} $W{=}2$ & $48.7$ \\
    \phantom{PU + Carry,} $W{=}4$ & $52.6$ \\
    \phantom{PU + Carry,} $W{=}8$ & $\mathbf{55.6}$ \\
    \bottomrule
  \end{tabular}
  \vspace{-1em}
\end{wraptable}

We now combine the choices of \Cref{sec:puma,sec:carry} and compare them with MDM and AR baselines.
To increase alignment, we train a standard MDM, PU, PU with the carry ($W{=}1$), and PU with the carry trained by BPTT over $W \in \{2,4,8\}$ steps.
As controls, we train MDM for $2\times$ and $8\times$ as many epochs, which matches the number of denoiser forward and backward passes of BPTT at $W{=}2$ and $W{=}8$, since each BPTT update runs $W$ of them.
We also train an AR model of the same size on the same data.

The three design choices compose: each one adds accuracy on top of the previous ones, at $u{=}2$ (\Cref{tab:tinygsm_compose}) and at every other decoding budget we evaluate (\Cref{fig:tinygsm_acc}).
At a matched number of denoiser passes, MDM trained $2\times$ or $8\times$ longer remains substantially below the corresponding BPTT at $W{=}2$ and $W{=}8$, respectively.
Accuracy also keeps increasing with the window, as \Cref{prop:window} predicts.
Together, the three choices reach the best AR checkpoint while revealing more than one token per step.
Accuracy during training and pass@$k$ are in \Cref{app:tinygsm_results}.

\section{A \method Recipe for Large-Scale SFT}
\label{sec:exp_llada}

Following \Cref{sec:exp_tinygsm}, our SFT recipe for LLaDA-8B-Base~\citep{nie2025llada} trains on inference-policy trajectories at a larger $u$ than at inference, avoiding local overfitting while keeping training masks aligned, and trains the carry by BPTT over a window $W$.

\paragraph{Progressive unmasking with variable lengths.}
TinyGSM samples have similar lengths, so the PUMA recipe, which reveals a fixed length fraction per step, unmasks a similar number of tokens for all of them.
SFT responses instead range from 14 tokens at the 10th percentile to 778 at the 90th (\Cref{app:pu_details}). We fix the minimum number of tokens per step: every trajectory starts fully masked, and each step reveals the $u$ most confident positions, plus every other position with confidence above $\tau{=}0.9$. \looseness=-1

\paragraph{Setup and evaluation.}
To separate the effect due to \method from the model simply learning the data distribution, we use two stages: standard SFT of LLaDA-8B-Base, during which the model captures most of the target distribution, followed by a \emph{Post-SFT} stage where we apply \method with different parameters and the corresponding baselines. Since the original LLaDA SFT dataset is not publicly available, we use the SFT dataset from the OLMo~3 post-training pipeline~\citep{olmo2025olmo}. We follow typical SFT practice: the prompt is fully revealed and only the response is maskable~\citep{nie2025llada,ye2025dream}.
Additional details are present in \Cref{app:experimental_setup}.

We evaluate along two axes: the model's learned ability to follow instructions, and potential regression on other capabilities.
For instruction following, we use IFEval~\citep{zhou2023ifeval}, which measures the central capability targeted by our general SFT. To check that other capabilities are preserved, we additionally evaluate on GSM8K~\citep{cobbe2021gsm8k} (math) and MBPP~\citep{austin2021mbpp} (code).
We refer to the mean of the three benchmarks as the \emph{score}.
We decode with Fast-dLLM~\citep{wu2025fastdllm} in blocks of size 32 and report each model's score--NFE curve over its threshold $\tau \in \{0.7, 0.8, 0.9, 1\}$, from one token per step ($\tau{=}1$) to more in parallel.
We consider two settings: \emph{full-canvas} and \emph{block diffusion}.

\subsection{Full-Canvas Generation}
\label{sec:exp_full_canvas}

In the full-canvas setting, the model attends to the full sequence bidirectionally.
We SFT the model with the vanilla MDM objective for 1.5 epochs ($\approx$ 25k training steps) at a global batch size of 128 sequences, over which the model sees 2.71\,B tokens.
From this checkpoint, we enter the Post-SFT stage and train for an additional 4000 steps.
We compare PU and PU with the carry trained by BPTT over a window $W$, where $W{=}1$ passes the carry without gradient, and sweep the training $u$ and $W$.
As baselines, we continue the 1.5-epoch checkpoint with MDM for the same 4000 steps (Post-SFT MDM), which rules out gains from the additional updates alone, and double the SFT budget to 3 epochs, which shows that MDM training has nearly saturated on this data.
We report results for $u \in \{32, 64\}$ with remaining values present in \Cref{app:llada_results}.

\begin{figure}[t]
  \centering
  \vspace{-0.75cm}
  \includegraphics[width=\linewidth]{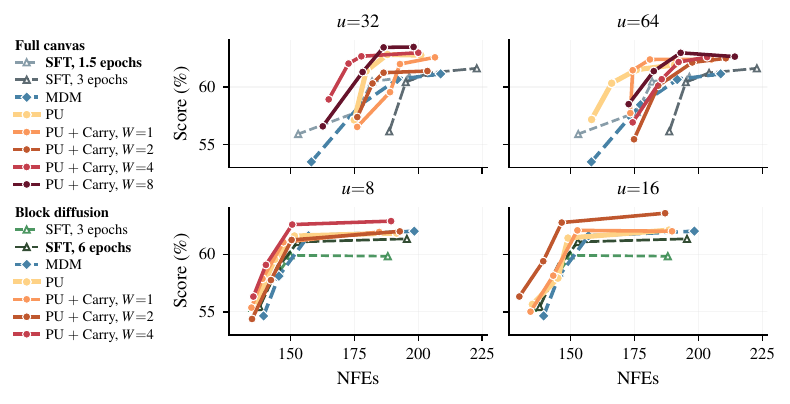}
  \vspace{\ifapplebuild-14pt\else-4pt\fi}
  \caption{Post-SFT of LLaDA-8B for full canvas (\emph{top row}) and block diffusion (\emph{bottom row}).
  In each legend, we set in bold the SFT checkpoint the Post-SFT runs of that setting start from.}
  \label{fig:llada_postsft_grid}
\end{figure}

Neither baseline improves much: Post-SFT MDM stays close to its starting checkpoint, and the 3-epoch SFT, despite twice the training, improves only slightly and only at higher NFEs, suggesting that MDM training has nearly saturated on this data.
In contrast, \method improves the score while reducing NFEs, with a fraction of the SFT step budget.
Its best configuration ($u{=}32$, $W{=}8$) gains $2.6$ points, over three times the gain of doubling the SFT budget, and at matched score \method needs up to $22\%$ fewer NFEs than the 3-epoch SFT (\Cref{fig:llada_postsft_grid}).

\ifapplebuild\else\vspace{-1.5cm}\fi %
\subsection{Block Diffusion}
\label{sec:exp_block_diff}

In the \emph{block diffusion} setting, the model generates the response block by block, attending block-causally to previous blocks and bidirectionally within the current one.
This attention pattern departs from LLaDA's fully bidirectional pre-training, and we find that it requires a longer SFT to reach competitive performance.
We therefore run SFT for 6 epochs, using the vectorized implementation of \citet{arriola2025block}, which trains all blocks of a sequence in a single forward pass, with blocks of size 32 to match the Fast-dLLM inference policy~\citep{wu2025fastdllm}.

Post-SFT starts from this checkpoint and, as in the full-canvas setting, runs for 4000 steps for each composition of PU and BPTT.
In the full-canvas setup, the reference for longer MDM training was a run with twice the SFT budget, which here would require 12 epochs.
We instead use the 3-epoch checkpoint of the same run, so that the step from 3 to 6 epochs measures what doubling SFT yields.
Since blocks have a fixed size of 32, $u$ fixes the number of steps per block, which upper-bounds the window: e.g., $u{=}16$ implies $W \leq 2$.
We report results for $u\in\{8,16\}$ in \Cref{fig:llada_postsft_grid}; the remaining values are in \Cref{app:llada_results}.
Consistent with the full-canvas setting, \method improves over the MDM baselines.
PU alone and the carry ($W{=}1$) bring modest gains, but BPTT with wider windows yields a clear improvement.
The best configuration in \Cref{fig:llada_postsft_grid} ($u{=}16$, $W{=}2$) gains $2.2$ points in 4000 steps, more than the 50k SFT steps from 3 to 6 epochs, and at matched score it needs $26\%$ fewer NFEs than the Post-SFT MDM control.
The \method variants also consistently reduce NFEs, a trend not observed with progressively longer SFT (see \Cref{tab:llada_postsft_block} for best scores per variant).

\ifapplebuild\else\vspace{-0.5cm}\fi %
\subsection{Failure Modes at Small Training \texorpdfstring{$u$}{u}}
\label{subsec:smallu}

\paragraph{Local overfitting at scale.} Random masking acts as implicit data augmentation, which favors MDM over AR models when data is repeated~\citep{prabhudesai2025diffusion}. However, these masks are drawn independently of the unmasking policy, so most orderings seen in training never occur at inference. PU concentrates training on the orderings the inference policy produces, yielding a more targeted and data-efficient signal. At moderate $u$ and matched compute, PU-based variants tend to outperform the MDM baselines. This advantage, however, holds only while $u$ is large enough: the smaller $u$, the longer each trajectory, and the fewer distinct samples the model sees in a fixed number of steps. \Cref{fig:llada_diversity} measures this cost directly: at full-canvas $u{=}1$ a run sees only $0.6\%$ of samples visited by MDM, rising to $26\%$ at $u{=}128$. Block diffusion is less exposed, reaching $100\%$ by $u{=}32$, since the parallel-block implementation stratifies unmasking: $u{=}1$ reveals one token \emph{per block} rather than one in the entire sequence. As $u$ decreases, fewer distinct samples are seen and each stays longer in the batch (\Cref{fig:llada_diversity_steps}), until the model overfits locally and training collapses (\Cref{fig:llada_collapse}). At the other extreme, PU at $u{=}128$ only matches the Post-SFT MDM control: each step reveals a large part of the response, so the trajectory reduces to a few coarse steps.
The best training $u$ needs to be small enough to align with inference and large enough to allow the model to learn. \looseness=-1
\vspace{-3pt}

\begin{figure}[t]
  \centering
  \vspace{-0.5cm}
  \begin{subfigure}[b]{0.3904\linewidth}
    \centering
    \phantomsubcaption
    \includegraphics[width=\linewidth]{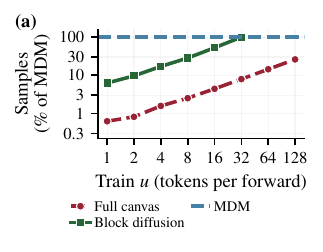}
    \label{fig:llada_diversity}
  \end{subfigure}
  \hfill
  \begin{subfigure}[b]{0.6011\linewidth}
    \centering
    \phantomsubcaption
    \includegraphics[width=\linewidth]{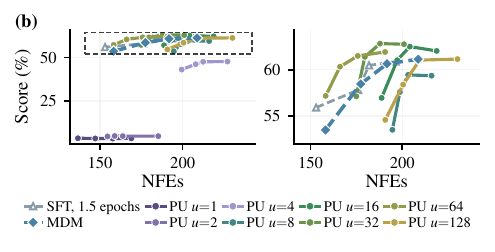}
    \label{fig:llada_collapse}
  \end{subfigure}
  \vspace{\ifapplebuild-36pt\else-28pt\fi}
  \caption{Trade-off between train--inference alignment and sample diversity.
  (\subref{fig:llada_diversity})~Sample diversity for PU with varying $u$, measured as the number of distinct training samples visited, as a percentage of the Post-SFT MDM control's.
  (\subref{fig:llada_collapse})~Performance as a function of NFEs for the full-canvas setting, at full extent (left) and zoomed into the dashed box (right).}
  \label{fig:llada_diversity_collapse}
\end{figure}

\section{Related Work}
\label{sec:related}

\paragraph{Masked diffusion language models.}
Diffusion language models operate on continuous representations \citep{li2022diffusionlm,dieleman2022continuous} or discrete tokens \citep{hoogeboom2021argmax,austin2021d3pm,campbell2022continuous,lou2024sedd}.
Masked diffusion simplifies the absorbing-state formulation \citep{shi2024simplified,sahoo2024mdlm} and scales to large language models through autoregressive conversion or training from scratch \citep{gong2025diffullama,ye2025dream,nie2025llada}.
Block Diffusion is a particularly effective diffusion formulation for language modeling, combining autoregressive generation across blocks with parallel denoising within them \citep{arriola2025block}.
We study how to post-train these models for the states their decoders induce.

\paragraph{Structured inference.}
MDM inference policies choose both the reveal order and the degree of parallelism.
Common criteria include confidence \citep{ghazvininejad2019maskpredict,chang2022maskgit,wu2025fastdllm}, probability margins \citep{kim2025train}, and predictive entropy \citep{ben2026accelerated}; others use explicit planners \citep{liu2025ddpd,peng2025pathplanning} or learned unmasking policies \citep{jazbec2025learning}.
These policies induce structured trajectories whose intermediate states differ from the random corruptions used in conventional training, motivating our focus on trajectory alignment.\looseness=-1

\paragraph{Train--inference alignment.}
Training on expert states while generating from model predictions is a longstanding challenge \citep{ross2011dagger,bengio2015scheduled}; in MDMs, the shift concerns both revealed positions and token content.
PUMA trains on progressive teacher-forced unmasking trajectories \citep{kim2026puma}, while PAPL reweights the loss for the decoding planner \citep{peng2026papl}.
Other approaches leverage task rewards along denoising trajectories \citep{he2025mdpo,wang2025tracerl} or supervise model-induced states through on-policy distillation \citep{ren2026tracedistill,su2026onpolicy,xu2026temporal}.
Relatedly, SDTT distills many-step denoising into fewer sampling steps \citep{deschenaux2025beyond}.
We investigate alignment directly within supervised instruction fine-tuning, without task-specific reward optimization or distillation.

\paragraph{State propagation across denoising steps.}
Loopholing preserves latent information through self-conditioning \citep{jo2026loopholing}.
RCD and DiffusionGemma instead feed back the probability-weighted input embedding of each prediction \citep{hu2026rcd,team2026diffusiongemma}.
MetaState adds persistent working memory to a frozen backbone \citep{xia2026metastate}.
Closely related to \method, Relay passes differentiable per-token representations between steps and trains them with truncated BPTT \citep{rozonoyer2026relay}, but uses a two-step unroll ($W{=}2$ in our notation) in all its experiments.
Our study combines trajectory-aligned supervised post-training with continuous state propagation, examines each component's contribution, and characterizes the trade-off between alignment and learnability.
\Cref{app:closest_work} details the comparison with the closest work.

\section{Conclusion}
\label{sec:conclusion}

\vspace{-5pt}
We improve masked diffusion training by aligning it with the trajectories encountered during generation. Building on progressive unmasking methods, we cache the denoiser's hidden state across steps and use truncated BPTT to train each step for its contribution to subsequent predictions. We study the components of the resulting method, \method, on TinyGSM and show we can scale \method to general-purpose instruction-following SFT on LLaDA-8B, improving the performance--NFE frontier under both full-canvas and block diffusion. In particular, the Post-SFT recipe we employ is practical: it pushes a model beyond what additional SFT allows, and the resulting reduction in denoising steps amortizes this extra training stage, which is short, since a small number of training steps suffices to realize the gains of \method. Our findings highlight the need to balance trajectory alignment with the diversity required for effective learning.

\vspace{-5pt}
\paragraph{Limitations and Future Work.} We identify local overfitting at small reveal counts, caused by reusing each training sequence over many consecutive steps, and ways to mitigate it; turning these findings into training procedures that sustain closer inference alignment across model scales, and across decoding policies beyond the confidence-based ones we study, remains open. Extending the carry's weak student forcing to sampled tokens did not improve on teacher forcing in preliminary experiments (\Cref{app:student_forcing}), since irreversible wrong commitments leave the context inconsistent with the target; remasking could address this. Finally, we use a minimal carry to isolate the benefits of trajectory alignment and BPTT; richer carries, such as the fixed-size working memory of MetaState~\citep{xia2026metastate}, and BPTT-specific gradient checkpointing, selective gradient propagation, or adaptive windows could improve the trade-off between training cost, memory, and credit-assignment horizon.\looseness=-1

\section*{Acknowledgments}
We thank Marco Cuturi, Oscar Davis, Eleonora Gualdoni, Michael Klein, Tatiana Likhomanenko, Barry Theobald, and Russ Webb for their helpful feedback and critical discussions throughout the process of writing this paper; Okan Akalin, Brian Gamp, Denise Hui, Li Li, Cindy Liu, Evan Samanas, Guillaume Seguin, and the wider Apple infrastructure team for assistance with developing scalable, fault-tolerant code.
Names are in alphabetical order by last name within group.

\bibliography{iclr2027_conference}
\bibliographystyle{iclr2027_conference}

\newpage
\appendix
\raggedbottom %
\crefalias{section}{appsec}
\crefalias{subsection}{appsec}
\crefalias{subsubsection}{appsec}

\section*{Appendix Contents}
\startcontents[appendix]
{\ifapplebuild\sffamily\fi %
\printcontents[appendix]{}{1}{\setcounter{tocdepth}{2}}}
\newpage

\section{Extended background}
\label{app:background}

This section gives the full version of \Cref{sec:background}.

\paragraph{Notation.}
Let $\mathcal{V}$ be a finite vocabulary containing a mask token $m \in \mathcal{V}$, and let $\mathbf{x} = (x^1,\ldots,x^L) \in \mathcal{V}^L$ be a sequence of length $L$, where $x^i$ is its $i$-th token and $\mathcal{M}(\mathbf{x}) = \{i : x^i = m\}$ its set of masked positions.
For any $d \in \mathbb{N}$, define the probability simplex $\Delta^d := \{\mathbf{p} \in \mathbb{R}_{\geq 0}^{d} : \mathbf{1}_d^{\top}\mathbf{p} = 1\}$, and write $\mathrm{Cat}(\mathbf{p})$ for the categorical distribution with parameter $\mathbf{p} \in \Delta^{|\mathcal{V}|}$ and $\mathbf{e}_v \in \Delta^{|\mathcal{V}|}$ for the one-hot vector at $v \in \mathcal{V}$.
Finally, $\mathrm{sg}[\cdot]$ denotes stop-gradient.

\paragraph{Masked diffusion.}
Masked diffusion defines a generative model over sequences through a forward process that masks tokens and a backward process that unmasks them.
The forward, or \emph{noising}, process is a family of corruptions indexed by a level $t \in [0,1]$ that interpolates between a clean sequence $\mathbf{x}_0 \in \mathcal{V}^L$ at $t = 0$ and the fully masked sequence $(m,\ldots,m)$ at $t = 1$.
It factorizes over positions, with $q(x_t^i \mid x_0^i) = \mathrm{Cat}\bigl(\alpha_t \mathbf{e}_{x_0^i} + (1 - \alpha_t)\,\mathbf{e}_m\bigr)$, so that $x_0^i$ is retained with probability $\alpha_t$ and replaced by $m$ otherwise.
The coefficient $\alpha_t$ decreases monotonically from $\alpha_0 = 1$ to $\alpha_1 = 0$, and a common choice is the linear schedule $\alpha_t = 1 - t$~\citep{sahoo2024mdlm,shi2024simplified}.
The backward, or \emph{denoising}, process runs from $t = 1$ to $t = 0$ and fills in the masked positions.
Its transitions are defined by the posterior $q(x_0^i \mid \mathbf{x}_t)$ over the clean token at a masked position, which depends on the unknown data distribution and is intractable.
Masked diffusion models approximate this posterior with a learned denoiser and generate data by iterating the resulting transitions from the fully masked sequence.

\paragraph{Training of MDMs.}
Since masking is absorbing and position-wise, the masking ratio of $\mathbf{x}_t$ determines the level $t$, so the posterior depends on $\mathbf{x}_t$ only through its mask pattern and revealed tokens \citep{gat2024discrete,sahoo2024mdlm,amin2026masking}.
The denoiser therefore takes only the corrupted sequence as input, $f_\theta : \mathcal{V}^L \to (\Delta^{|\mathcal{V}|})^L$, and returns one categorical distribution per position, $f_\theta^i(\cdot \mid \mathbf{x}_t) \in \Delta^{|\mathcal{V}|}$.
Fitting $f_\theta$ maximizes an evidence lower bound on the data log-likelihood, which for this corruption reduces to a weighted cross-entropy over the masked positions:
\begin{equation}
  \mathcal{L}_{\mathrm{MDM}}
  = \mathbb{E}_{\mathbf{x}_0,\, t,\, \mathbf{x}_t}
    \bigl[ \ell(\mathbf{x}_0, \mathbf{x}_t) \bigr],
  \qquad
  \ell(\mathbf{x}_0, \mathbf{x}_t)
  = \tfrac{1}{t} \textstyle\sum_{i \in \mathcal{M}(\mathbf{x}_t)}
      -\log f_\theta^i(x_0^i \mid \mathbf{x}_t).
  \label{eq:mdm_app}
\end{equation}
The factor $1/t$ is the weight the bound takes under the linear schedule $\alpha_t = 1 - t$.
It divides the sum by the $Lt$ positions masked in expectation, so heavily masked sequences do not dominate the loss \citep{bethune2026design}.
Each training step needs a single denoiser call.
It draws a clean sequence $\mathbf{x}_0$ and a level $t \sim \mathcal{U}[0,1]$, masks each token independently with probability $1 - \alpha_t$, and computes the loss on the predictions at the masked positions.

\paragraph{Inference in MDMs.}
With a perfect denoiser, revealing one position at a time samples the joint distribution exactly by the chain rule, since each new token is conditioned on all tokens committed so far.
Revealing several positions at once is faster, but samples from the product of their marginals instead \citep{webb2026limitsconfidencediffusion}.
We index these steps by $j$ and write $\mathbf{x}_{t_j}$ for the sequence after step $j$, whose masking ratio $t_j = |\mathcal{M}(\mathbf{x}_{t_j})| / L$ gives its level.
Each step applies an \emph{unmasking policy} $g$ that selects $S = g(\mathbf{x}_{t_j}) \subseteq \mathcal{M}(\mathbf{x}_{t_j})$ and fills every $i \in S$ with a draw from $f_\theta^i(\cdot \mid \mathbf{x}_{t_j})$, giving $\mathbf{x}_{t_{j+1}}$.

Revealing positions with more certain marginals keeps the product closer to the joint and gives better samples at a fixed number of steps \citep{nie2025llada,wu2025fastdllm}.
Unmasking policies include hand-designed rules \citep{kim2025train,ben2026accelerated,ye2025dream} and learned policies \citep{jazbec2025learning}.
The \emph{top-$u$} policy \citep{chang2022maskgit} reveals the $u$ masked positions with the highest confidence $c_i = \max_{v \in \mathcal{V}} f_\theta^i(v \mid \mathbf{x}_{t_j})$.
For large MDMs, Fast-dLLM~\citep{wu2025fastdllm,wu2026fast} is an effective policy that reveals every position with $c_i \geq \tau$, or the single most confident position if none exceeds the threshold.
The number of positions revealed per step thus follows the confidence of the marginals, and every step reveals at least one token, so generation terminates.
\citet{wu2025fastdllm} pair it with semi-autoregressive decoding, which generates the sequence block by block \citep{arriola2025block}.
Together, the two preserve generation quality at far fewer denoiser calls per sample, or number of function evaluations (NFEs).

\clearpage
\section{Training procedure and experimental setup}
\label{app:experimental_setup}

In this section, we describe the data, task construction, training protocol, and evaluation procedure for each experimental setting.
Unless stated otherwise, comparisons within a setting use the same data, tokenizer, optimizer configuration, sequence length, and number of updates; only the construction of the training sequences changes between training procedures.

\subsection{\method training procedure}
\label{app:pumba_alg}

In this section, we give the pseudocode of \method (\Cref{alg:train}).

\begin{algorithm}[ht]
  \caption{\method training for one sample.
  In practice, we batch and advance several trajectories in parallel, each restarting from a fresh sample once fully revealed.}
  \label{alg:train}
  \begin{algorithmic}[1]
    \Require Carry-augmented denoiser $f_{\mathrm{carry}}$ (\Cref{eq:carry}); unmasking policy $g$ (\Cref{sec:puma}); unroll window $W$
    \Repeat
      \State Draw $\mathbf{x}_0$ from the dataset;\quad $\mathbf{x} \gets (m, \ldots, m)$;\quad $h \gets 0$
        \Comment{Fresh trajectory, fully masked}
      \While{$\mathcal{M}(\mathbf{x}) \neq \emptyset$} \Comment{One window, and one optimizer step, per iteration}
        \State $\mathcal{L} \gets 0$;\quad $w \gets 0$
        \While{$w < W$ \textbf{and} $\mathcal{M}(\mathbf{x}) \neq \emptyset$}
          \State $\bigl(f_\theta(\cdot \mid \mathbf{x}, h),\ h'\bigr) \gets f_{\mathrm{carry}}(\mathbf{x}, h)$
            \Comment{Forward pass}
          \State $\mathcal{L} \gets \mathcal{L} + \ell(\mathbf{x}_0, \mathbf{x}, h)$ \Comment{\Cref{eq:mdm}, with the carry as extra input}
          \State $S \gets g(\mathbf{x}, h)$ \Comment{Position selection: no gradient through $g$}
          \State $x^i \gets x_0^i$ for $i \in S$ \Comment{Revealed positions are teacher-forced}
          \State $h \gets h'$;\quad $w \gets w + 1$ \Comment{Gradient flows through the carry}
        \EndWhile
        \State Update $\theta$ from $\nabla_\theta (\mathcal{L} / w)$
        \State $h \gets \mathrm{sg}[h]$ \Comment{Gradient does not cross the window boundary}
      \EndWhile
    \Until{the training budget is exhausted}
  \end{algorithmic}
\end{algorithm}

\paragraph{Clarification on the loss.} The loss of each window is averaged over its steps, which keeps the update scale independent of $W$.
Since the $1/t$ factor of \Cref{eq:mdm} normalizes each step's loss per masked token, all steps of a window contribute equally regardless of their masking ratio.

\subsection{Progressive unmasking: the two constructions}
\label{app:pu_details}

The LLaDA-8B runs (\Cref{sec:exp_llada}) use the confidence-threshold construction: each step commits every masked position whose confidence exceeds $\tau$, and at least $u$ positions.
The TinyGSM runs (\Cref{sec:exp_tinygsm}) instead follow the construction of \citet{kim2026puma} as published.
\Cref{tab:pu_constructions} compares the two, and the paragraphs below give the details.
In both, confidence is the maximum probability over the vocabulary, computed in fp32 from the forward pass that produces the loss, and committed positions receive the ground-truth token.

\begin{table}[ht]
  \centering
  \caption{The two progressive unmasking constructions.
  $L$ is the number of maskable positions of a sample and $K$ the number of stages.
  A trajectory is at stage $p$ when a fraction in $[p/K,(p+1)/K)$ of its $L$ positions is revealed.}
  \label{tab:pu_constructions}
  \small
  \renewcommand{\arraystretch}{1.08}
  \begin{tabular}{@{}P{0.20\linewidth}P{0.38\linewidth}P{0.36\linewidth}@{}}
    \toprule
    & TinyGSM, \citet{kim2026puma} & LLaDA-8B, \Cref{sec:exp_llada} \\
    \midrule
    Start of a trajectory & A random fraction $r\sim\mathcal{U}[0,1/K)$ of the $L$ positions revealed, chosen uniformly at random & Fully masked \\
    Progress index & Stage $p\in\{0,\dots,K-1\}$ & None \\
    Positions committed per step & Enough to reach $\mathrm{round}(rL)$ revealed positions, $r\sim\mathcal{U}[(p+1)/K,(p+2)/K)$, capped at $L-1$ & The $u$ most confident \\
    Confidence threshold & Every other position above $\tau=0.9$; the stage is then recomputed from the revealed fraction & Every other position above $\tau=0.9$ \\
    Schedule & $K$ raised stepwise from 12 to 42 (\Cref{tab:tinygsm_hparams}) & $u$ fixed per run (\Cref{tab:llada_postsft_hparams}) \\
    Retirement & After the step taken at stage $K-1$ & Once the response and its EOS are fully revealed \\
    Maskable region & Response, including the EOS padding to 512 tokens & Response, its first EOS, and padding (\Cref{app:setup_llada_sft}) \\
    \bottomrule
  \end{tabular}
\end{table}

\paragraph{The stage-indexed construction of \citet{kim2026puma}, used for TinyGSM.}
The $K$ stages split the revealed fraction of a trajectory into equal intervals: a trajectory is at stage $p$ when a fraction in $[p/K,(p+1)/K)$ of its maskable positions is revealed.
Each step nominally advances a trajectory by one stage, so a trajectory lasts about $K$ steps.
A new trajectory starts at stage $0$ with a small fraction of its positions already revealed; these positions are drawn uniformly at random, not by confidence.
A step taken at stage $p$ samples a target fraction from the next interval and commits the most confident masked positions until that target is met.
The number of positions committed per step is therefore random, with mean close to $L/K$.
The target is capped at $L-1$, so the reveal step alone never completes a trajectory.
Every other masked position with confidence above $\tau$ is then committed as well, and the stage is recomputed from the revealed fraction.
A trajectory can therefore skip stages, although retirement is still decided by the stage before this recomputation.
The maskable region includes the EOS padding that fills the canvas to 512 tokens (the TinyGSM sequence length), so the padding counts toward $L$ and enters the loss.
In practice, the threshold commits most of the padding within about one step, so a trajectory lasts far fewer than $K$ steps: about $3.5$ at $K=12$ and $7$ at $K=42$.

\paragraph{Our construction, used for LLaDA-8B.}
Every trajectory starts fully masked, and there is no stage index and no schedule.
The prompt is always visible and never masked.
Each step commits the $u$ most confident masked positions, together with every other masked position whose confidence exceeds $\tau=0.9$.
This is the same set as committing every position above $\tau$ and topping up with the next most confident positions until at least $u$ are committed.
A trajectory is retired once its response and EOS are fully revealed, and a fresh, fully masked sample takes its place.
On the full canvas, padding inside the attended window is also maskable, so some of the $u$ positions committed at a step can be padding, which does not enter the loss.

\paragraph{Why we depart from the published construction.}
Progressive unmasking aims to train on the sequences that the inference policy visits (\Cref{sec:puma}).
The published construction approximates them with a schedule: a trajectory nominally lasts about $K$ steps, so each step commits about $L/K$ positions, whatever the length $L$ of the sample.
A confidence-threshold decoder instead commits a number of positions per step that does not depend on $L$, so longer responses take more steps.
The two agree only when $L$ is nearly the same for every sample.
On TinyGSM, $L$ includes the EOS padding to 512 tokens and varies only with the prompt length, so we keep the published construction and its curriculum on $K$, and our PU baseline is the recipe of \citet{kim2026puma}.
On Dolci-Instruct-SFT, the response length ranges from 14 tokens at the 10th percentile to 778 at the 90th (\Cref{fig:dolci_lengths}).
At $K=42$, the final value of that curriculum, a 14-token response would receive about $0.3$ commits per step, so most of its steps would commit nothing beyond the positions above $\tau$, while a 778-token response would receive about $18.5$ commits per step.
No single $K$ therefore matches the decoding policy across the mixture.
Our construction drops the schedule and takes its reveal rule from the decoding policy itself: every trajectory starts fully masked, as decoding does, and each step commits by the same confidence-threshold rule, with $\tau$ and $u$ as its only parameters.
It also removes the random reveal at stage $0$ and the sampled commit count, which have no counterpart at inference.

\subsection{TinyGSM}
\label{app:setup_tinygsm}

\paragraph{Dataset and task.}
The TinyGSM experiments provide the controlled small-scale setting used to isolate the contribution of progressive unmasking, the on-trajectory hidden-state carry, and BPTT.

We train on TinyGSM~\citep{liu2023tinygsm}, holding out $2\%$ for validation, and evaluate on the $1{,}319$ problems of the GSM8K test set~\citep{cobbe2021gsm8k}.
The prompt is always visible; only the response is masked and enters the loss, as typically adopted in the SFT settings~\citep{nie2025llada,ye2025dream}.
Each TinyGSM response is a Python function, so we score a generation by executing the function it defines in a sandbox with a 1-second limit and comparing its return value with the GSM8K answer, exactly for integers and to within $10^{-3}$ for floats.
We report this accuracy with EMA weights.
Decoding is greedy and reveals the $u$ most confident masked positions at each step.

\paragraph{Compared procedures.}
All diffusion variants share the architecture and start from random initialization.
For $W>1$, we lower the per-GPU batch size and accumulate gradients so that the effective batch stays at 256.
MDM $\times2$ and $\times8$ train for 40 and 160 epochs.
The AR model has the same size ($\approx$125M parameters), uses causal attention, and is trained with next-token cross-entropy on the same data, averaged over the response tokens up to and including the first EOS.
The denoiser, tokenizer, output parameterization, and loss shape are held fixed across the diffusion runs, whose loss also covers the EOS padding.
\citet{jo2026loopholing} propose carry dropout, which uses the carry only a fraction $p$ of the time during training and helps for $p \in [0.5, 0.9]$ in their experiments.
Our TinyGSM runs do not use it ($p=1$), so the carry is always passed on.
We also follow the curriculum on $K$ of \citet{kim2026puma}: $K$ starts at 12 and rises by 3 every 30k updates from update 60k, reaching 42 at update 330k.
Each GPU keeps 32 trajectories, which start at stages spread evenly over $\{0,\dots,K-1\}$.
Every change of $K$ re-initializes all trajectories this way, with a zero carry.

\begin{table}[ht]
  \centering
  \caption{Hyperparameters for the TinyGSM experiments.}
  \label{tab:tinygsm_hparams}
  \small
  \renewcommand{\arraystretch}{1.08}
  \begin{tabular}{@{}P{0.34\linewidth}P{0.60\linewidth}@{}}
    \toprule
    Hyperparameter & Value \\
    \midrule
    Architecture & Bidirectional transformer, 14 layers, hidden size 512, MLP size 1536, 8 heads, tied embeddings ($\approx$125M parameters) \\
    Tokenizer & Qwen2, restricted to its first 151{,}645 ids; id 151{,}644 serves as the mask token \\
    Maximum sequence length & 512 \\
    Optimizer & AdamW, learning rate $3\times10^{-4}$, linear warmup over 1000 updates, then cosine decay to 0 \\
    Weight decay / gradient clipping & 0.01 / 1.0 \\
    Batch size / training length & 256 sequences / 20 epochs ($\approx$907k updates) \\
    EMA decay & 0.9999 \\
    PU schedule & Stage-indexed, $K$ raised stepwise from 12 to 42 over the first 330k updates; confidence threshold 0.9 (\Cref{app:pu_details}) \\
    Carry & Last hidden state before the output projection, added to the token embeddings through a zero-initialized LayerNorm; zero at the first step; stored in fp32. Dropout $p=1$ \\
    BPTT window $W$ & $\{1,2,4,8\}$ \\
    \bottomrule
  \end{tabular}
\end{table}

\FloatBarrier
\subsection{LLaDA-8B}
\label{app:setup_llada}

The LLaDA-8B experiments follow the two-stage procedure of \Cref{sec:exp_llada}.
The SFT stage fine-tunes LLaDA-8B-Base~\citep{nie2025llada} with the MDM objective.
The Post-SFT stage continues from an SFT checkpoint with one of the compared training procedures.
\Cref{tab:llada_common} lists the shared settings, \Cref{app:setup_llada_sft,app:setup_llada_postsft} those specific to each stage, and \Cref{app:setup_llada_eval} the evaluation.

\paragraph{Dataset.}
Both stages train on Dolci-Instruct-SFT from the OLMo~3 post-training pipeline~\citep{olmo2025olmo}.
The mixture contains instruction-following, mathematical, code, conversational, and tool-use examples.
We use the full mixture, about 2.15M (prompt, response) pairs.
Prompts are rendered as plain-text role labels (\texttt{user:}, \texttt{assistant:}) without a BOS token, and the response is followed by a trainable EOS that supplies the stop signal.
Each sample is truncated to a 4096-token canvas by reserving the response and its EOS first and left-truncating the prompt into the remainder.
Truncation therefore never removes response tokens unless the response alone exceeds the canvas (which almost never happens; only 1645 samples); it removes 16.2\% of all prompt tokens, concentrated in the 2.5\% of samples that do not fit the canvas.
\Cref{fig:dolci_lengths} shows the prompt and response lengths before truncation, and \Cref{tab:llada_token_budget} reports the token budget after it.

\begin{figure}[ht]
  \centering
  \includegraphics[width=\appscale\linewidth]{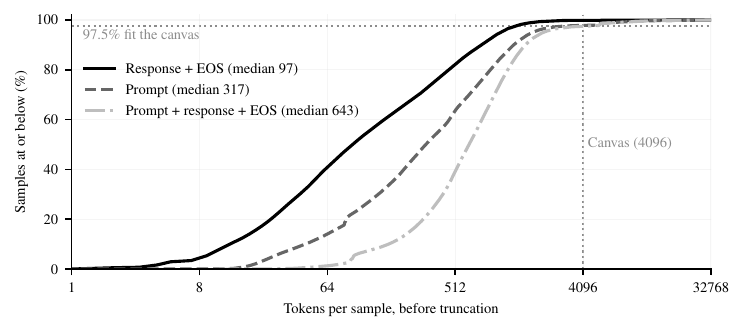}
  \caption{Empirical CDFs of prompt, response, and total length per sample in Dolci-Instruct-SFT, before truncation to the 4096-token canvas (vertical dotted line).
  Response length includes the terminal EOS.
  The horizontal dotted line marks the 97.5\% of samples that fit the canvas.}
  \label{fig:dolci_lengths}
\end{figure}

The benchmarks in \Cref{app:setup_llada_eval} are used only for downstream assessment.
We run no separate contamination check: every compared run trains on the same data, so any overlap with the benchmarks is shared by all of them.

\paragraph{Objective.}
The prompt is never masked, but not every masked position enters the loss: the response and its first EOS always do, and whether padding does depends on the setting (\Cref{app:setup_llada_sft}).
The MDM objective weights each sample by $1/t$ (\Cref{eq:mdm}), and we cap this factor at $\gamma=5$, following the min-SNR weighting of \citet{hang2023efficient}.
In our runs, the cap slightly stabilizes training.
An auxiliary $z$-loss is weighted $10^{-5}$, and the loss is normalized by the number of positions it is computed on, not by the maskable ones.

\begin{table}[ht]
  \centering
  \caption{Settings shared by the SFT and Post-SFT stages of LLaDA-8B.}
  \label{tab:llada_common}
  \small
  \renewcommand{\arraystretch}{1.08}
  \begin{tabular}{@{}P{0.34\linewidth}P{0.60\linewidth}@{}}
    \toprule
    Setting & Value \\
    \midrule
    Model & LLaDA-8B-Base architecture and tokenizer; hidden size 4096; 126{,}464 tokens, with the existing \texttt{<|mdm\_mask|>} as the mask token \\
    Data & Dolci-Instruct-SFT, full mixture, about 2.15M (prompt, response) pairs \\
    Canvas & 4096 tokens; response and EOS reserved first, prompt left-truncated \\
    Optimizer & AdamW ($\beta_1=0.9$, $\beta_2=0.95$, $\epsilon=10^{-8}$), weight decay $0.01$, global gradient-norm clipping at $1.0$ \\
    Learning-rate schedule & Linear warmup from $0$, then cosine decay to $0$ \\
    Precision & bfloat16 parameters, gradient reduction, and outputs \\
    Global batch size & 128 sequences, 2 per GPU on 64 GPUs; no gradient accumulation \\
    Parallelism & Hybrid-sharded data parallelism: parameters sharded within each node of 8 GPUs, replicated across 8 nodes \\
    Hardware & 8 nodes of 8 NVIDIA H100 80GB, except where \Cref{tab:llada_sft,tab:llada_postsft_hparams} state B200 \\
    MDM loss weighting & $\min(1/t, 5)$; $z$-loss $10^{-5}$; normalized by the positions in the loss \\
    \bottomrule
  \end{tabular}
\end{table}

\FloatBarrier
\subsubsection{SFT stage}
\label{app:setup_llada_sft}

The SFT stage trains with the MDM objective in the same attention pattern the model later decodes with (\Cref{tab:llada_sft}).
Full-canvas SFT attends bidirectionally over the whole canvas.
As in the SFT of \citet{nie2025llada}, each local batch is padded with EOS only up to its longest sequence rather than to the full canvas.
In preliminary runs, we observed this made performance vary less with the canvas length.
Unlike \citet{nie2025llada}, we keep these padding tokens out of the loss: they are attended and can be masked, but only the first EOS after the response is trained on.
In LLaDA, EOS and padding are the same token, and when we also trained on all the padding of the full canvas, the model collapsed to EOS or to very short responses, an effect \citet{nie2025llada} also report.
Block diffusion SFT follows the vectorized block MDM algorithm of \citet{arriola2025block} with blocks of 32 tokens.
Each sample holds a clean and a noised copy of its response in one sequence of twice the length, and the attention mask lets each noised block attend to itself and to the clean blocks before it, so all blocks train in one forward pass.
We interleave the two copies block by block, which leaves this attention pattern unchanged.
Each response is rounded up to a whole number of blocks with EOS, and this filler enters the loss.
Both settings select the peak learning rate by a sweep at $1.5$ epochs (\Cref{app:llada_sft_schedule,app:llada_semiar_sft}).

\begin{table}[ht]
  \centering
  \caption{SFT stage of LLaDA-8B.
  Wall-clock time is the training run alone, without evaluation.}
  \label{tab:llada_sft}
  \small
  \renewcommand{\arraystretch}{1.08}
  \begin{tabular}{@{}P{0.26\linewidth}P{0.33\linewidth}P{0.33\linewidth}@{}}
    \toprule
    & Full canvas & Block diffusion \\
    \midrule
    Attention & Bidirectional & Block-causal, bidirectional within blocks of 32 \\
    Peak learning rate & $5\times10^{-6}$ & $5\times10^{-6}$ \\
    Warmup & 2000 updates & 2000 updates \\
    Training length & 1.5 epochs, 25{,}195 updates; 3-epoch control & 6 epochs, 100{,}781 updates; 1.5- and 3-epoch controls \\
    Padding & Attended and maskable, not in the loss & EOS filler to the block edge, in the loss \\
    Wall-clock (reported run) & 9.4\,h on 64 H100 & 48.6\,h on 64 B200 \\
    \bottomrule
  \end{tabular}
\end{table}

\subsubsection{Post-SFT stage}
\label{app:setup_llada_postsft}

Every Post-SFT run trains for 4000 updates at a peak learning rate of $2.5\times10^{-6}$, with a linear warmup over the first 200 updates (5\% of the run) and cosine decay to $0$.
This is half the SFT peak, and together with the new warmup it keeps the continuation from the SFT checkpoint as smooth as the 4000-update horizon allows.
Full-canvas runs start from the 1.5-epoch full-canvas SFT checkpoint, and block diffusion runs start from the 6-epoch block diffusion SFT checkpoint.
Within each setting, every Post-SFT run therefore shares its warm start, data, recipe, and number of updates, and only the construction of the training sequences changes.
The Post-SFT MDM control uses the same recipe with the MDM training algorithm.

Progressive unmasking follows the construction in \Cref{app:pu_details}.
Each GPU keeps two trajectories, so that the 64 GPUs advance 128 trajectories per update, and every trajectory advances at every update.
Under block diffusion training, each block of 32 positions is its own trajectory, and all blocks of a sample advance in lockstep, so a block is retired after at most $\lceil 32/u \rceil$ steps.
The trajectories are not checkpointed, so a resumed run starts from fresh, fully masked ones.
The PU loss weights each sequence by $\min(1/t_k, 5)$, where $t_k$ is the realized masked fraction of the positions in the loss: the response up to and including its first EOS on the full canvas, and each block separately under block diffusion.
It is rescaled by the fraction of sequences in the batch that still have masked positions in the loss, and it is averaged over the $W$ forward passes of a window.
The carry is the last hidden state after the final normalization, one vector per position, and it is added to the token embeddings through a LayerNorm whose weight and bias start at zero.
This LayerNorm is new at the start of Post-SFT and trains at the same learning rate as the rest of the model.
The carry is zero at the first step of a trajectory and is stored in fp32 between updates; the sum with the token embeddings is cast back to bfloat16 before the first transformer block.
Carry dropout is not used ($p=1$).

\begin{table}[ht]
  \centering
  \caption{Post-SFT stage of LLaDA-8B.
  Wall-clock time is the training run alone, without evaluation, on 64 H100 unless stated.}
  \label{tab:llada_postsft_hparams}
  \small
  \renewcommand{\arraystretch}{1.08}
  \begin{tabular}{@{}P{0.26\linewidth}P{0.33\linewidth}P{0.33\linewidth}@{}}
    \toprule
    & Full canvas & Block diffusion \\
    \midrule
    Warm start & 1.5-epoch SFT (25{,}195 updates) & 6-epoch SFT (100{,}781 updates) \\
    Learning rate & \multicolumn{2}{p{0.66\linewidth}}{Peak $2.5\times10^{-6}$, 200 warmup updates, cosine to $0$, 4000 updates} \\
    Confidence threshold $\tau$ & $0.9$ & $0.9$ \\
    Reveal count $u$ & $\{32, 64\}$ reported & $\{8, 16\}$ per block reported \\
    BPTT window $W$ & $\{1,2,4,8\}$ & $\{1,2,4\}$ at $u=8$; $\{1,2\}$ at $u=16$ \\
    Carry & \multicolumn{2}{p{0.66\linewidth}}{Final hidden state; zero-initialized LayerNorm; zero at the first step; fp32 storage; $p=1$} \\
    Wall-clock & MDM 1.7\,h; at $u=32$: PU 1.8\,h, $W=1$ 2.2\,h, $W=2$ 4.2\,h, $W=4$ 7.7\,h, $W=8$ 14.4\,h & MDM 3.1\,h; at $u=16$: PU 3.0\,h, $W=1$ 3.0\,h, $W=2$ 5.6\,h; at $u=8$: $W=4$ 6.9\,h on 64 B200 \\
    \bottomrule
  \end{tabular}
\end{table}

\subsubsection{Evaluation}
\label{app:setup_llada_eval}

We measure verifiable instruction following with IFEval~\citep{zhou2023ifeval}, mathematical reasoning with GSM8K~\citep{cobbe2021gsm8k}, and program synthesis with MBPP~\citep{austin2021mbpp}.
All three run in lm-eval~0.4.9.1~\citep{eval-harness}.
We report prompt-level strict accuracy on IFEval (0-shot), exact match with flexible answer extraction on GSM8K (8-shot), and pass@1 on MBPP (3-shot).
The \emph{score} is the mean of the three.
Decoding uses the Fast-dLLM confidence-threshold policy~\citep{wu2025fastdllm}.
It runs left to right in blocks of 32 tokens for every model, at temperature $T=0$, and sweeps the confidence threshold $\tau\in\{0.7,0.8,0.9,1\}$.
$\tau=1$ recovers one-token-per-step confidence decoding.
The generation budget is 1024 tokens for GSM8K and MBPP and 1280 tokens for IFEval, the values lm-eval sets per task.
Carry models pass the carry from each decoding step to the next.
On the full canvas, the carry also passes from one block to the next.
Under block diffusion, we reset the carry to zero at the start of every block, to match training, where the vectorized layout of \Cref{app:setup_llada_sft} advances all blocks of a sample in parallel from a zero carry.

Prompts use one of two formats, both plain text without special tokens such as BOS or EOS.
The \emph{plain prompt format} concatenates the few-shot demonstrations and the query into a single string, as lm-eval does by default.
The \emph{chat prompt format} renders them as alternating turns with the same \texttt{user:} and \texttt{assistant:} labels as our SFT data, and ends with \texttt{assistant:}.
We report all results in the chat prompt format, except the base-model check in \Cref{app:llada_comparison}.

\begin{table}[ht]
  \centering
  \caption{Token budget of LLaDA-8B SFT, per epoch over Dolci-Instruct-SFT, with the total of the reported run in parentheses: 1.5 epochs on the full canvas (25{,}195 updates of 128 sequences) and 6 epochs under block diffusion (100{,}781 updates).
  Counts are measured over one epoch of the training dataloader and, under block diffusion, of the vectorized layout the training step builds from it.
  \emph{Content} is prompt plus response plus the terminal EOS after truncation to the 4096-token canvas, \emph{EOS filler} completes the last block of each response, \emph{in the loss} counts the positions the loss is computed on, and \emph{attended} counts the positions the model attends over; under block diffusion, these are the context and the clean and noised copies of the response.
  \emph{Total canvas} counts every sample as a full canvas, whatever part of it the sample uses: 4096 positions on the full canvas and 8192 under block diffusion, whose layout holds the clean and noised copies side by side.}
  \label{tab:llada_token_budget}
  \small
  \renewcommand{\arraystretch}{1.08}
  \begin{tabular}{@{}lcc@{}}
    \toprule
    & Full canvas & Block diffusion \\
    \midrule
    Content (prompt + response + EOS) & 1.81\,B (2.71\,B) & 1.81\,B (10.84\,B) \\
    \quad of which prompt & 1.22\,B (1.83\,B) & 1.22\,B (7.34\,B) \\
    \quad of which response + EOS & 0.58\,B (0.88\,B) & 0.58\,B (3.50\,B) \\
    EOS filler & --- & 0.03\,B (0.20\,B) \\
    In the loss & 0.58\,B (0.88\,B) & 0.62\,B (3.71\,B) \\
    Attended & 2.32\,B (3.49\,B) & 2.46\,B (14.75\,B) \\
    Total canvas & 8.81\,B (13.21\,B) & 17.61\,B (105.68\,B) \\
    \bottomrule
  \end{tabular}
\end{table}

\subsection{Compute}
\label{app:compute}

The runs reported in this paper used about $85{,}000$ H100 and $7{,}300$ B200 GPU-hours.
The TinyGSM experiments account for about $36{,}000$ H100 GPU-hours, and the LLaDA-8B experiments for about $49{,}000$ H100 and $7{,}300$ B200 GPU-hours, of which evaluation takes about $24{,}000$ H100 GPU-hours.

\clearpage
\section{Additional TinyGSM results}
\label{app:tinygsm_results}

\subsection{Accuracy and \texorpdfstring{pass@$k$}{pass@k}}

\Cref{fig:tinygsm_acc} complements \Cref{tab:tinygsm_compose} with the accuracy of each variant across decoding budgets $u$ and during training at $u{=}2$.

\begin{figure}[ht]
  \centering
  \includegraphics[width=0.9\linewidth]{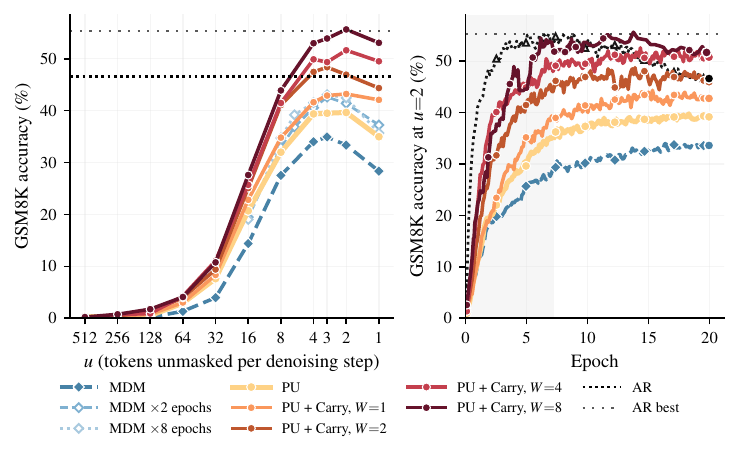}
  \caption{GSM8K accuracy on TinyGSM.
  \textit{Left}: final checkpoint against the number of tokens revealed per step, $u$.
  \textit{Right}: accuracy at $u{=}2$ during training; the gray area marks the annealing phase of the PUMA recipe.
  Dotted lines: AR at its final and best checkpoints.}
  \label{fig:tinygsm_acc}
\end{figure}

We investigate the effect of the number of sampling tokens per denoising step $u$ in \Cref{fig:tinygsm_epochs} during training. We observe that the trend and ordering of the models remain the same, with carry (+ $W{=}8$) performance matching AR's best performance. Note that AR's performance decays with more epochs of training as expected. We also study pass@$k$ as an alternative metric to accuracy. \citet{ni2026flexibility} show that pass@$k$ can be used to compare sample diversity across different methods, such as AR and MDM. We change the sampling temperature $T \in \{0.2, 0.7, 1.0\}$. \Cref{fig:tinygsm_passk} shows the result of this experiment. We observe that PU, PU with carry, and BPTT have a positive effect on this metric. While AR seems to be less sensitive to temperature and performs better at higher temperatures, we see that at lower temperatures, BPTT with higher values of $W$ can outperform AR.

\begin{figure}[ht]
  \centering
  \includegraphics[width=\ifapplebuild0.8\else0.6\fi\linewidth]{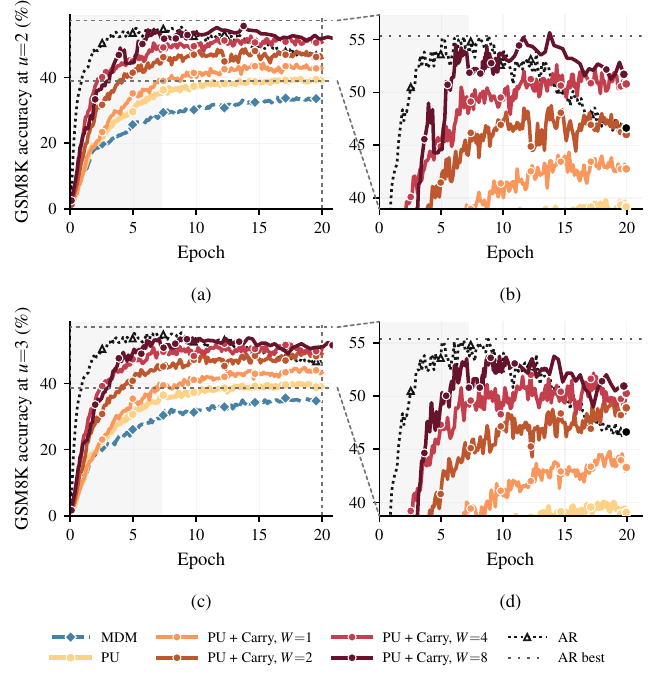}
  \caption{GSM8K accuracy during training at $u{=}2$ (top) and $u{=}3$ (bottom).
  \textit{Right}: zoom on the dashed window. \textit{Shaded}: ramp of the PU schedule.}
  \label{fig:tinygsm_epochs}
\end{figure}

\begin{figure}[ht]
  \centering
  \includegraphics[width=\ifapplebuild0.85\else0.7\fi\linewidth]{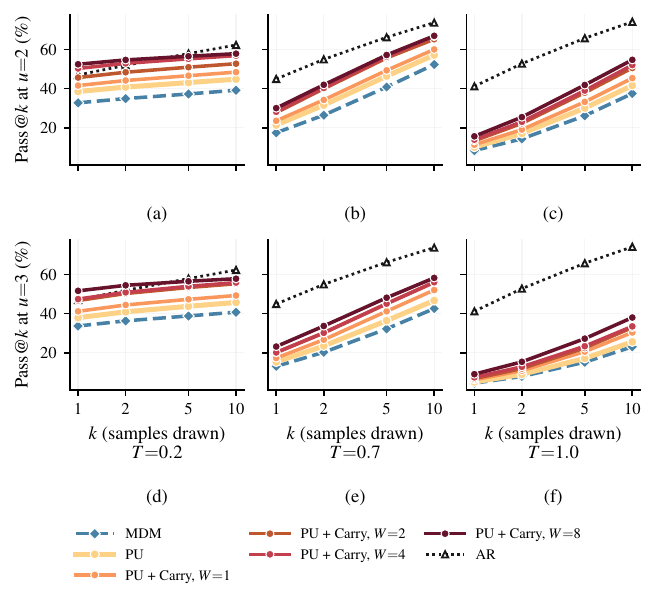}
  \caption{pass@$k$ from 100 samples per problem~\citep{chen2021evaluating}, at $u{=}2$ (top) and $u{=}3$ (bottom).
  AR leads at higher temperatures while carry with $W{=}8$ has comparable performance at lower temperatures; the ordering among diffusion variants is unchanged with PU improving over MDM, carry and BPTT having more positive effects in addition to PU.}
  \label{fig:tinygsm_passk}
\end{figure}

\FloatBarrier

\subsection{Carry and gradient estimators}
\label{app:chain_ablation}

\Cref{fig:tinygsm_chain_curves} shows the training curves behind \Cref{tab:tinygsm_chain_ablation}, together with the runs that separate the carry from the multi-step objective, and \Cref{fig:tinygsm_chain_budget} evaluates their latest checkpoints across decoding budgets $u$.
All runs keep the PUMA recipe at its initial setting instead of annealing it (\Cref{sec:carry}).
They share the model, an effective batch size of 256, and a learning rate of $3\times10^{-4}$, and are evaluated zero-shot on GSM8K.
We sweep the Gumbel-softmax temperature over $\{0.5, 1, 2\}$ and the REINFORCE weight over $\{0.1, 1, 10\}$, with a leave-one-out baseline and normalized advantages, i.e., to reduce the variance of gradients, each sequence's reward is compared with the average reward of the other sequences in the batch, and the resulting differences are rescaled to unit standard deviation.
The carry without BPTT ($W{=}1$) tracks PU throughout training, and supervising two steps per update without a carry adds little, whereas the carry trained by BPTT at $W{=}2$ and $W{=}4$ separates from both early in training (\Cref{fig:tinygsm_chain_curves}, \textit{left}).
No discrete bridge exceeds the two-step objective without a carry (\textit{right}), so the gain of \Cref{tab:tinygsm_chain_ablation} comes from training the carry through the following steps rather than from supervising several steps per update.
The same ordering holds at every decoding budget we evaluate, and the gap narrows at $u{=}16$, where every variant degrades (\Cref{fig:tinygsm_chain_budget}).

\begin{figure}[ht]
  \centering
  \includegraphics[width=\appscale\linewidth]{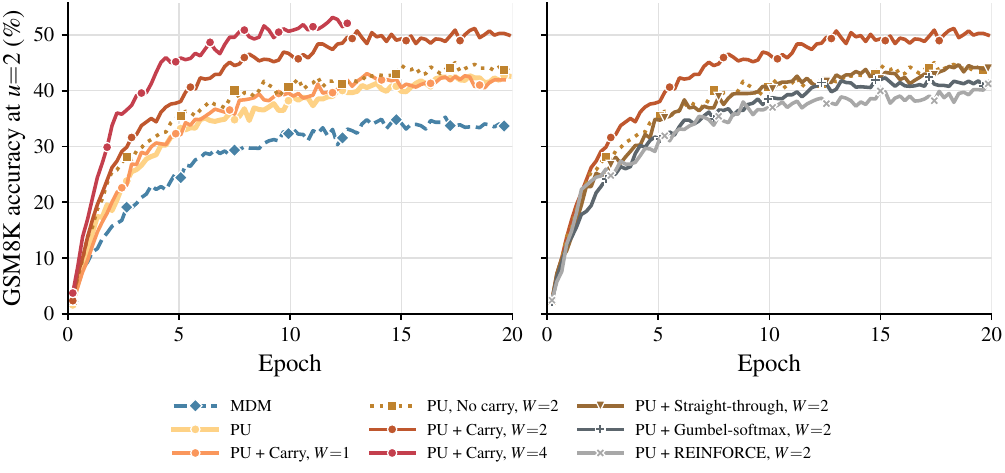}
  \caption{GSM8K accuracy at $u{=}2$ during training.
  \textit{Left}: MDM, PU, the carry without BPTT ($W{=}1$), a two-step objective without a carry, and the carry with BPTT ($W{=}2,4$).
  \textit{Right}: the carry and three gradient estimators at $W{=}2$.
  The $W{=}4$ run is at $63\%$ of its training budget.}
  \label{fig:tinygsm_chain_curves}
\end{figure}

\begin{figure}[ht]
  \centering
  \includegraphics[width=0.5\linewidth]{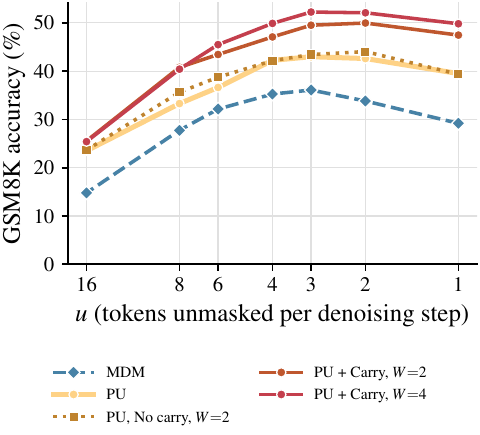}
  \caption{GSM8K accuracy against the number of tokens revealed per step, $u$, at the latest checkpoint of the runs of \Cref{fig:tinygsm_chain_curves}.}
  \label{fig:tinygsm_chain_budget}
\end{figure}

\FloatBarrier

\subsection{From teacher forcing to student forcing}
\label{app:student_forcing}
PU writes ground-truth tokens at the revealed positions (\Cref{sec:puma}), whereas at inference these positions hold the model's own predictions.
We test whether scheduled sampling~\citep{bengio2015scheduled}, the standard remedy for this mismatch, improves PU on TinyGSM.

\paragraph{Setup.}
We train PU without the carry, with the architecture and training budget of \Cref{tab:tinygsm_hparams}.
For simplicity, we fix the number of stages to $K{=}12$ instead of raising it from 12 to 42 as in \citet{kim2026puma}, so that the student-forcing schedule is the only schedule that changes during training.
Each revealed position receives the ground-truth token with probability $\varepsilon$ and the model's argmax prediction otherwise, with a straight-through gradient through the predicted tokens.
The ratio $\varepsilon$ follows the inverse-sigmoid decay of \citet{bengio2015scheduled} with a floor,
\[
  \varepsilon(s) = \max\Bigl(0.1,\ \frac{k}{k + \exp(s/k)}\Bigr),
\]
where $s$ is the optimizer step after a $1{,}000$-step warmup and $k$ sets how fast training moves from teacher to student forcing.
We sweep eight values of $k$ from $15{,}542$ to $98{,}862$, with one run per value, and report GSM8K accuracy at $u{=}2$ as the mean of the last five evaluations.

\paragraph{Results.}
\Cref{fig:scheduled_sampling} shows that no schedule improves on the PU baseline ($39.0\%$).
Accuracy increases with the average ratio of ground-truth tokens $\bar{\varepsilon}$ and reaches the baseline only as $\bar{\varepsilon}$ approaches $1$.
The two closest runs ($37.5\%$ and $38.7\%$) have $\bar{\varepsilon} = 0.97$ and $0.99$: their schedules never reach the floor and are almost pure teacher forcing.
Below $\bar{\varepsilon} \approx 0.9$, every run falls under the MDM baseline ($34.3\%$), and at $\bar{\varepsilon} = 0.24$ accuracy drops to $13.7\%$.
In additional runs, detaching the predicted tokens instead of using the straight-through gradient changes accuracy by less than one point, so the drop comes from the predicted tokens in the context rather than from the gradient estimator.

\paragraph{Why the carry avoids this failure.}
Committed positions are never revisited, so a wrong predicted token stays in the context for the rest of the trajectory, while the loss still asks for $\mathbf{x}_0$ at the remaining masked positions.
The model is then trained to predict targets that contradict its own context.
The carry exposes each step to its predecessor's imperfect computation through a continuous input, while the visible tokens stay consistent with $\mathbf{x}_0$ (\Cref{sec:carry}).
The model can therefore learn how much to rely on the carry, and BPTT trains the predecessor to produce a useful one.

We do not rule out a student-forcing schedule that outperforms teacher forcing, for example one that remasks wrong commitments, but none of the schedules we tried did.

\begin{figure}[ht]
  \centering
  \includegraphics[width=0.8\linewidth]{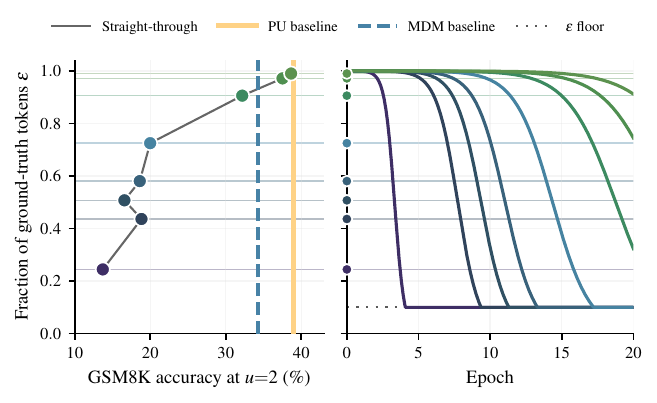}
  \caption{Scheduled sampling on PU (TinyGSM). \textit{Right}: ratio of ground-truth tokens $\varepsilon$ during training for each schedule; the dot on the $y$-axis marks its average $\bar{\varepsilon}$. \textit{Left}: GSM8K accuracy at $u{=}2$ against $\bar{\varepsilon}$, at the same height as the corresponding schedule. $\varepsilon=1$ is pure teacher forcing.}
  \label{fig:scheduled_sampling}
\end{figure}

\FloatBarrier

\subsection{Estimating the mask discrepancy}
\label{app:dmask}

This section details how we estimate $D_{\mathrm{mask}}$ (\Cref{def:dmask}) for a given checkpoint.

\paragraph{Masks and kernel.}
Prompt positions are never masked, so we restrict masks to the $N$ response positions of a validation problem, $\mathbf{m} \in \{0,1\}^N$, where $m^i = 1$ if position $i$ is masked.
The Hamming distance counts the positions on which two masks differ,
\[
  d_H(\mathbf{m}, \mathbf{m}') = \sum_{i=1}^{N} \bigl|m^i - m'^i\bigr|.
\]
We set the kernel bandwidth relative to the response length, $\sigma = 0.2\,N$, so that the kernel has the same scale for short and long responses. We evaluate the masking ratios $t \in \{0.1, 0.2, \ldots, 0.9\}$.

\paragraph{Inference side.}
We run the sampler used for evaluation: greedy decoding that, at each step, reveals the $u$ most confident masked positions and fills them with the most likely token.
For carry models, the carry starts at zero and is passed from step to step, as at evaluation.
This sampler is deterministic, so $P_{\mathrm{inf}}$ is a point mass at a single mask $\mathbf{m}_{\mathrm{inf}}$.
We take $\mathbf{m}_{\mathrm{inf}}$ to be the mask at the first step $j$ at which the masked fraction reaches $t$,
\[
  \frac{|\mathcal{M}(\mathbf{x}_{t_j})|}{N} \le t .
\]
For $u>1$, this mask can have up to $u-1$ fewer masked positions than $tN$.
This slightly increases $D_{\mathrm{mask}}$ in the same way for all methods.

\paragraph{Training side.}
Our MDM implementation first draws the number of masked positions and then masks a uniformly random subset of that size.
At ratio $t$, we therefore draw $N_{\mathrm{tr}} = 32$ masks, each masking a uniformly random subset of $\mathrm{round}(tN)$ response positions. The PU construction of \citet{kim2026puma} does not allow sampling a mask at a requested ratio: masks arise along trajectories whose reveal counts vary.
We therefore run the training construction itself on $\mathbf{x}_0$ and collect the masks it produces.
\begin{enumerate}[nosep, leftmargin=*]
  \item We create a pool of $2K$ trajectories, where $K$ is the number of stages the checkpoint was trained with.
  Every trajectory is a copy of $\mathbf{x}_0$ and starts at a random stage, as in training.
  \item We advance the pool with the training update: reveal the most confident positions, reveal every position whose confidence exceeds $0.9$, and restart a trajectory from a new random stage once it is complete.
  For carry models, the carry is passed along each trajectory; for $W>1$, the pool is advanced once per window of $W$ steps, as in \Cref{alg:train}.
  \item At every step, we record the mask of each trajectory and assign it to the nearest ratio $t$ on the grid, if it lies within $0.05$ of it.
  Each ratio keeps at most $N_{\mathrm{tr}} = 32$ masks.
\end{enumerate}
We run each pool for $K{+}1$ steps and use up to five independent pools, stopping once every ratio holds $32$ masks.
At the checkpoints we report, all ratios are full for all problems.

\paragraph{Estimator.}
Let $\mathbf{m}_1, \ldots, \mathbf{m}_{N_{\mathrm{tr}}}$ be the training masks at ratio $t$.
The unbiased MMD estimator~\citep{gretton2012kernel} compares the average kernel value within the training masks with the average kernel value between training and inference masks.
Since $P_{\mathrm{inf}}$ is a point mass, its within-set term equals $k(\mathbf{m}_{\mathrm{inf}}, \mathbf{m}_{\mathrm{inf}}) = 1$, and the estimator becomes
\begin{equation}
  \widehat{\mathrm{MMD}}_k^2
  = \frac{1}{N_{\mathrm{tr}}(N_{\mathrm{tr}}-1)}\sum_{i \neq j} k(\mathbf{m}_i, \mathbf{m}_j)
  \;+\; 1
  \;-\; \frac{2}{N_{\mathrm{tr}}}\sum_{i=1}^{N_{\mathrm{tr}}} k(\mathbf{m}_i, \mathbf{m}_{\mathrm{inf}}).
\end{equation}
The first term is large when the training masks are similar to each other.
The last term is large when they are similar to the inference mask.
$D_{\mathrm{mask}}(t)$ is the average of this estimate over $50$ problems from the validation split, using the EMA weights at the final checkpoint.

\paragraph{Binning control.}
MDM masks have exactly $\mathrm{round}(tN)$ masked positions, whereas PU masks are grouped within $\pm 0.05$ of $t$ and therefore have varying counts.
To check that this difference does not drive the comparison, we also build MDM masks with the same procedure as PU: we draw i.i.d.\ MDM masks at random ratios and group them in the same way.
At $u{=}2$, this re-binned control differs from MDM by at most $0.0015$, so we omit it from \Cref{fig:tinygsm_dmask}; the gap between MDM and PU is therefore not an artifact of the grouping.

\clearpage
\section{Local overfitting at small training \texorpdfstring{$u$}{u}}
\label{app:local_overfitting}

This section details the TinyGSM experiments behind \Cref{fig:local_overfitting} (\Cref{sec:puma}).
We describe the runs and how we follow individual sequences (\Cref{app:lo_setup}), report the sweep over $u$ and the ablations that separate the causes of the failure (\Cref{app:lo_results}), and explain why, at $u{=}2$, the loss of a sequence rises again before the sequence leaves the batch (\Cref{app:lo_rebound}).

\subsection{Setup}
\label{app:lo_setup}

\paragraph{Runs.}
We use the model, data, optimizer and batch size of \Cref{app:setup_tinygsm}, with three changes that isolate the effect of $u$.
First, each run lasts $3{,}000$ updates: the learning rate warms up linearly over $600$ updates and then stays at $3\times10^{-4}$, so that every insertion (below) happens at the same learning rate.
Second, the number of stages $K$ is fixed within a run, with $K \in \{231, 115, 57, 7, 4\}$, that is $K \approx L/u$ for $u \in \{2, 4, 8, 64, 128\}$ and the $L \approx 460$ maskable positions of a sample.
Third, the stage-indexed construction of \Cref{app:pu_details} runs without its confidence threshold $\tau$, so that no trajectory skips a stage: every trajectory has $J = K$ decoding steps, and every sequence stays in the batch for exactly $J$ updates.
The $32$ trajectories of a GPU start at distinct stages drawn at random from $\{0,\dots,K-1\}$ when $K > 32$, and cover all stages evenly otherwise.
There is no carry and no BPTT, and MDM follows the same schedule.
Each configuration is trained once.

\paragraph{Inserted sequences.}
We follow $16$ sequences from the validation split, which are not trained on otherwise.
Twelve of them are inserted into the batch, four at each of updates $750$, $1{,}500$ and $2{,}250$, and the other four are never inserted.
Under PU, an inserted sequence takes the place of the next sequence to enter the pool of one GPU and starts at stage $0$; it is then trained on for exactly $J$ consecutive updates, as one of the $256$ sequences of each batch.
Under MDM, it replaces one sequence of one batch, for a single update.
In a separate run at $u{=}2$, we also followed the next $16$ training sequences to enter the pool, at their natural times.
They respond like the sequences inserted in the same run, with a completion loss $65\%$ lower $20$ updates after entry, so inserted sequences behave like ordinary training sequences.

\paragraph{Completion loss.}
To measure how well the model predicts a sequence, we reveal its prompt and the first half of its answer, and average the cross-entropy over the rest of the answer, up to and including its first EOS; the padding that follows is not included in the loss.
This mask is the same throughout the run, and the loss is computed with the raw weights, not the EMA, after every update.
The validation completion loss averages the same quantity over $256$ other validation sequences, computed every $20$ updates, and we report its mean over the last $200$ updates.
The relative completion loss of an inserted sequence is its completion loss divided by its mean over the $20$ updates before insertion, divided by the same ratio for the validation completion loss, which removes the drift shared by all sequences.
\Cref{fig:local_overfitting} (\textit{middle}) averages its logarithm over the $12$ inserted sequences and smooths it over $9$ updates; the bands show one standard error.

\subsection{Results}
\label{app:lo_results}

\paragraph{Sweep over $u$.}
\Cref{tab:lo_sweep} summarizes the runs of \Cref{fig:local_overfitting}, and \Cref{fig:app_local_overfitting} (\textit{left}) plots their validation completion loss against $u$.
As $u$ decreases, each sequence stays longer in the batch and a run visits fewer distinct sequences, $3.6$k at $u{=}2$ against $768$k for MDM.
The inserted sequences are then fitted more deeply while they are in the batch, and the validation completion loss is higher: PU beats MDM when each sequence stays in the batch for only a few steps, but trails it for $u \leq 8$.
At every $u$, the gain has disappeared $300$ updates after the sequence leaves the batch.
The training loss moves in the opposite direction (\Cref{fig:local_overfitting}, \textit{right}): at $u{=}2$, it falls to $0.07$, while the validation completion loss is $6.92$, the highest of the sweep.

\begin{table}[ht]
  \centering
  \caption{Runs of \Cref{fig:local_overfitting}.
  Distinct sequences: training sequences visited over the $3{,}000$ updates.
  Relative completion loss of the inserted sequences: its minimum, its value when the sequence leaves the batch ($J$ updates after insertion, one update for MDM), and its value $300$ updates later.
  Validation completion loss: mean over the last $200$ updates.}
  \label{tab:lo_sweep}
  \small
  \renewcommand{\arraystretch}{1.08}
  \begin{tabular}{@{}lrrcccc@{}}
    \toprule
    & & Distinct & \multicolumn{3}{c}{Relative completion loss} & Validation \\
    \cmidrule(lr){4-6}
    & $J$ & sequences & Minimum & At exit & $300$ later & completion loss \\
    \midrule
    PU, $u{=}2$   & $231$ & $3.6$k  & $0.54$ & $0.88$ & $0.98$ & $6.92$ \\
    PU, $u{=}4$   & $115$ & $6.9$k  & $0.70$ & $0.71$ & $0.98$ & $6.58$ \\
    PU, $u{=}8$   & $57$  & $13.7$k & $0.75$ & $0.75$ & $0.98$ & $5.72$ \\
    PU, $u{=}64$  & $7$   & $110$k  & $0.97$ & $0.99$ & $1.00$ & $4.37$ \\
    PU, $u{=}128$ & $4$   & $192$k  & $0.98$ & $1.00$ & $1.03$ & $4.36$ \\
    MDM           & ---   & $768$k  & $0.98$ & $0.98$ & $1.00$ & $4.72$ \\
    \bottomrule
  \end{tabular}
\end{table}

\begin{figure}[t]
  \centering
  \includegraphics{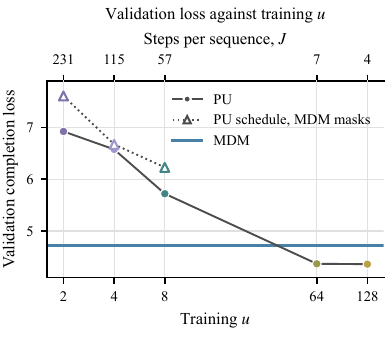}\hfill
  \includegraphics{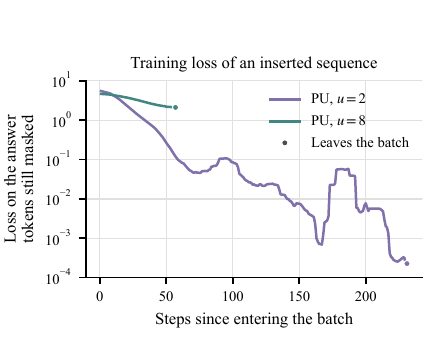}%
  \caption{Complements to \Cref{fig:local_overfitting}.
  \textit{Left}: validation completion loss at the end of training. PU beats MDM when each sequence stays in the batch for only a few steps, but degrades as $u$ shrinks, and training on the PU schedule with MDM's random masks does not help.
  \textit{Right}: loss of the trajectory of an inserted sequence on the answer tokens it still masks, while the sequence is in the batch. At $u{=}2$, the trajectory fits its answer within a quarter of its stay.}
  \label{fig:app_local_overfitting}
\end{figure}

\paragraph{Ablations.}
\Cref{tab:lo_ablations} changes one aspect of PU at a time, at $u \in \{2, 4, 8\}$, and \Cref{fig:app_lo_ablations} follows the inserted sequences under the first two changes at $u{=}2$.
\begin{itemize}[nosep, leftmargin=*]
  \item \emph{PU schedule, MDM masks.} Each sequence still stays in the batch for $J$ consecutive updates, but every update draws a fresh MDM mask for it instead of advancing its trajectory.
  The inserted sequences are fitted even more deeply, with a relative completion loss of $0.02$ at $u{=}2$ that is still $0.52$ $300$ updates after they leave, and the validation loss is as high as with PU or higher.
  The failure therefore does not come from the masks of PU (\Cref{fig:app_lo_ablations}, \textit{middle}).
  \item \emph{Spread visits.} Each GPU keeps a pool of $3{,}000 \cdot 32/J$ sequences, and each update advances the next $32$ of them along their trajectories, so that a sequence is visited once every $3{,}000/J$ updates, about $13$ at $u{=}2$.
  The trajectories, the number of visits per sequence and the number of distinct sequences are the same as with PU, but consecutive batches share no sequence.
  The validation loss is higher than with PU at every $u$: the model memorizes the whole pool over the run, with a training loss of $0.03$ at $u{=}2$, instead of one batch at a time (\Cref{fig:app_lo_ablations}, \textit{right}).
  Reuse is therefore what costs, and the consecutive visits of PU only make the overfitting local.
  \item \emph{Initial stages $0$ to $31$.} Instead of drawing the initial stages at random, the $32$ trajectories of a GPU start at stages $0$ to $31$.
  Since every trajectory lasts exactly $J$ updates, the whole pool then moves along the trajectory as a single window of $32$ consecutive stages, and all losses oscillate with period $J$.
  The validation loss is much higher at $u \leq 4$; at $u{=}8$, where the window covers $32$ of the $57$ stages, it is unchanged.
  This is the noise-level concentration of \Cref{subsec:smallu}, which here lasts for the whole run because no trajectory skips a stage.
  \item \emph{No $1/t$ factor.} Every masked token gets the same weight, instead of the per-sequence normalization of \Cref{eq:mdm}.
  This deepens the drop at $u{=}2$, to a minimum relative completion loss of $0.36$, but changes the validation loss by at most $0.16$ at every $u$, including $u \in \{64, 128\}$ and MDM.
  \item \emph{Batch of $128$ sequences.} Halving the batch makes each sequence a larger share of each update and halves the number of distinct sequences.
  The drop deepens, to $0.37$ at $u{=}2$, and the validation loss is higher at every $u$.
\end{itemize}

\begin{figure}[t]
  \centering
  \includegraphics{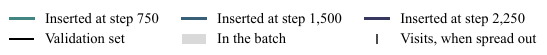}\\
  \includegraphics{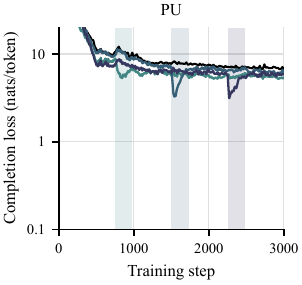}\hfill
  \includegraphics{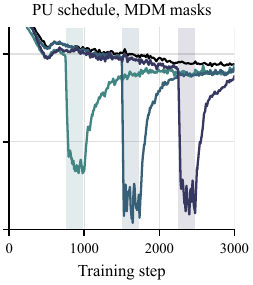}\hfill
  \includegraphics{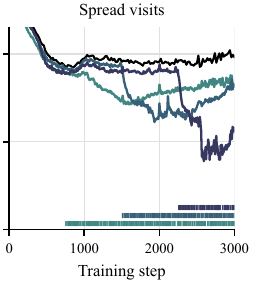}%
  \caption{Inserted sequences under two ablations of PU, at $u{=}2$.
  Each colored curve is the mean completion loss of the four sequences inserted at the same update, smoothed over $9$ updates.
  With MDM's masks, an inserted sequence is almost memorized during its stay and forgotten after it; with spread visits, the model memorizes the inserted sequences over the rest of the run while the validation loss rises.}
  \label{fig:app_lo_ablations}
\end{figure}

\begin{table}[ht]
  \centering
  \caption{Validation completion loss when one aspect of PU is changed at a time (MDM: $4.72$).}
  \label{tab:lo_ablations}
  \small
  \renewcommand{\arraystretch}{1.08}
  \begin{tabular}{@{}lccc@{}}
    \toprule
    & $u{=}2$ & $u{=}4$ & $u{=}8$ \\
    \midrule
    PU                         & $6.92$ & $6.58$ & $5.72$ \\
    PU schedule, MDM masks     & $7.61$ & $6.67$ & $6.23$ \\
    Spread visits              & $9.70$ & $7.18$ & $6.03$ \\
    Initial stages $0$ to $31$ & $9.71$ & $7.64$ & $5.73$ \\
    No $1/t$ factor            & $7.08$ & $6.44$ & $5.81$ \\
    Batch of $128$ sequences   & $8.22$ & $7.12$ & $6.24$ \\
    \bottomrule
  \end{tabular}
\end{table}

\subsection{Why the loss rises again before the sequence leaves the batch}
\label{app:lo_rebound}

At $u{=}2$, the relative completion loss of an inserted sequence bottoms out about $50$ updates after insertion, and three quarters of the drop are gone when the sequence leaves the batch, $231$ updates after insertion (\Cref{fig:local_overfitting}, \textit{middle}).
At $u{=}4$ and $u{=}8$, it keeps falling until the sequence leaves (\Cref{tab:lo_sweep}).
Diagnostics logged in further runs at the same settings explain this difference.

\paragraph{The trajectory fits its answer early.}
While an inserted sequence is in the batch, we log the loss of its trajectory on the answer tokens that are still masked, which is what training minimizes on this sequence, restricted to its answer (\Cref{fig:app_local_overfitting}, \textit{right}).
At $u{=}2$, this loss falls below $0.1$ nats $55$ updates after insertion, while $95\%$ of the answer is revealed only after $137$ updates.
From then on, the sequence contributes almost no gradient, and the updates on the other sequences of the batch erase what it had taught the model.
At $u{=}4$, the trajectory fits its answer only $112$ updates into its $115$-update stay, and at $u{=}8$ it does not fit it at all ($1.7$ nats when the sequence leaves), so the completion loss keeps falling until the sequence leaves and rises only afterwards.
Revealing positions in a random order instead of by confidence still fits the trajectory early, $26$ updates after insertion at $u{=}2$, and keeps a smaller rebound, so the reveal order is not its cause.

\paragraph{A revealed token is no longer maintained.}
In two further runs at $u{=}2$ with $24$ inserted sequences, we record the stage at which each position is revealed, and follow the completion loss at each position of the second half of the answer, whose tokens are the targets of the completion loss.
Grouped by the stage at which they are revealed, the tokens of each group start degrading at their own reveal (\Cref{fig:app_lo_reveal}, \textit{left}).
Under confidence order, the $1{,}179$ tokens revealed between stages $60$ and $170$ keep improving while they are still masked, by $0.62 \pm 0.11$ nats between the bottom of the curve and their reveal, and their loss rises by $1.10 \pm 0.07$ nats in the $50$ updates after their reveal (\Cref{fig:app_lo_reveal}, \textit{right}).
Under random order, masked and revealed tokens both degrade by about $0.45$ nats per $50$ updates once the trajectory has fitted its answer, and the reveal does not change this rate.
A token's prediction is therefore maintained only while the token remains a target that the trajectory has not fitted yet.
Under confidence order, the tokens that remain masked are the ones the model is least confident about, so they keep receiving gradient while the revealed ones are forgotten.

\begin{figure}[t]
  \centering
  \includegraphics{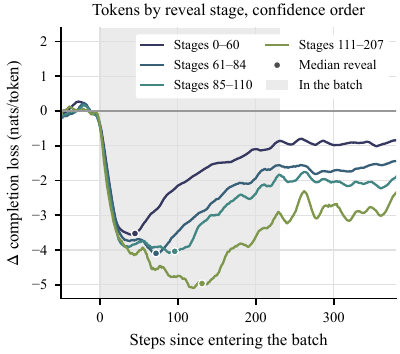}\hfill
  \includegraphics{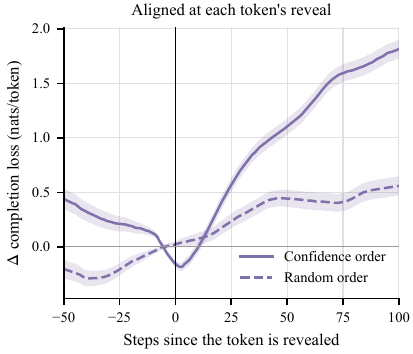}%
  \caption{A token is forgotten once its trajectory reveals it ($u{=}2$).
  Tokens: the second-half answer tokens of the inserted sequences, the targets of the completion loss.
  \textit{Left}: tokens grouped into quartiles of the stage at which they are revealed; change in their completion loss relative to before insertion.
  \textit{Right}: tokens revealed at stages $60$ to $170$; change relative to the $10$ updates before their reveal (mean $\pm$ one standard error).
  Under confidence order, each group degrades from its own reveal on; under random order, the reveal makes no difference, and the four groups of the left panel (not shown) coincide before insertion and do not turn up at their reveal.}
  \label{fig:app_lo_reveal}
\end{figure}

\FloatBarrier

\clearpage
\section{Theory for the carry and the BPTT window}
\label{app:theory}

\paragraph{Setting.}
We use the notation of \Cref{sec:background,sec:exp_tinygsm}.
The prompt is fixed and omitted from the notation.
Let $\Omega$ be the set of response positions.
For a set of positions $A$, we write $\mathbf{x}^A = (x^i)_{i \in A}$.
Let $p^\star$ be the data distribution of the clean response $\mathbf{x}_0^\Omega$.
The carry-augmented denoiser (\Cref{eq:carry}) returns a distribution per position and a new carry:
\[
  f_{\mathrm{carry}}(\mathbf{x}, h) = \bigl(f_\theta(\cdot \mid \mathbf{x}, h),\, h'\bigr).
\]
The sampler starts from $\mathbf{x}_{t_0}$, with the response fully masked, and $h_0 = 0$.
At each step $j$, it proceeds as follows.
\begin{enumerate}[nosep, leftmargin=*]
  \item It runs the denoiser on the current sequence and carry:
  \[
    \bigl(f_\theta(\cdot \mid \mathbf{x}_{t_j}, h_j),\, h_{j+1}\bigr) = f_{\mathrm{carry}}(\mathbf{x}_{t_j}, h_j).
  \]
  \item The unmasking policy selects $S_j = g(\mathbf{x}_{t_j}, h_j)$, the $u$ most confident positions of $\mathcal{M}(\mathbf{x}_{t_j})$ (or all remaining ones, if fewer), breaking ties by index.
  \item It fills each selected position independently:
  \[
    x^i \sim f_\theta^i(\cdot \mid \mathbf{x}_{t_j}, h_j), \qquad i \in S_j .
  \]
\end{enumerate}
We write $A_j = \Omega \setminus \mathcal{M}(\mathbf{x}_{t_j})$ for the response positions revealed before step $j$.
The sampler stops after $J$ steps, when no masked position remains, and we denote by $p_\theta$ the resulting distribution over responses.
Training follows the same policy $g$ but commits ground-truth tokens (\Cref{alg:train}).
We write $\mathcal{L}^{(j)}$ for the loss at step $j$, and the total loss is
\[
  \mathcal{L}_\infty = \sum_{j<J} \mathcal{L}^{(j)} .
\]

\begin{assumption}[Denoiser]
\label{as:denoiser}
The denoiser never predicts the mask token and puts positive mass on every other token: for all $\theta$, $\mathbf{x}$, $h$ and $i$,
\[
  f_\theta^i(m \mid \mathbf{x}, h) = 0,
  \qquad
  f_\theta^i(v \mid \mathbf{x}, h) > 0 \quad \text{for } v \in \mathcal{V} \setminus \{m\}.
\]
Moreover, for every fixed input sequence $\mathbf{x}$, the map $(\theta, h) \mapsto f_{\mathrm{carry}}(\mathbf{x}, h)$ is real-analytic.
\end{assumption}

The first part holds for a softmax restricted to $\mathcal{V} \setminus \{m\}$, i.e., with the mask logit set to $-\infty$.
The second part holds for transformers built from analytic operations, such as attention, softmax, SiLU or GELU activations, and LayerNorm or RMSNorm with $\epsilon > 0$, but not for ReLU activations.
Since the data contain no mask token, $p^\star$ is supported on $(\mathcal{V} \setminus \{m\})^\Omega$, which is finite.

\paragraph{Regular parameters.}
Since $S_j$ depends on $\theta$ through the confidences, $\mathcal{L}_\infty$ can jump where a selection changes, and it is not differentiable there.
We call $\theta$ \emph{regular} if, for every $\mathbf{x}_0$ in the support of $p^\star$ and every step $j$ of its teacher-forced trajectory with $|\mathcal{M}(\mathbf{x}_{t_j})| > u$, the $u$-th largest confidence in $\mathcal{M}(\mathbf{x}_{t_j})$ is strictly larger than the $(u{+}1)$-th.
The confidences are continuous in $\theta$ (\Cref{as:denoiser}) and the support of $p^\star$ is finite, so, by induction on $j$, every set $S_j$ is constant in a neighborhood of a regular $\theta$.
In that neighborhood, $\mathcal{L}_\infty$ is a finite sum of analytic functions, and $\nabla\mathcal{L}_\infty(\theta)$ is the gradient obtained by backpropagating through all carries with the sets $S_j$ held fixed.

\begin{assumption}[Non-degenerate selection]
\label{as:ties}
The set of non-regular $\theta$ has Lebesgue measure zero.
\end{assumption}

This holds, for instance, if no two token probabilities are identical as functions of $\theta$: by \Cref{as:denoiser}, each of them is real-analytic in $\theta$, and the zero set of a real-analytic function that is not identically zero has Lebesgue measure zero~\citep{mityagin2015zero}.

For simplicity, we assume that every trajectory has $J$ steps and that $W$ divides $J$.
Training splits the trajectory into consecutive windows of $W$ steps and detaches the carry at each window boundary (\Cref{alg:train}).
At a regular $\theta$, we denote by $g_W(\theta)$ the resulting truncated gradient, with the sets $S_j$ held fixed; in particular, $g_J = \nabla\mathcal{L}_\infty$.

\begin{proposition}[Likelihood identity]
\label{prop:nll}
Define the teacher-forced loss on revealed positions,
\[
  \widetilde{\mathcal{L}}(\theta) = \mathbb{E}_{\mathbf{x}_0 \sim p^\star}\Bigl[\sum_{j<J}\sum_{i \in S_j} -\log f_\theta^i\bigl(x_0^i \mid \mathbf{x}_{t_j}, h_j\bigr)\Bigr],
\]
where the trajectory is built from $\mathbf{x}_0$ by teacher forcing.
Then
\[
  \widetilde{\mathcal{L}}(\theta) = H(p^\star) + \mathrm{KL}\bigl(p^\star \,\Vert\, p_\theta\bigr).
\]
\end{proposition}

\begin{proof}
By induction on $j$, the pair $(\mathbf{x}_{t_j}, h_j)$ and the set $S_j$ are deterministic functions of the revealed values $\mathbf{x}^{A_j}$.
Indeed, the initial state depends only on the prompt, and the state at step $j{+}1$ is computed from the state at step $j$ and the values written at $S_j$.
Hence each response $\mathbf{x}^\Omega$ is produced by exactly one sampling path, namely the one that selects the sets $S_0, S_1, \ldots$ computed from $\mathbf{x}$.
Since the draws within a step are independent,
\[
  p_\theta(\mathbf{x}^\Omega) = \prod_{j<J}\prod_{i \in S_j} f_\theta^i\bigl(x^i \mid \mathbf{x}_{t_j}, h_j\bigr).
\]
By \Cref{as:denoiser}, the sampler never writes the mask token, so $p_\theta$ is a probability distribution on $(\mathcal{V} \setminus \{m\})^\Omega$, which contains the support of $p^\star$.
Taking $-\log$ and the expectation under $p^\star$ gives the claim.
All terms are finite because $f_\theta^i(v \mid \cdot) > 0$ for every $v \neq m$ (\Cref{as:denoiser}).
\end{proof}

\begin{assumption}[Contractive carry]
\label{as:contract}
There exist constants $\gamma < 1$, $C_h$ and $C_L$ such that, for all $\theta$, every $\mathbf{x}_0$ in the support of $p^\star$ and every step $j$ of its teacher-forced trajectory,
\[
  \Bigl\|\frac{\partial h_{j+1}}{\partial h_j}\Bigr\| \le \gamma,
  \qquad
  \Bigl\|\frac{\partial h_{j+1}}{\partial \theta}\Bigr\| \le C_h,
  \qquad
  \Bigl\|\frac{\partial \mathcal{L}^{(j)}}{\partial h_j}\Bigr\| \le C_L .
\]
\end{assumption}

\begin{proposition}[Truncation bias]
\label{prop:trunc}
Under \Cref{as:denoiser,as:contract}, for every regular $\theta$,
\[
  \bigl\|g_W(\theta) - \nabla\mathcal{L}_\infty(\theta)\bigr\| \le B\Bigl(\frac{J}{W} - 1\Bigr),
  \qquad
  B = \frac{C_L C_h}{(1-\gamma)^2}.
\]
\end{proposition}
\begin{proof}
Since $\theta$ is regular, both $g_W(\theta)$ and $\nabla\mathcal{L}_\infty(\theta)$ are computed with the sets $S_j$ held fixed, and they differ only in the carry paths that cross a window boundary.
Unrolling the carry, the gradient of $\mathcal{L}^{(j)}$ contains, for each lag $1 \le r \le j$, the term
\[
  \frac{\partial \mathcal{L}^{(j)}}{\partial h_j}
  \Bigl(\prod_{s=1}^{r-1}\frac{\partial h_{j-s+1}}{\partial h_{j-s}}\Bigr)
  \frac{\partial h_{j-r+1}}{\partial\theta}.
\]
By \Cref{as:contract}, its norm is at most $C_L C_h \gamma^{r-1}$.
Let $p \in \{0, \ldots, W-1\}$ be the position of step $j$ in its window.
The truncated gradient keeps exactly the lags $r \le p$, so the dropped terms of step $j$ sum to at most
\[
  \sum_{r > p} C_L C_h\, \gamma^{r-1} = \frac{C_L C_h}{1-\gamma}\,\gamma^{p}.
\]
In the first window, $p = j$ and nothing is dropped.
Each of the remaining $J/W - 1$ windows contributes at most
\[
  \frac{C_L C_h}{1-\gamma}\sum_{p=0}^{W-1}\gamma^{p} \;\le\; \frac{C_L C_h}{(1-\gamma)^2} = B. \qedhere
\]
\end{proof}

\begin{definition}[Polyak--{\L}ojasiewicz inequality]
\label{def:pl}
Let $F$ be a function with infimum $F^\star = \inf_\theta F(\theta)$.
It satisfies the Polyak--{\L}ojasiewicz inequality at $\theta$ with constant $\mu > 0$ if $F$ is differentiable at $\theta$ and
\[
  \tfrac{1}{2}\,\|\nabla F(\theta)\|^2 \;\ge\; \mu\,\bigl(F(\theta) - F^\star\bigr).
\]
It satisfies the Polyak--{\L}ojasiewicz inequality~\citep{karimi2016linear} if this holds at every $\theta$.
\end{definition}

The PL inequality does not require convexity.
When it holds everywhere, the gradient is small only near the minimum value, and every stationary point is a global minimizer.
Since $\mathcal{L}_\infty$ jumps where the selected sets change, it cannot satisfy the inequality everywhere in general, and \Cref{thm:window} only requires it at the point of interest.

\begin{theorem}[Window and approximation]
\label{thm:window}
Assume \Cref{as:denoiser,as:contract} and $\mathcal{L}_\infty = \widetilde{\mathcal{L}}$.
Let $\theta_W$ be a regular point with $g_W(\theta_W) = 0$, at which $\mathcal{L}_\infty$ satisfies the PL inequality (\Cref{def:pl}) with constant $\mu$.
Then
\[
  \mathrm{KL}\bigl(p^\star \,\Vert\, p_{\theta_W}\bigr) - \inf_\theta \mathrm{KL}\bigl(p^\star \,\Vert\, p_\theta\bigr)
  \;\le\; \frac{B^2}{2\mu}\Bigl(\frac{J}{W} - 1\Bigr)^2 .
\]
The bound decreases as $W$ grows and vanishes at $W = J$.
\end{theorem}
\begin{proof}
Since $\theta_W$ is regular and $g_W(\theta_W) = 0$, \Cref{prop:trunc} gives
\[
  \|\nabla\mathcal{L}_\infty(\theta_W)\| \le B\Bigl(\frac{J}{W} - 1\Bigr).
\]
The PL inequality at $\theta_W$ then yields
\[
  \mathcal{L}_\infty(\theta_W) - \inf_\theta \mathcal{L}_\infty
  \;\le\; \frac{\|\nabla\mathcal{L}_\infty(\theta_W)\|^2}{2\mu}
  \;\le\; \frac{B^2}{2\mu}\Bigl(\frac{J}{W} - 1\Bigr)^2 .
\]
By \Cref{prop:nll}, $\mathcal{L}_\infty$ and the KL differ only by the constant $H(p^\star)$, which gives the claim.
\end{proof}

\begin{corollary}[Comparison with an autoregressive model]
\label{cor:ar}
Let $q$ be a distribution over responses, for instance that of an autoregressive model, with error $\delta = \mathrm{KL}(p^\star \Vert q) > 0$, and let $\varepsilon^\star = \inf_\theta \mathrm{KL}(p^\star \Vert p_\theta)$.
Assume that, for every window $W$ dividing $J$, the hypotheses of \Cref{thm:window} hold at $\theta_W$ with the same constant $\mu$.
If $\varepsilon^\star < \delta$, then $\mathrm{KL}(p^\star \Vert p_{\theta_W}) < \delta$ for every such $W$ with
\[
  W \;>\; \frac{J}{1 + \sqrt{2\mu(\delta - \varepsilon^\star)}/B},
\]
and $W = J$ always qualifies.
In particular, under the zero-infimum condition $\varepsilon^\star = 0$, a large enough window beats $q$ for every $\delta > 0$.
\end{corollary}
\begin{proof}
By \Cref{thm:window}, $\mathrm{KL}(p^\star \Vert p_{\theta_W}) \le \varepsilon^\star + \frac{B^2}{2\mu}\bigl(\frac{J}{W} - 1\bigr)^2$.
The condition on $W$ is equivalent to $\frac{J}{W} - 1 < \sqrt{2\mu(\delta - \varepsilon^\star)}/B$, so the second term is below $\delta - \varepsilon^\star$.
Since the right-hand side of the condition is smaller than $J$, $W = J$ satisfies it.
\end{proof}

The zero-infimum condition requires the carry-augmented sampler to approximate $p^\star$ arbitrarily well under the reveal rule $g$.
For $u = 1$, the chain rule makes this a capacity assumption.
For $u > 1$, each step draws its $u$ tokens independently given the state, so $p_\theta$ factorizes within each step and $\varepsilon^\star$ is generally positive; the corollary then applies only to AR models with $\delta > \varepsilon^\star$.

\paragraph{Scope of the bound.}
\label{app:theory_scope}
Both parts assume that training reveals positions as inference does. Part (i) holds along any teacher-forced trajectory, so it does not depend on the training $u$. Part (ii), however, bounds the sampler that decodes with the \emph{training} reveal rule, including its $u$ and its number of steps $J$. Local overfitting forces us to train at a much larger $u$ than we decode with (effective training $u \approx 133$ to $67$ on TinyGSM, against $u{=}2$ at inference). So the bound does not directly cover the sampler we evaluate, which takes more steps and visits different states. We read it as a guarantee for an idealized, aligned setting. Whether the benefit carries over to the smaller inference $u$ is an empirical question: \Cref{fig:tinygsm_dmask} addresses it for masking patterns, and \Cref{tab:tinygsm_compose} for accuracy.

\ifapplebuild\needspace{14\baselineskip}\fi %
\begin{theorem}[Credit assignment]
\label{thm:credit}
Assume \Cref{as:denoiser}.
Fix $\mathbf{x}_0$ in the support of $p^\star$ and its teacher-forced trajectory.
Let step $j$ produce the carry $h_{j+1}$, and consider the loss of step $j+r$, $r \ge 1$ steps later.
\begin{enumerate}[label=(\roman*), nosep, leftmargin=*]
  \item If the only state passed between steps is the committed tokens (no carry), then at every regular $\theta$, and hence under \Cref{as:ties} for almost every $\theta$, the total derivative $\mathrm{d}\mathcal{L}^{(j+r)}/\mathrm{d}\theta$ equals the partial derivative with the trajectory held fixed.
  \item With the carry, let $\theta$ be regular and let $\sigma = (S_0, \ldots, S_{j+r-1})$ be the sets it selects.
  If steps $j$ and $j+r$ lie in the same window, the truncated gradient of $\mathcal{L}^{(j+r)}$ contains the term
  \[
    G_{j,r}^{\sigma}(\theta) = \frac{\partial\mathcal{L}^{(j+r)}}{\partial h_{j+r}}
    \Bigl(\prod_{s=1}^{r-1}\frac{\partial h_{j+r-s+1}}{\partial h_{j+r-s}}\Bigr)
    \frac{\partial h_{j+1}}{\partial\theta},
  \]
  where the product is ordered from $s=1$ on the left to $s=r-1$ on the right, and all factors are evaluated along the trajectory defined by $\sigma$.
  As a function of $\theta$ with $\sigma$ fixed, $G_{j,r}^{\sigma}$ is real-analytic on the whole parameter space.
  Hence $G_{j,r}^{\sigma}$ is either identically zero or nonzero for almost every $\theta$, and the second case holds as soon as $G_{j,r}^{\sigma}(\theta') \neq 0$ at a single $\theta'$.
  If steps $j$ and $j+r$ lie in different windows, this term is absent.
  In particular, the loss of a step reaches only earlier steps of its own window, so $r \le W-1$.
\end{enumerate}
\end{theorem}

\begin{proof}
(i) Under teacher forcing, the committed tokens are the ground-truth values $x_0^i$, which do not depend on $\theta$.
Without a carry, the input $\mathbf{x}_{t_{j+r}}$ therefore depends on $\theta$ only through the sets $S_0, \ldots, S_{j+r-1}$.
At a regular $\theta$, these sets are constant in a neighborhood of $\theta$ (see the paragraph on regular parameters).
The input $\mathbf{x}_{t_{j+r}}$ is thus locally constant, and $\mathrm{d}\mathcal{L}^{(j+r)}/\mathrm{d}\theta$ equals the partial derivative with the trajectory held fixed.
Under \Cref{as:ties}, regular points form a set of full measure.

(ii) The sets $\sigma$ are constant near $\theta$, so near $\theta$ the carries $h_{j+1}, \ldots, h_{j+r}$ and the loss $\mathcal{L}^{(j+r)}$ are computed along the fixed trajectory defined by $\sigma$.
Within a window, the carry is never detached, and the chain rule along the path
\[
  \theta \to h_{j+1} \to \cdots \to h_{j+r} \to \mathcal{L}^{(j+r)}
\]
gives $G_{j,r}^{\sigma}$.
With $\sigma$ fixed, each carry is a composition of the maps $(\theta, h) \mapsto f_{\mathrm{carry}}(\mathbf{x}_{t_k}, h)$, which are real-analytic by \Cref{as:denoiser}.
The loss is analytic in the predicted probabilities, since they are positive at the target tokens.
Partial derivatives and products of real-analytic functions are real-analytic, so $G_{j,r}^{\sigma}$ is real-analytic on the whole parameter space.
The dichotomy follows from the fact that the zero set of a real-analytic function that is not identically zero has Lebesgue measure zero~\citep{mityagin2015zero}.
At initialization, the zero-initialized LayerNorm maps every carry to zero, so $\partial\mathcal{L}^{(j+r)}/\partial h_{j+r} = 0$ and $G_{j,r}^{\sigma}$ vanishes; the same holds at every $\theta$ where this LayerNorm has zero gain and bias.
Initialization alone therefore does not tell which of the two cases holds.
If the path crosses a window boundary, the carry is detached there, and the term is absent.
\end{proof}

\Cref{thm:window} assumes that training uses the inference policy $g$ and applies the loss only to the positions each step reveals.
Our TinyGSM training uses the stage-indexed schedule of \citet{kim2026puma} and applies the loss to all masked positions with the weighting of \Cref{eq:mdm}, so the theorem describes an idealized version of it.
In particular, $p_\theta$ in \Cref{prop:nll,thm:window} is the sampler that decodes with the training reveal rule. Whenever inference uses a smaller $u$, as in all our experiments, the bound applies to that sampler and not to the one we evaluate.
The theorem also treats $g_W$ as the sum of the window gradients at a single $\theta$, whereas \Cref{alg:train} divides each window's loss by its length and updates $\theta$ after every window, so it characterizes the fixed points of the idealized full-trajectory update rather than the iterates of \Cref{alg:train}.
The theorem bounds the excess KL of the sampler at points where the truncated gradient vanishes, and this bound decreases with $W$; it says nothing directly about the accuracy of the trained model.
\Cref{thm:credit} concerns exact gradients.
Straight-through, relaxed and score-function estimators provide surrogate signals, which is why they can help but do not match the carry (\Cref{sec:carry}).
Finally, with teacher-forced commits, the gradient through the carry trains each step to pass useful information forward, not to correct its own committed errors.

\clearpage
\section{Additional results on LLaDA}
\label{app:llada_results}

\FloatBarrier
\subsection{Full-canvas SFT schedule}
\label{app:llada_sft_schedule}

All full-canvas Post-SFT runs start from an SFT checkpoint trained for $1.5$ epochs (25k steps) with a cosine schedule.
We choose this length because SFT has largely saturated by then: doubling the schedule to $3$ epochs (50k steps) raises the score at $\tau=1$ by $0.7$ points, from $60.9$ to $61.6$ (\Cref{fig:app_sft_epochs_tau_grid} and \Cref{tab:llada_postsft}).
At this length, we select the peak learning rate from $\{10^{-8}, 10^{-7}, 10^{-6}, 5\times10^{-6}, 10^{-5}, 10^{-4}\}$ (\Cref{fig:app_sft_lr_grid}).
The selection criterion is the score at $\tau=1$, which decodes one token per step, as diffusion language models are usually benchmarked~\citep{nie2025llada,ye2025dream}.
The peak learning rate $5\times10^{-6}$ scores highest ($60.9$, against $60.5$ at $10^{-5}$), and its checkpoint is the starting point of all full-canvas Post-SFT runs.
The Post-SFT MDM control, which continues this checkpoint for 4000 steps, scores $61.2$ against its $60.9$; we attribute the small change to the short warm-up and re-annealing of the learning rate.
All scores use the chat prompt format (\Cref{app:setup_llada}); \Cref{app:llada_comparison} compares them with published LLaDA results.
IFEval is the benchmark that gains the most from SFT, from $23.3$ to $68.0$ at $\tau=1$, against $60.3$ to $75.9$ on GSM8K and $36.6$ to $38.8$ on MBPP (\Cref{tab:llada_published_comparison}).
This is consistent with the instruction-following focus of the Dolci-Instruct-SFT data.

\begin{figure}[ht]
  \centering
  \includegraphics[width=\appscale\linewidth]{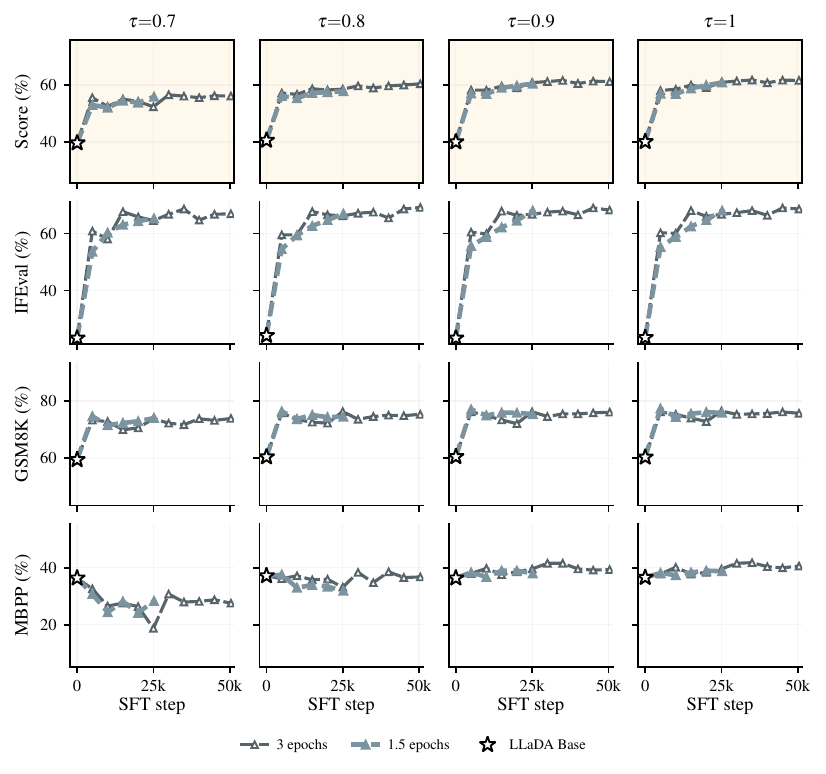}
  \caption{Full-canvas SFT of LLaDA-8B-Base for $1.5$ and $3$ epochs, against SFT step.
  Columns: decode threshold $\tau$.
  Rows: score, then each benchmark.
  The star marks LLaDA-8B-Base.}
  \label{fig:app_sft_epochs_tau_grid}
\end{figure}

\begin{figure}[ht]
  \centering
  \includegraphics[width=\appscale\linewidth]{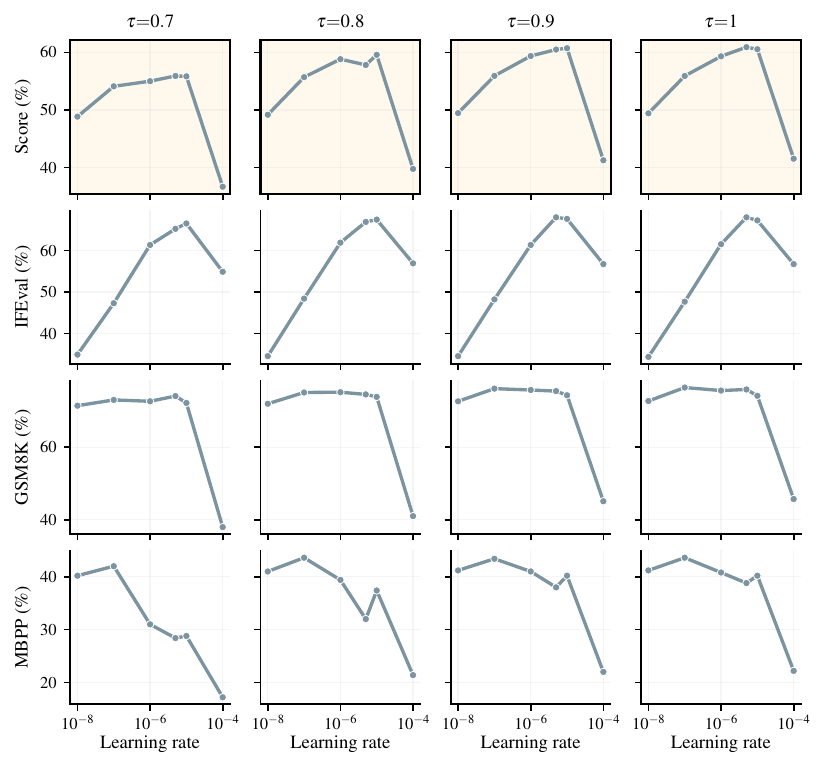}
  \caption{Full-canvas SFT for $1.5$ epochs, against peak learning rate.
  Columns: decode threshold $\tau$.
  Rows: score, then each benchmark.}
  \label{fig:app_sft_lr_grid}
\end{figure}

\FloatBarrier
\subsection{Full-canvas BPTT}
\label{app:llada_bptt}

\Cref{tab:llada_postsft} summarizes the best score of every full-canvas Post-SFT configuration.

\begin{table}[ht]
  \centering
  \caption{Best score of each full-canvas Post-SFT variant, for every training $u$ of the BPTT sweep, with the change relative to the departure SFT checkpoint in parentheses.
  \textbf{Best} and \underline{2nd best} highlighted.}
  \label{tab:llada_postsft}
  \small
  \setlength{\tabcolsep}{5pt}
  \begin{tabular}{@{}lcccc@{}}
    \toprule
    \multicolumn{5}{@{}l}{\textbf{Full canvas}} \\
    & $u{=}16$ & $u{=}32$ & $u{=}64$ & $u{=}128$ \\
    \midrule
      PU & $62.5$\,{\tiny$(+1.6)$} & $62.8$\,{\tiny$(+1.9)$} & $61.9$\,{\tiny$(+1.0)$} & $61.2$\,{\tiny$(+0.3)$} \\
      PU + Carry, $W{=}1$ & $61.9$\,{\tiny$(+1.0)$} & $62.6$\,{\tiny$(+1.7)$} & $62.4$\,{\tiny$(+1.5)$} & $61.7$\,{\tiny$(+0.8)$} \\
      \phantom{PU + Carry,} $W{=}2$ & $61.8$\,{\tiny$(+0.9)$} & $61.4$\,{\tiny$(+0.5)$} & $62.5$\,{\tiny$(+1.6)$} & $62.0$\,{\tiny$(+1.1)$} \\
      \phantom{PU + Carry,} $W{=}4$ & $61.9$\,{\tiny$(+1.0)$} & $\underline{63.0}$\,{\tiny$(+2.1)$} & $62.6$\,{\tiny$(+1.7)$} & $62.6$\,{\tiny$(+1.7)$} \\
      \phantom{PU + Carry,} $W{=}8$ & $62.6$\,{\tiny$(+1.7)$} & $\mathbf{63.5}$\,{\tiny$(\mathbf{+2.6})$} & $\underline{63.0}$\,{\tiny$(+2.1)$} & $62.4$\,{\tiny$(+1.5)$} \\
    \midrule
    SFT, 1.5 epochs (departure) & \multicolumn{4}{c}{$60.9$} \\
    SFT, 3 epochs               & \multicolumn{4}{c}{$61.6$\,{\tiny$(+0.7)$}} \\
    Post-SFT MDM                & \multicolumn{4}{c}{$61.2$\,{\tiny$(+0.3)$}} \\
    \bottomrule
  \end{tabular}
\end{table}

\Cref{fig:app_bptt_grid} reports full-canvas BPTT for $u \in \{16, 32, 64, 128\}$ and $W \in \{1, 2, 4, 8\}$, for the score and for each benchmark.
Every carry configuration reaches a higher best score than the Post-SFT MDM control ($61.2$).
The gain over PU without carry is not monotonic in $W$, but the best window improves on PU at every $u$.
This improvement grows with $u$, from $+0.1$ at $u=16$ to $+1.4$ at $u=128$, mostly because PU alone weakens as $u$ grows: at $u=128$ it only matches the Post-SFT MDM control (\Cref{subsec:smallu}).
The best score overall is reached at $u=32$ (\Cref{tab:llada_postsft}).
Below $u=16$, PU without carry already falls under both baselines ($59.4$ at $u=8$ and $47.6$ at $u=4$, \Cref{fig:app_trainu_taus}), a consequence of the small-$u$ optimization issues of \Cref{subsec:smallu}, so we do not extend BPTT there.

\begin{figure}[ht]
  \centering
  \includegraphics[width=\appscale\linewidth]{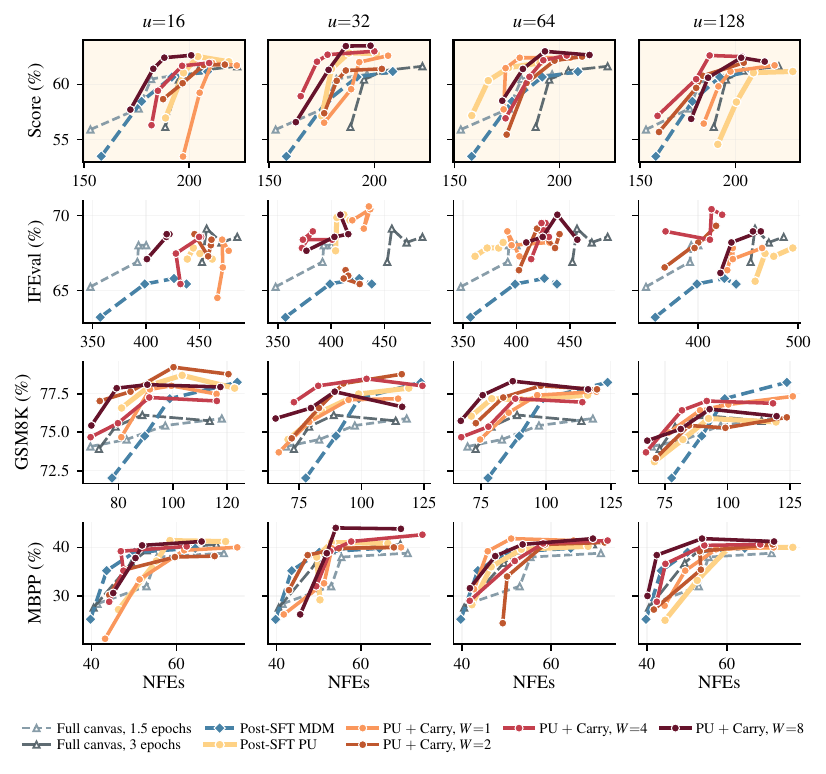}
  \caption{Full-canvas Post-SFT with BPTT, against NFEs.
  Columns: train $u$.
  Rows: score, then each benchmark.
  Each curve sweeps the decode threshold $\tau \in \{0.7, 0.8, 0.9, 1\}$.}
  \label{fig:app_bptt_grid}
\end{figure}

\FloatBarrier
\subsection{Block diffusion SFT}
\label{app:llada_semiar_sft}

As on the full canvas, we select the learning rate for block diffusion SFT at $1.5$ epochs (25k steps), with the score at $\tau=1$ as criterion (\Cref{fig:app_semiar_sft_lr_grid}).
The peak learning rate $5\times10^{-6}$ again scores highest ($57.8$, against $57.4$ at $10^{-5}$), and all block diffusion SFT runs use it.
Unlike the runs in \Cref{fig:app_semiar_sft_epochs}, these sweep runs do not train on the EOS padding that completes the last block of each response, so their scores differ slightly from those runs.\footnote{Training on this padding also shortens generation: at $\tau=0.9$, the $1.5$-epoch run of \Cref{fig:app_semiar_sft_epochs} needs $135$ NFEs, against $190$ for the sweep run with the same learning rate.
Nonetheless, since all sweep runs share the same setting, we do not expect it to change the ranking of learning rates.}

Unlike the full canvas, block diffusion SFT has not saturated at $1.5$ epochs.
The score at $\tau=1$ rises from $57.2$ at $1.5$ epochs to $59.8$ at $3$ epochs and $61.4$ at $6$ epochs (\Cref{fig:app_semiar_sft_epochs}), against a $0.7$-point gain from $1.5$ to $3$ epochs on the full canvas.
We attribute the slower convergence to the mismatch between the block-wise attention pattern, causal across blocks and bidirectional within the current block, and the fully bidirectional attention of LLaDA's pre-training.
We therefore start the block diffusion Post-SFT runs from the $6$-epoch checkpoint (100k steps).
As on the full canvas, IFEval gains the most, from $23.3$ to $65.4$ at $6$ epochs, against $60.3$ to $77.0$ on GSM8K and $36.6$ to $41.6$ on MBPP.

\Cref{fig:app_sft_vs_baselines} compares the block diffusion checkpoints at $1.5$, $3$, and $6$ epochs with the full-canvas checkpoints at $1.5$ and $3$ epochs, as a function of NFEs.
At $6$ epochs, the block diffusion checkpoint lies above both full-canvas checkpoints over most of the NFE range: it reaches $61.1$ at $153$ NFEs, which the $3$-epoch full-canvas checkpoint matches only at $204$ NFEs.
The exception is the $3$-epoch full-canvas checkpoint at $\tau=1$, which scores $0.2$ points higher ($61.6$) at $223$ NFEs, against $196$.
We attribute this to block diffusion SFT training with the same block-wise generation order as the block-$32$ decoding used at evaluation.

\begin{figure}[ht]
  \centering
  \includegraphics[width=\ifapplebuild0.7\else1\fi\linewidth]{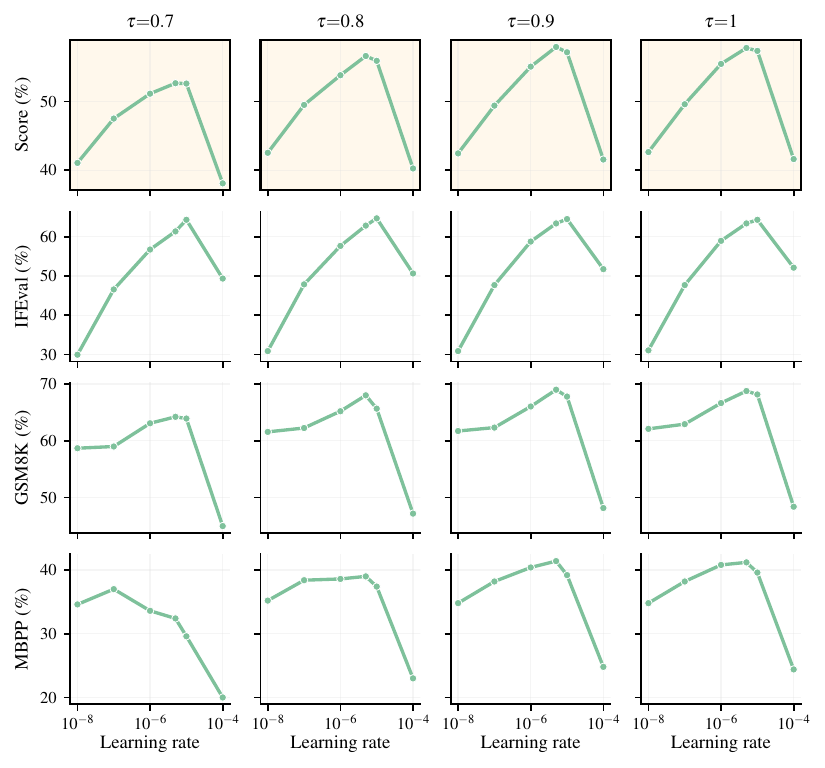}
  \caption{Block diffusion SFT for $1.5$ epochs, against peak learning rate.
  Columns: decode threshold $\tau$.
  Rows: score, then each benchmark.}
  \label{fig:app_semiar_sft_lr_grid}
\end{figure}

\begin{figure}[ht]
  \centering
  \includegraphics[width=\appscale\linewidth]{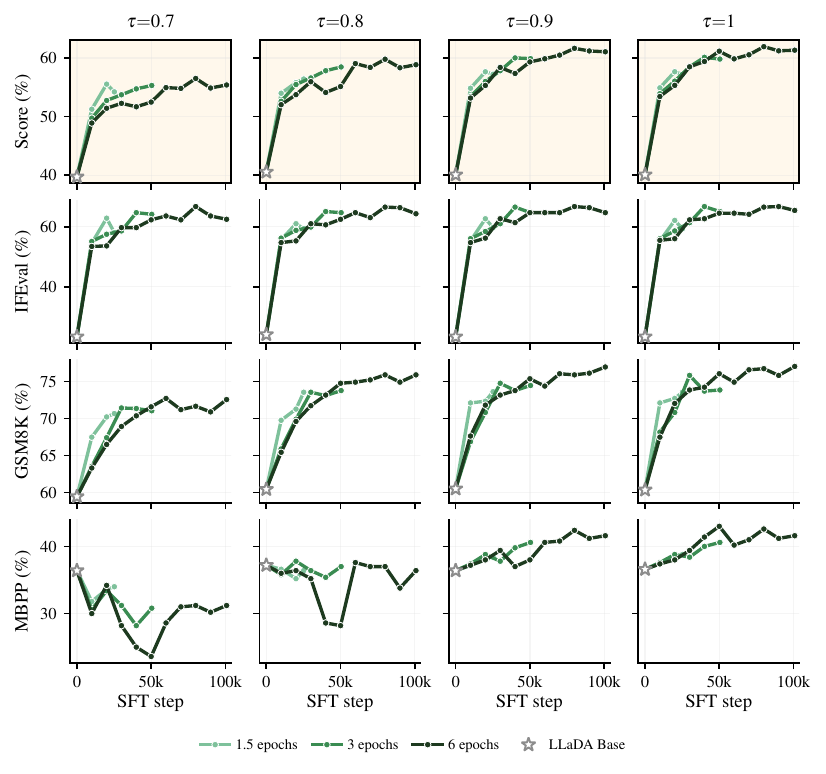}
  \caption{Block diffusion SFT of LLaDA-8B-Base for $1.5$, $3$, and $6$ epochs, against SFT step.
  Columns: decode threshold $\tau$.
  Rows: score, then each benchmark.
  The star marks LLaDA-8B-Base.}
  \label{fig:app_semiar_sft_epochs}
\end{figure}

\begin{figure}[ht]
  \centering
  \includegraphics[width=\appscale\linewidth]{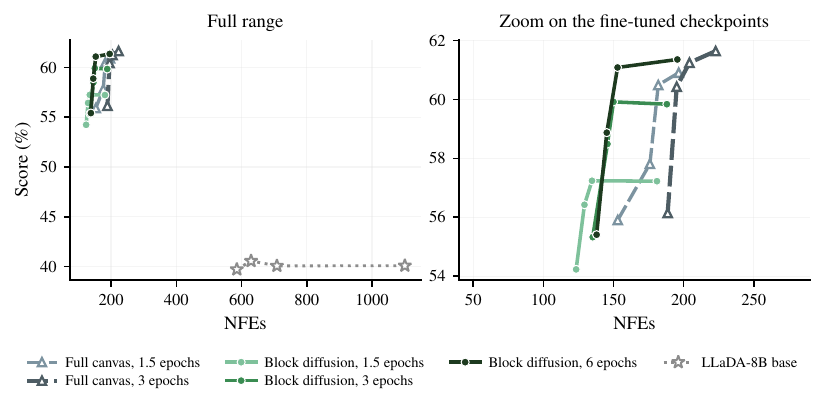}
  \caption{Score against NFEs for the SFT checkpoints and LLaDA-8B-Base.
  Each curve sweeps the decode threshold $\tau \in \{0.7, 0.8, 0.9, 1\}$.
  Left: full range.
  Right: fine-tuned checkpoints only.}
  \label{fig:app_sft_vs_baselines}
\end{figure}

\FloatBarrier
\subsection{Block diffusion \texorpdfstring{$u$}{u} sweep and BPTT}
\label{app:llada_semiar_bptt}

\Cref{fig:app_semiar_trainu} reports block diffusion PU Post-SFT for $u \in \{1, 2, 4, 8, 16\}$, starting from the $6$-epoch SFT checkpoint.
Unlike the full canvas, no value of $u$ fully collapses, although $u \in \{1, 2\}$ still trail the others in the score--NFE trade-off.

All blocks of a sample are unmasked in parallel, so its trajectory lasts at most $\lceil 32/u \rceil$ steps.
With the training-time threshold, which commits more than $u$ tokens per step when the model is confident, a trajectory at $u=1$ lasts $16.0$ steps on average, against $160.2$ on the full canvas.

\Cref{fig:app_semiar_bptt_grid} reports block diffusion BPTT for $u \in \{4, 8, 16\}$.
The window cannot exceed the trajectory length, so $W \leq \lceil 32/u \rceil$: $u=4$ admits $W \in \{1, 2, 4, 8\}$, $u=8$ admits $W \in \{1, 2, 4\}$, and $u=16$ admits $W \in \{1, 2\}$.
The best window improves on PU without carry by $1.0$--$1.5$ points at every $u$, while the carry-only variant ($W=1$) does not.
\Cref{tab:llada_postsft_block} summarizes the best score of every block diffusion Post-SFT configuration.

\begin{table}[ht]
  \centering
  \caption{Best score of each block diffusion Post-SFT variant, for every training $u$ of the BPTT sweep, with the change relative to the departure SFT checkpoint in parentheses.
  \textbf{Best} and \underline{2nd best} highlighted.}
  \label{tab:llada_postsft_block}
  \small
  \setlength{\tabcolsep}{5pt}
  \begin{tabular}{@{}lccc@{}}
    \toprule
    \multicolumn{4}{@{}l}{\textbf{Block diffusion}} \\
    & $u{=}4$ & $u{=}8$ & $u{=}16$ \\
    \midrule
      PU & $62.8$\,{\tiny$(+1.4)$} & $61.9$\,{\tiny$(+0.5)$} & $62.1$\,{\tiny$(+0.8)$} \\
      PU + Carry, $W{=}1$ & $61.1$\,{\tiny$(-0.2)$} & $62.0$\,{\tiny$(+0.6)$} & $62.1$\,{\tiny$(+0.7)$} \\
      \phantom{PU + Carry,} $W{=}2$ & $\mathbf{63.8}$\,{\tiny$(\mathbf{+2.5})$} & $62.0$\,{\tiny$(+0.6)$} & $\underline{63.6}$\,{\tiny$(+2.2)$} \\
      \phantom{PU + Carry,} $W{=}4$ & $62.3$\,{\tiny$(+0.9)$} & $62.9$\,{\tiny$(+1.5)$} & --- \\
      \phantom{PU + Carry,} $W{=}8$ & $62.1$\,{\tiny$(+0.8)$} & --- & --- \\
    \midrule
    SFT, 6 epochs (departure) & \multicolumn{3}{c}{$61.4$} \\
    SFT, 3 epochs             & \multicolumn{3}{c}{$59.9$\,{\tiny$(-1.4)$}} \\
    Post-SFT MDM              & \multicolumn{3}{c}{$62.0$\,{\tiny$(+0.7)$}} \\
    \bottomrule
  \end{tabular}
\end{table}

\begin{figure}[ht]
  \centering
  \includegraphics[width=\appscale\linewidth]{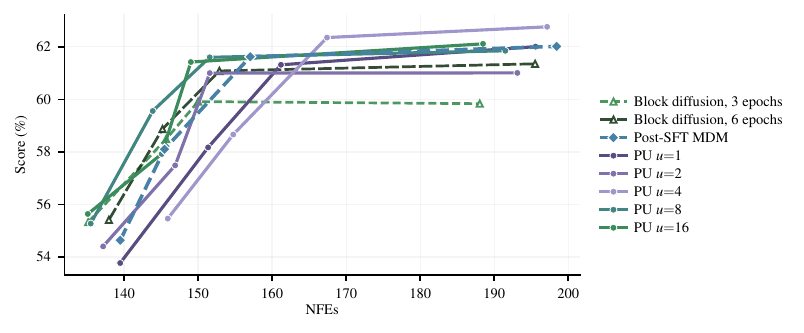}
  \caption{Score against NFEs for block diffusion PU Post-SFT across $u$.
  Each curve sweeps the decode threshold $\tau \in \{0.7, 0.8, 0.9, 1\}$.}
  \label{fig:app_semiar_trainu}
\end{figure}

\begin{figure}[ht]
  \centering
  \includegraphics[width=0.835\linewidth]{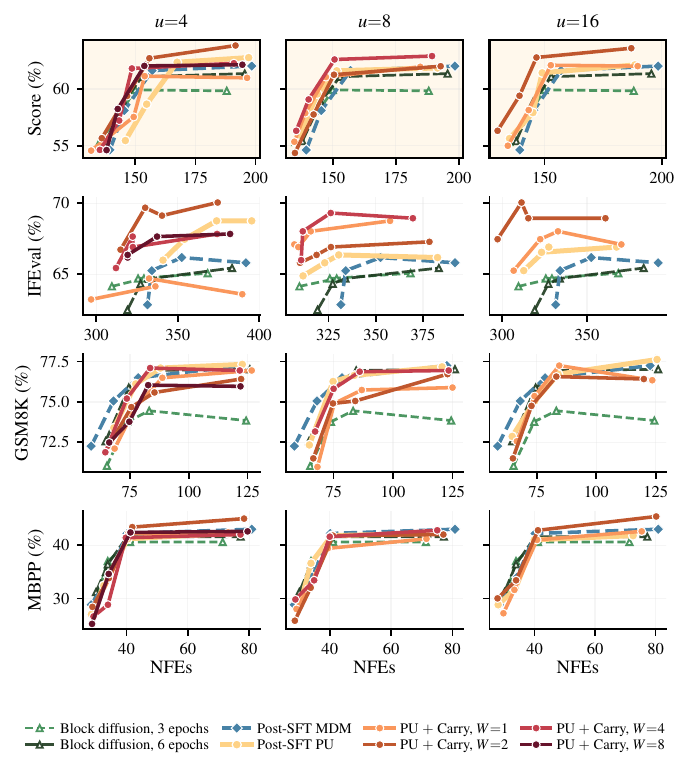}
  \caption{Block diffusion Post-SFT with BPTT, against NFEs.
  Columns: $u$.
  Rows: score, then each benchmark.
  Each curve sweeps the decode threshold $\tau \in \{0.7, 0.8, 0.9, 1\}$.}
  \label{fig:app_semiar_bptt_grid}
\end{figure}

\FloatBarrier
\subsection{The limit case of \texorpdfstring{$u$}{u} equal to the block size in block diffusion}
\label{app:llada_semiar_u_block}

With a block size of $32$, a training step at $u=32$ commits the whole block, so every trajectory lasts one step.
PU at $u=32$ is therefore a limit case of PU: it only trains on fully masked blocks that follow a clean prefix.
At inference, every block starts from such a sequence, except that the model generated the prefix.
The later reveals within a block receive no Post-SFT training at $u=32$, so the model makes them with what it learned in pre-training and SFT.

PU at $u=32$ reaches $63.1$ at $\tau=1$, against $62.0$ for Post-SFT MDM (\Cref{fig:app_semiar_u_block}).
Its score--NFE curve also lies above those of PU at every other $u$ (\Cref{fig:app_semiar_trainu}).
We believe this gain comes from the short Post-SFT horizon of $4000$ steps: the model improves its first reveal of each block, and pre-training and SFT still let it complete the block.
The later reveals receive no training signal, so we expect a longer Post-SFT at $u=32$ to collapse.

We cannot apply BPTT at $u=32$ because the trajectory length is $1$.
\Cref{fig:app_semiar_u_block} compares PU at $u=32$ with PU at $u=16$, alone and with BPTT at $W=2$, the configurations reported in \Cref{sec:exp_block_diff}.
We observe that the score--NFE trade-off of PU at $u=16$ with BPTT at $W=2$ dominates that of PU at $u=32$, although PU alone at $u=16$ scores lower ($62.1$ against $63.1$).
The best reveal count without BPTT is therefore not necessarily the best one with it: a smaller $u$ with BPTT can match or exceed a larger $u$ that rules BPTT out.
Our experiments do not isolate why $u{=}32$ suits PU alone; the cause may lie in our setup (model, dataset, and benchmarks), in the optimization issues of small $u$ (\Cref{subsec:smallu,app:llada_smallu}), or in the decode block of $32$.

\begin{figure}[ht]
  \centering
  \includegraphics[width=\appscale\linewidth]{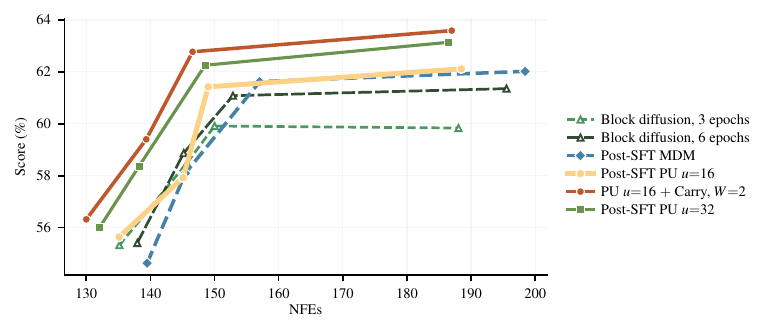}
  \caption{Score against NFEs for block diffusion Post-SFT with PU at $u=32$, and with PU at $u=16$ alone and with BPTT at $W=2$, against the SFT checkpoints and Post-SFT MDM.
  Each curve sweeps the decode threshold $\tau \in \{0.7, 0.8, 0.9, 1\}$.}
  \label{fig:app_semiar_u_block}
\end{figure}

\FloatBarrier
\subsection{Model collapse as a function of training \texorpdfstring{$u$}{u}}
\label{app:llada_smallu}

In \Cref{subsec:smallu}, we present the effect of decreasing the training $u$ on the learning dynamics. Below, we provide more evidence of the challenges brought by small training $u$, which we trace to three factors: sample diversity, noise diversity, and the objective mismatch.

  \paragraph{Sample diversity.} PU-style training causes consecutive steps to visit similar sequences: a sequence not yet fully unmasked at step $t$ reappears at step $t+1$, one unmasking step further along its trajectory. At a fixed token budget, training therefore covers a smaller set of distinct sequences (\Cref{fig:llada_diversity}) and keeps each of them in the batch for many consecutive steps, which is what makes the model overfit locally (\Cref{subsec:smallu}). \Cref{fig:llada_diversity_steps} measures how long sequences stay in the batch, by the Jaccard similarity---the intersection over union---between the sequences visited by two steps $\Delta$ apart. We estimate it from the dynamics observed on a single machine in a multi-GPU run and extrapolate to the full run. Small training $u$ yields substantially higher overlap that persists over large $\Delta$, whereas larger $u$ shows a faster decay as $\Delta$ increases. Confidence collapse dampens the effect: by committing several high-confidence positions at once, it allows sequences to complete in fewer steps and be replaced sooner.
\begin{figure}[ht]
    \centering
    \includegraphics[width=\appscale\linewidth]{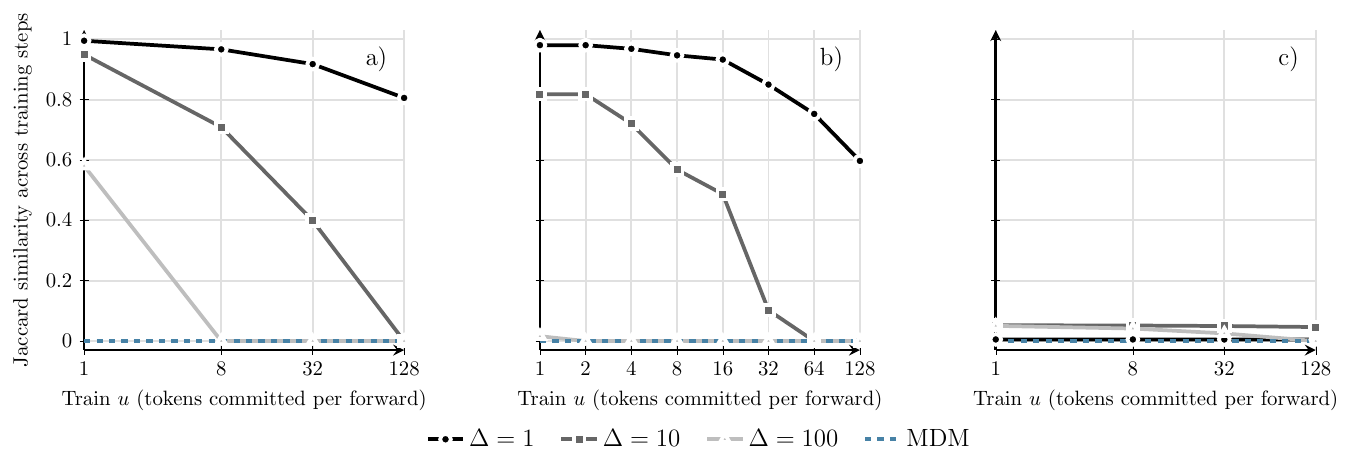}
    \caption{Intersection over union of samples within the training batch $\Delta$ training steps apart across training $u$: (a) PU training without confidence collapse; (b) PU training with confidence collapse; (c) PU training without confidence collapse when sampling batches from a pool of training trajectories 10$\times$ the size of the batch.}
    \label{fig:llada_diversity_steps}
\end{figure}
\paragraph{Noise diversity.} PU-style training iterates through sequences by progressively unmasking them, starting from fully masked sequences. Early in training the number of visible tokens across a batch is therefore concentrated near $0$. MDMs are not subject to this effect, as noise levels are drawn i.i.d. per sample at every training step. \Cref{fig:llada_noise_diversity} illustrates this behavior. We track the standard deviation, across training batches logged every $10$ steps, of the number of unmasking steps each sequence has undergone. For all training $u$, the standard deviation starts near zero as all the sequences are fully masked and grows as sequences complete and are renewed at different times. What changes between $u$ is the pace at which this mixing happens: for $u{=}1$, the progression toward reaching a steady state is longer than for $u{=}32$ and $u{=}128$, for which it can happen at a pace the logging frequency cannot track. While this mixing is slower at low $u$, it affects only the first few hundred training steps, so we expect its impact on the general performance to diminish as the number of training steps increases.
Two further factors shape this mixing. Sequences differ in length and are completed after different numbers of steps, so the within-batch noise-level distribution broadens as training progresses. In block diffusion, the parallel-block implementation unmasks $u$ tokens per block, so a small training $u$ corresponds to a higher effective number of tokens unmasked per step across the sequence.

\paragraph{Objective mismatch.} PU-style training computes the loss over all masked tokens, akin to MDMs. The weight given to uncommitted positions is therefore not set directly but follows from $u$: as $u$ decreases, the ratio of supervised to committed positions grows, and the gradient is increasingly dominated by predictions at positions the step does not commit. A single hyperparameter thus governs both the granularity of the trajectory and the balance of supervision between committed and uncommitted positions. Interestingly, \citet{xia2026metastate} propose an alternative loss which decouples the two, defining the loss as a weighted sum of the term over all masked tokens and the term over committed tokens alone.

\begin{figure}[ht]
    \centering
    \includegraphics[width=\appscale\linewidth]{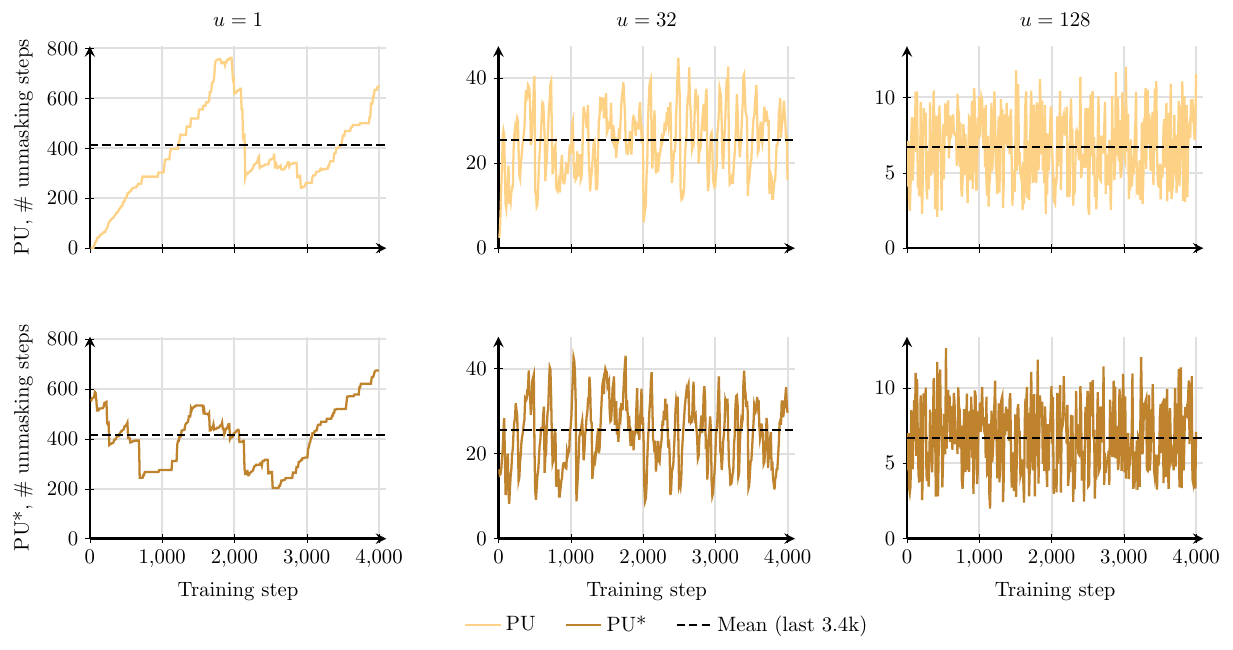}
    \caption{Standard deviation of the number of unmasking steps within a batch across training, for PU-style training (top) and PU-style training with randomly initialized masking ratios (bottom), at $u=1$ (left), $u=32$ (middle) and $u=128$ (right). All runs use full-canvas diffusion without confidence collapse. The dotted line refers to the mean computed on the last 3,000 training steps.}
    \label{fig:llada_noise_diversity}
\end{figure}

\subsubsection{Mitigating the gap at small train \texorpdfstring{$u$}{u}}
We attempt to mitigate the effect of training with a small $u$ by targeting the three
factors identified above: sample diversity, noise diversity, and
the objective mismatch.

\paragraph{Sample diversity.} We maintain a pool of $P$ training sequences, ten times
larger than the training batch size $B$. At each training step $s$, rather than
advancing the sequences seen at $s-1$, we advance the next $B$ sequences in the pool.
Consecutive steps therefore visit non-overlapping sets of sequences
(\Cref{fig:llada_diversity_steps} c) and the number of distinct sequences seen over
training increases. This comes at a cost in memory: the
pool holds $P$ sequences and their masks rather than $B$, and the overhead grows further
with carry and BPTT, which store additional tensors per sequence. In practice, this
bounds the pool size we can reasonably use.

\paragraph{Noise diversity.} We initialize training from partially unmasked sequences, with
masking ratios sampled uniformly. \Cref{fig:llada_noise_diversity} shows the effect on
early training at $u{=}1$: PU with random initial masking ratios maintains a higher
within-batch standard deviation than PU from the first step onward.

\paragraph{Objective mismatch.} Finally, we explicitly
balance the contributions of the tokens to be unmasked at the current step against
those of the remaining masked tokens. \citet{xia2026metastate} propose a weighted sum of
two cross-entropy terms, one over all masked tokens and one over the tokens revealed at
the current step:
\begin{equation}\label{eq:metastate}
  \mathcal{L}_\text{MS}^{k}
  = w_k \Big[ \lambda \, \ell_k^{\mathrm{dense}}
            + (1 - \lambda) \, \ell_k^{\mathrm{reveal}} \Big],
  \qquad
  w_k = \frac{n_k}{N_m},
\end{equation}
where $n_k$ is the number of positions revealed at step $k$, $N_m$ the total number of
maskable positions, and $\ell_k^{\mathrm{dense}}$ and $\ell_k^{\mathrm{reveal}}$ the
cross-entropy losses averaged over all masked tokens and over the tokens revealed at
step $k$, respectively. We adapt \Cref{eq:metastate} as
\begin{equation}\label{eq:pu_plus_loss}
  \mathcal{L}_{PU+}^{k}
  = \Big[ \lambda \, \ell_k^{\mathrm{reveal}}
            + (1 - \lambda) \, \ell_k^{\mathrm{differ}} \Big],
\end{equation}
where $\ell_k^{\mathrm{differ}}$ is averaged over the tokens that remain masked after step
$k$. We do not ablate $\lambda$ thoroughly, but instead
offer a proof of concept: we set $\lambda = 32/L$, where $L$ is the total number of
masked positions per sequence, matching the relative weight that the best-performing PU
run ($u{=}32$) assigns to revealed versus still-masked tokens.

\begin{figure}[ht]
    \centering
    \includegraphics[width=0.5\linewidth]{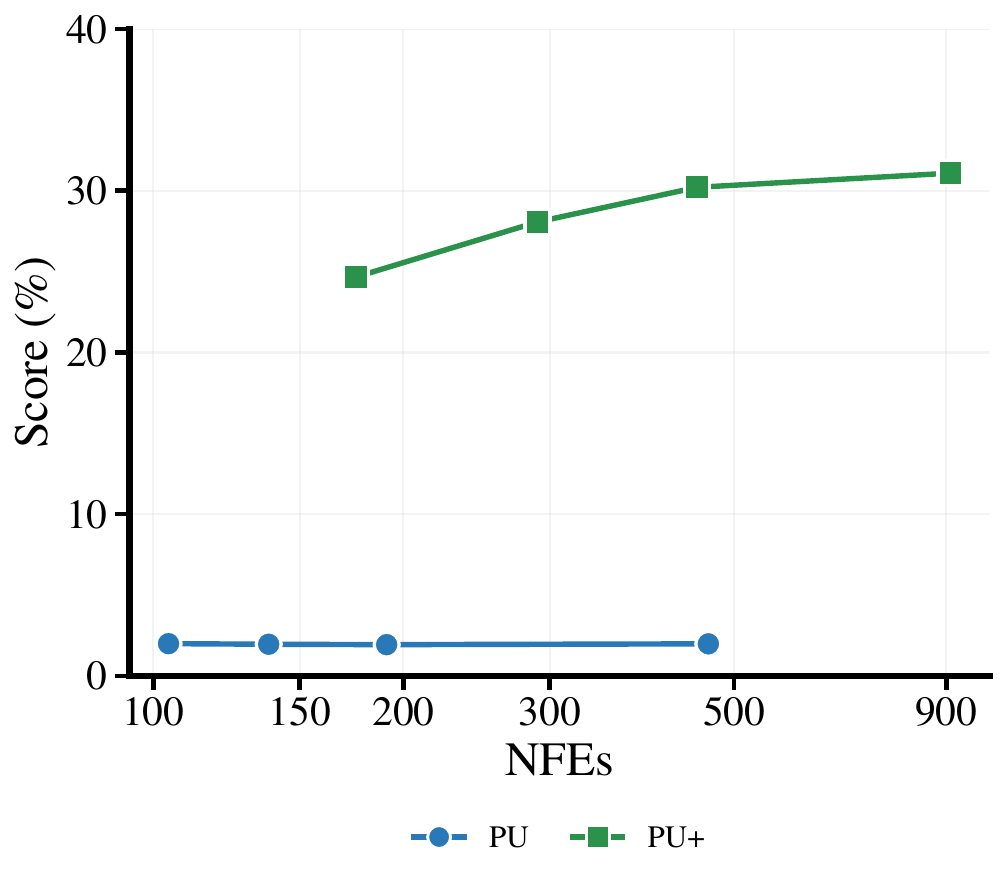}
    \caption{Mitigating model collapse in PU-style training at small training $u$: we compare PU, trained for full-canvas generation, top-$u$ unmasking with $u=1$, with PU+, which integrates tricks aiming at mitigating issues found in PU at small $u$. Inference here follows the Fast-dLLM policy with confidence threshold $\tau \in \{0.7,0.8,0.9,1\}$.}
    \label{fig:smallu_improve}
\end{figure}

  In \Cref{fig:smallu_improve}, we apply these changes to PU-style full-canvas training, where the top-$u$ tokens are unmasked at each step; at inference we use the Fast-dLLM policy and vary the confidence threshold $\tau$. Together, these changes substantially reduce the performance gap at small $u$ without closing it. Additionally, these tricks cost memory and tuning effort and only partially mitigate collapse; the simpler alternative remains to choose a training $u$ that trades off training–inference alignment against healthy dynamics.

These three effects make it impractical to do PU at small $u$, while this is the regime typically targeted at inference (e.g., unmask $\simeq8$ tokens at a time). Beyond memory and tuning, the adjustments above are also harder to integrate with complementary techniques such as the carry or BPTT.

\paragraph{Block diffusion reduces reuse.}
Block diffusion allays this trade-off. With blocks of $b$ tokens, all advancing in parallel, a sequence stays in the batch for at most $\lceil b/u \rceil$ steps instead of about $L/u$, so that at the same $u$ it is reused $L/b$ times less (\Cref{sec:exp_block_diff,app:setup_llada_postsft}). This is consistent with block diffusion being less exposed to the loss of sample diversity (\Cref{fig:llada_diversity}) and with no value of $u$ fully collapsing in that setting (\Cref{app:llada_semiar_bptt}).

\ifapplebuild\else\FloatBarrier\fi
\subsection{Relevance of confidence thresholding at training time}
\label{app:llada_trainu}

All PU and BPTT runs outside this subsection use a training-time confidence threshold $\tau_{\text{train}}=0.9$, the default of \citet{kim2026puma}, who find PU mostly insensitive to the threshold except for very extreme values (e.g., $\tau \leq 0.65$).
Each training step commits every masked position whose confidence exceeds $\tau_{\text{train}}$ and, if fewer than $u$ qualify, the most confident remaining positions up to $u$ (\Cref{sec:exp_llada,app:pu_details}).
Here we compare it with a threshold-free variant that commits exactly $u$ positions per step.

The threshold has two effects.
First, it brings training trajectories closer to those of the confidence-threshold decoding used at inference.
Second, it acts on the small-$u$ optimization issues of \Cref{subsec:smallu}: positions the model already predicts confidently are committed at once, so the model spends fewer steps on tokens it already knows and trajectories end sooner.
As a result, a run visits more distinct samples (\Cref{fig:app_sample_diversity}).
At $u=1$, a run with the threshold visits $0.6\%$ of the samples the Post-SFT MDM control visits, against $0.3\%$ without.
The difference shrinks as $u$ grows and is within the uncertainty bands from $u=8$ on ($2.5\%$ against $2.3\%$).
We estimate the number of samples a run visits from the mean number of response tokens revealed per step, divided by the mean response length measured on the Post-SFT MDM control.\footnote{Training logs one rank every tenth step, so the exact count is not recorded.
A trajectory ends only once all its response tokens are revealed, so each sample contributes exactly its response length and the ratio counts samples.
The bands are one standard error, combining the variation in response length over the logged rows and a block bootstrap over logged steps, with blocks of one trajectory length.}

\Cref{fig:app_trainu_taus,fig:app_trainu_full} report the effect on score.
The threshold helps most where PU is limited by small $u$: it adds $13.2$--$13.8$ points at $u=4$ and $2.4$--$5.4$ points at $u=8$, across decode thresholds.
For $u \in \{16, 32, 64\}$ the difference lies between $-0.3$ and $+2.0$ points, and at $u=128$ between $-0.8$ and $+0.3$.
At $u \leq 2$, both variants collapse below $6$ points.
These comparisons are at matched $u$ and $\tau$, not at matched NFEs, since the threshold also changes the NFEs a run needs at a given $\tau$.
\Cref{fig:app_nfe_topk} accounts for cost.
Without the threshold, no value of $u$ clearly dominates both the $1.5$-epoch SFT checkpoint and the Post-SFT MDM control across NFEs.
With it, several values of $u$, for example $u \in \{32, 64\}$, lie above both baselines over large ranges of NFEs (\Cref{fig:llada_collapse}).

\begin{figure}[ht]
  \centering
  \includegraphics[width=0.48\linewidth]{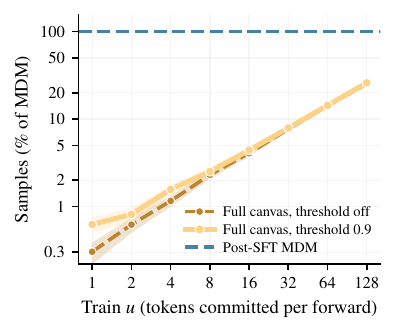}
  \caption{Distinct training samples visited as a percentage of Post-SFT MDM, for full-canvas PU with and without the training-time threshold.
  Bands: one standard error.}
  \label{fig:app_sample_diversity}
\end{figure}

\begin{figure}[ht]
  \centering
  \includegraphics[width=\appscale\linewidth]{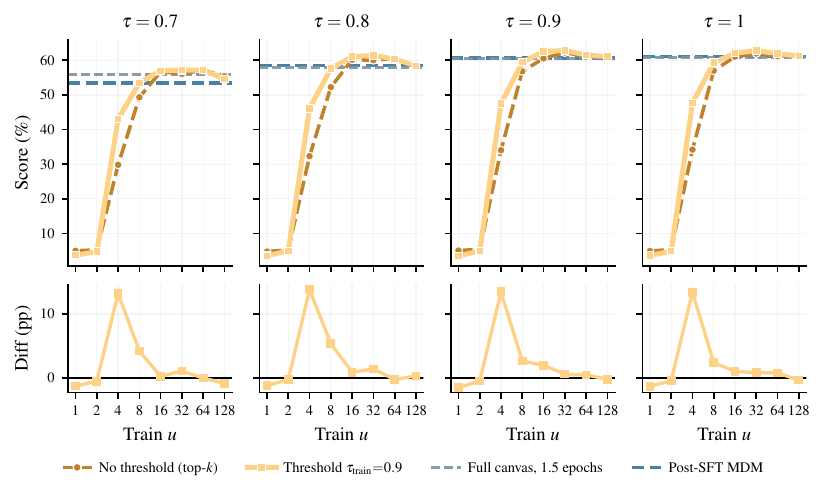}
  \caption{Full-canvas PU Post-SFT across $u$, with ($\tau_{\text{train}}=0.9$) and without the training-time threshold.
  Columns: decode threshold $\tau$.
  Top: score.
  Bottom: difference, with minus without, in points at matched train $u$ and $\tau$; NFEs are not matched.}
  \label{fig:app_trainu_taus}
\end{figure}

\begin{figure}[ht]
  \centering
  \includegraphics[width=\appscale\linewidth]{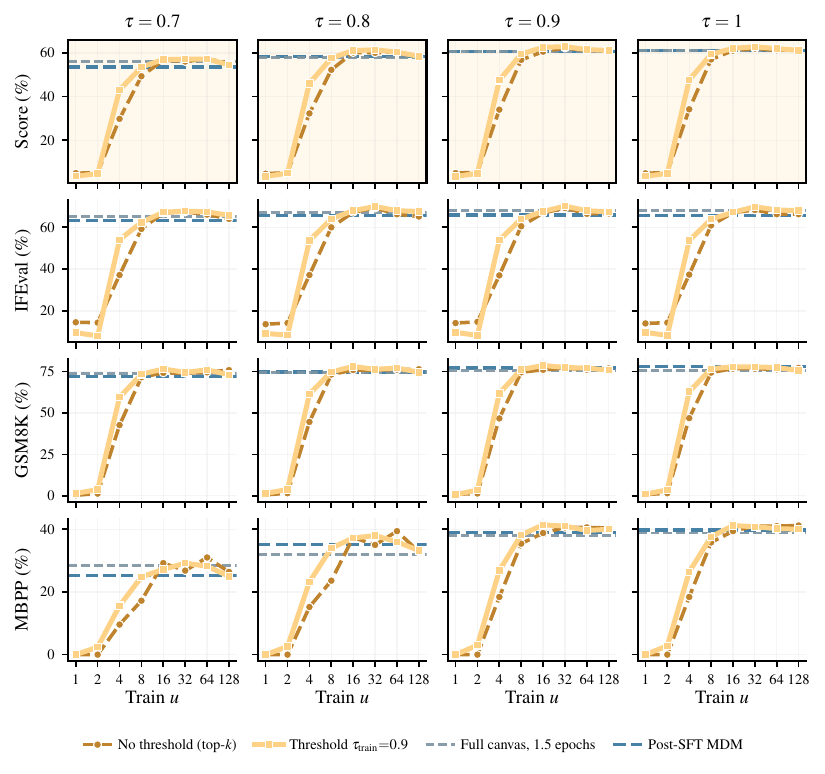}
  \caption{Per-benchmark scores for the runs in \Cref{fig:app_trainu_taus}.
  Columns: decode threshold $\tau$.
  Rows: score, then each benchmark.}
  \label{fig:app_trainu_full}
\end{figure}

\begin{figure}[ht]
  \centering
  \includegraphics[width=\appscale\linewidth]{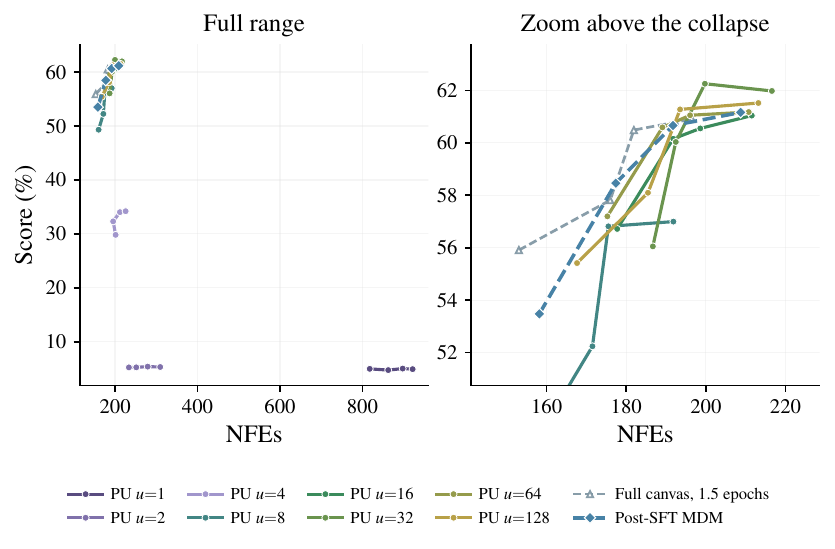}
  \caption{Score against NFEs for full-canvas PU Post-SFT across $u$, without the training-time threshold.
  Each curve sweeps the decode threshold $\tau \in \{0.7, 0.8, 0.9, 1\}$.
  Left: full range.
  Right: scores above $50\%$.}
  \label{fig:app_nfe_topk}
\end{figure}

\FloatBarrier
\subsection{Comparison with published LLaDA results}
\label{app:llada_comparison}

\Cref{tab:llada_published_comparison} compares our measurements with published LLaDA-8B values.
These are the only external numbers in the paper; every other comparison is between runs that share the data, decoding policy, and evaluation harness.
With the plain prompt format (\Cref{app:setup_llada}), our harness reproduces the published base figures.
All other results use the chat prompt format (\Cref{app:setup_llada}), which matches the format of our SFT data.
With it, the base model scores $10.1$ points lower on GSM8K, $3.2$ lower on MBPP, and $2.4$ higher on IFEval.
The prompt format is fixed across all our runs, so this shift does not affect our comparisons, but our absolute scores should not be compared directly with published ones.
Published figures also disagree with each other: for LLaDA-8B-Instruct on GSM8K, \citet{nie2025llada} report $69.4$ with full-canvas decoding and $77.5$ with block-$32$ decoding, and \citet{ye2025dream} report $78.6$.

The checkpoints our Post-SFT runs start from score $75.9$ and $77.0$ on GSM8K, $38.8$ and $41.6$ on MBPP, and $68.0$ and $65.4$ on IFEval.
These are in the same range as the published LLaDA-8B-Instruct figures, and higher on IFEval than LLaDA-8B-Instruct as evaluated by \citet{ye2025dream} ($59.9$).

\begin{table}[ht]
  \centering
  \caption{LLaDA-8B: published figures against our own measurements.
  The published rows are reference only and are not protocol-matched to anything else in this paper.
  Shot counts are given per row as IFEval/GSM8K/MBPP, with \emph{n/r} where the source does not report a value.
  Our rows decode with Fast-dLLM at $\tau=1$, left to right in blocks of $32$ tokens, and report IFEval 0-shot prompt-level strict accuracy, GSM8K 8-shot exact match, and MBPP 3-shot pass@1.
  The two LLaDA-8B-Base rows differ only in the prompt format; the two SFT rows are the checkpoints the Post-SFT runs start from, evaluated in the chat format.
  The last two published rows are LLaDA as evaluated by \citet{ye2025dream}, not Dream's own model.}
  \label{tab:llada_published_comparison}
  \small
  \renewcommand{\arraystretch}{1.08}
  \begin{tabular}{@{}llcccc@{}}
    \toprule
    Model & Source & Shots & IFEval & GSM8K & MBPP \\
    \midrule
    \multicolumn{6}{@{}l}{\emph{Published, reference only}} \\
    LLaDA-8B-Base & LLaDA & n/r\,/\,4\,/\,4 & n/r & $70.3$ & $40.0$ \\
    LLaDA-8B-Instruct, full canvas & LLaDA & n/r\,/\,4\,/\,4 & n/r & $69.4$ & $41.0$ \\
    LLaDA-8B-Instruct, block $32$ & LLaDA & n/r\,/\,4\,/\,4 & n/r & $77.5$ & $34.2$ \\
    LLaDA-8B-Base & Dream & n/r\,/\,8\,/\,4 & n/r & $70.9$ & $39.0$ \\
    LLaDA-8B-Instruct & Dream & n/r & $59.9$ & $78.6$ & $34.2$ \\
    \midrule
    \multicolumn{6}{@{}l}{\emph{Ours, LLaDA-8B-Base at $\tau=1$}} \\
    \quad Plain prompt format & This work & $0$\,/\,$8$\,/\,$3$ & $20.9$ & $70.4$ & $39.8$ \\
    \quad Chat prompt format & This work & $0$\,/\,$8$\,/\,$3$ & $23.3$ & $60.3$ & $36.6$ \\
    \midrule
    \multicolumn{6}{@{}l}{\emph{Ours, LLaDA-8B-Base after SFT on Dolci-Instruct-SFT, at $\tau=1$}} \\
    \quad Full canvas, $1.5$ epochs & This work & $0$\,/\,$8$\,/\,$3$ & $68.0$ & $75.9$ & $38.8$ \\
    \quad Block diffusion, $6$ epochs & This work & $0$\,/\,$8$\,/\,$3$ & $65.4$ & $77.0$ & $41.6$ \\
    \bottomrule
  \end{tabular}
\end{table}

\FloatBarrier

\clearpage
\section{Relation to the closest work}
\label{app:closest_work}

Our contribution concerns how to train masked diffusion models for the trajectories they encounter during generation: which states to train on, what information to propagate between them, and how far to propagate the learning signal. PUMA~\citep{kim2026puma}, Loopholing~\citep{jo2026loopholing}, MetaState~\citep{xia2026metastate}, and Relay~\citep{rozonoyer2026relay} establish complementary foundations for these questions. Building on them, we study trajectory alignment with the inference policy and temporal credit assignment via BPTT together, evaluate their contribution to general-purpose instruction post-training, and investigate the conditions under which closer alignment remains learnable.

\paragraph{PUMA: from trajectory construction to temporal credit assignment.}
\citet{kim2026puma} establish progressive unmasking as an effective training strategy, with a marginal-agreement guarantee under idealized posterior sampling. We build on their teacher-forced trajectories: the policy selects reveal positions, but ground-truth commitments still differ from an imperfect model's sampled content. Our carry adds the predecessor's model-produced representation, providing a weak form of student forcing and a differentiable path from later losses to earlier computations.
PUMA also recognizes that long chains can provide redundant supervision and motivates a coarse-to-fine schedule partly by the unreliable policy early in training. Its schedule increases the number of stages, reducing the fraction of each sequence revealed per step, whereas inference includes two- and three-token steps.\footnote{In the paper's notation, $K$ refers to the number of stages, which increases from $12$ to $42$ for the case of TinyGSM; the number of unmasked tokens per step is then nominally $L/K$, with $L$ close to $512$. The confidence threshold makes the measured value significantly larger (\Cref{app:pu_details}).} Its schedule ablations already suggest that matching the reveal order need not mean matching its granularity.
Our post-training results show that small-count degradation persists with a pretrained, instruction-tuned denoiser. An informed policy does not by itself ensure a learnable trajectory. We trace it to local overfitting caused by sample reuse (\Cref{subsec:smallu}), and also examine masking-level diversity and the balance of supervision on committed versus uncommitted positions (\Cref{app:llada_results}), connecting schedule sensitivity to concrete learning dynamics. PUMA's large-model experiments establish effectiveness in a specialized setting: LoRA adaptation of a code model, evaluated on HumanEval and MBPP. Our full-model, general instruction post-training tests a broader question: whether the gains extend to instruction following while retaining reasoning and code capabilities, under both full-canvas and block diffusion.

\paragraph{Loopholing: from self-conditioned context to trajectory-trained state.}
\citet{jo2026loopholing} establish the value of retaining continuous information across discrete sampling. They name the underlying limitation the \emph{sampling wall}: collapsing each predicted distribution to a single token discards what the model knew about the alternative tokens, and the carry lets part of this information survive the commitment. We adopt their simple, empirically effective hidden-state carry to isolate the role of training and trajectory alignment with minimal architectural changes. Their self-conditioning pass takes a corrupted sequence and an empty carry, then supplies its detached output (with stop-gradient) to a second prediction on the same sequence. Our cached carry instead comes from the preceding rollout step, which itself consumed its predecessor's carry after trajectory initialization. We therefore hypothesize that this recurrent history better matches inference-time conditioning than a context generated in isolation, while recognizing that teacher-forced token content still leaves a mismatch. A stored carry that crosses an optimizer update was also produced by the parameters before that update.
Loopholing reports approximately $30\%$ additional training time. Our caching approach removes its auxiliary context-generation pass, thus avoiding this additional training cost, and replaces it with storage of one hidden-width vector per token, rather than a vocabulary-sized distribution. This is a compact representation relative to logits, although total storage still scales with sequence length and the number of active trajectories. Within each unroll window, we can also train the computation that produced the carry through subsequent losses. \method thus develops a multi-step training direction, combining recurrent conditioning with temporal credit assignment. The removed auxiliary pass and the additional memory and computation of BPTT are separate costs, and we provide more details on this trade-off between them in \Cref{tab:llada_postsft_hparams}.

\paragraph{RCD and DiffusionGemma: distribution-derived carries.}
Residual Context Diffusion (RCD; \citealp{hu2026rcd}) and DiffusionGemma~\citep{team2026diffusiongemma} carry a different quantity from the backbone hidden state we adopt from Loopholing. Both pass the expected input embedding $\mathbf{E}^{\top}\mathbf{p}$ to the next step, where $\mathbf{p} \in \Delta^{|\mathcal{V}|}$ is the predicted distribution at a position and $\mathbf{E} \in \mathbb{R}^{|\mathcal{V}| \times d}$ is the input embedding table. The resulting vector $\mathbf{E}^{\top}\mathbf{p} \in \mathbb{R}^{d}$ has the same width as the hidden-state carry. RCD entropy-weights this residual. It trains the target model using residuals supplied by a frozen reference model, then recursively feeds back the target model's own predictions at inference. DiffusionGemma applies a temperature to the logits before computing $\mathbf{p}$ and passes the expected embedding through a feedforward network, which yields a self-conditioning signal for the next denoising step. Adapting such carries to our trajectory-trained setting is an interesting direction.

\paragraph{MetaState: memory architecture and trajectory alignment.}
\citet{xia2026metastate} address the \emph{information island}: continuous computations are lost between denoising steps when only discrete token decisions persist. Their fixed-size working memory improves reasoning with a frozen backbone. Their objective also combines supervision on all masked positions with a term focused on the positions about to be revealed. This emphasis connects to trajectory alignment, and we investigate a loss following that design in our small-$u$ analysis (\Cref{subsec:smallu,app:llada_results}).
Both approaches train recurrent state through unrolling, but MetaState uses random reveal orderings and Dirichlet--Multinomial reveal counts, whereas its decoder is confidence-based. \method follows the model's unmasking policy and updates the denoiser itself, thus better aligning with our training--inference alignment objective. MetaState uses four unrolled steps by default and reports broadly stable performance over $K\in\{3,4,5\}$. These steps partition the entire training trajectory; our window $W$ instead truncates gradients along trajectories that can span multiple updates.
MetaState trains on 50k examples from the general-purpose T\"ulu-3 mixture and evaluates mathematical reasoning and code generation. Our distinction is not general-purpose training data alone: we study full-model instruction post-training and explicitly evaluate instruction following alongside reasoning and code. Their work advances memory architecture and efficiency, while we fix a minimal carry to isolate training and alignment choices. Therefore, we do not compare the architectures directly; combining more effective or compact memories with \method is an important direction for future work (\Cref{sec:conclusion}).

\paragraph{Relay: shared mechanism, complementary evidence and analysis.}
Concurrent work by \citet{rozonoyer2026relay} is particularly close: it combines model-selected, teacher-forced rollouts with differentiable per-token state and truncated BPTT. It includes rollout-only and detached-carry (i.e., with stop-gradient) controls in both Sudoku and Fast-dLLM v2 (1.5B). The latter is adapted on a code-and-mathematics mixture and evaluated on HumanEval and MBPP. Our setting is general instruction post-training of LLaDA-8B: we evaluate instruction following as the target capability and mathematical reasoning and code generation for capability retention, under both full-canvas and block diffusion.
Relay's results also motivate studying the components across settings: the detached carry gives a substantial gain on Sudoku, while the language-model results show a different balance between rollout accuracy and the NFE benefit of temporal training. Their large-model study establishes effectiveness for a two-step unroll within block diffusion. We broaden this evidence by jointly varying training-trajectory granularity and the unroll horizon at 8B scale, under both full-canvas and block decoding: the question is how reliably the benefits extend across post-training configurations, and what controls their magnitude. In particular, we characterize when closer alignment improves the quality--NFE frontier, when repeated states and concentrated masking levels impair learning, and how targeted interventions mitigate these effects (\Cref{subsec:smallu,app:llada_results}). This analysis provides practical guidance for carry-based post-training beyond establishing that a differentiable channel can help.

\applefootnote{ \textcolor{textgray}{\sffamily Apple and the Apple logo are trademarks of Apple Inc., registered in the U.S. and other countries and regions.}}

\end{document}